\documentclass{article}

\PassOptionsToPackage{numbers, compress, sort}{natbib}

 \usepackage[main, final]{neurips_2026}

\usepackage[utf8]{inputenc} 
\usepackage[T1]{fontenc}    
\usepackage{hyperref}       
\usepackage{url}            
\usepackage{booktabs}       
\usepackage{amsfonts}       
\usepackage{nicefrac}       
\usepackage{microtype}      
\usepackage{xcolor}         
\usepackage{graphicx}
\usepackage{booktabs}
\usepackage{caption}
\usepackage{multirow}
\usepackage{xcolor}
\usepackage[svgnames]{xcolor}
\usepackage{paralist}
\usepackage{subcaption}
\usepackage{comment}
\usepackage[british]{babel}
\usepackage{amsmath}
\usepackage{amsthm}
\usepackage{cleveref}
\usepackage{wrapfig}
\usepackage{enumitem}
\usepackage{thmtools}
\usepackage{thm-restate}
\usepackage{MnSymbol}
\usepackage{siunitx}
\usepackage[noblocks]{authblk}

\usepackage{algorithm}
\usepackage{algorithmic}
\usepackage{titletoc}
\DeclareMathOperator{\Div}{Div}
\DeclareMathOperator{\DivCurve}{DivCurve}
\DeclareMathOperator{\diam}{diam}

\declaretheorem[name=Definition] {definition}
\declaretheorem[name=Proposition]{prop}

\hypersetup{
    colorlinks = true,
    urlcolor   = RoyalBlue,
    linkcolor  = RoyalBlue,
    citecolor  = RoyalBlue
}

\title{Diversity Curves for Graph Representation Learning}

\renewcommand\Affilfont{\small\normalfont}

\makeatletter
\renewcommand\AB@affilsepx{\hspace{1.5cm}\protect\Affilfont}
\makeatother

\renewcommand{\thefootnote}{\fnsymbol{footnote}}

\author[1,2]{Katharina Limbeck}
\author[3]{Nadja Häusermann}
\author[3]{Martin Carrasco}
\author[$\dagger$,4,5]{\\ Guy Wolf}
\author[$\dagger$,3]{Bastian Rieck}

\affil[1]{Helmholtz Munich}
\affil[2]{Technical University of Munich} 
\affil[3]{University of Fribourg}
\affil[4]{University of Montréal}
\affil[5]{Mila - Quebec AI Institute}

\begin{document}

\footnotetext[2]{These authors jointly supervised this work.}

\maketitle

\begin{abstract}

Graph-level representations are crucial tools for 
characterising structural differences between graphs.  
%
However, comparing graphs with different cardinalities, even when sampled from the same underlying distribution, remains challenging.
Unsupervised tasks in particular require  interpretable, scalable, and reliable size--aware graph representations.
Our work addresses these issues by tracking the structural diversity of a graph across coarsening levels.
%
The resulting graph embeddings, which we denote \emph{diversity curves}, are interpretable by construction, 
efficient, and 
directly comparable across coarsening hierarchies. 
 Specifically, we track the \emph{spread} of graphs, a novel isometry invariant that is inherently well-suited for encoding the metric diversity and geometry of graphs. 
We utilise edge contraction coarsening and prove that 
this improves expressivity, thus leading to more powerful graph-level 
representations than structural descriptors alone.
Demonstrating their utility over a range of baseline methods in practice, we use \emph{diversity curves} to 
\begin{inparaenum}[(i)]
  \item cluster and visualise simulated graphs across varying sizes,
  \item distinguish the geometry of single-cell graphs, 
  \item compare the structure of molecular graph datasets, and
  \item characterise geometric shapes. 
\end{inparaenum}
\end{abstract}

\setcounter{footnote}{0}
\renewcommand{\thefootnote}{\arabic{footnote}}

\section{Introduction}
\label{sec:intro}

Understanding structural differences between graphs is a key step in exploring, describing, and distinguishing graph datasets or distributions. 
In particular when comparing graphs drawn from the same distribution, suitable graph representations should adapt to differences in scale and capture coarse structural similarities between graphs without relying on supervision.
However, many existing methods for graph comparison and graph-level representation learning are inherently sensitive to differences in graph sizes \citep{tantardini2019comparing, yehudai2021local}. 
This can negatively impact interpretation and downstream performance.
We address this issue by investigating to what extent \emph{graph coarsening} techniques can be used to perform structural comparisons between graphs. 
To this end, we propose \emph{size-aware methods} that are aware of both the cardinality of graphs, as well as their intrinsic size, i.e.\ their structural diversity across coarsening scales.
These size-aware representations aim to address geometric tasks, where graphs have been sampled from the same underlying distribution and share the same inherent structure but vary in size. 
%
Our work makes the following \textbf{contributions}: 
\begin{description}[left = 0.5em]
    \item[Novel graph representations:] We propose \emph{diversity curves} for comparing graphs. Our approach is the first to use coarsening together with structural notions of diversity to derive hierarchical graph descriptors 
    that track how graph structure evolves across the coarsening hierarchy. 
    
    \item[Expressivity guarantees:] We prove that edge contraction coarsening improves the expressivity of graph structural comparisons via graph invariants. 

    \item[Excellent performance:] We show that diversity curves excel at geometric tasks, outperforming unsupervised baseline methods across a range of applications. Our methods perform well at clustering simulated graph distributions, distinguishing between trajectory- and cluster-like single-cell graphs, and characterising geometric graphs across varying sizes. 
\end{description}

\section{Related work}
\label{sec:relatedwork}

The problem of finding \emph{size-robust graph representations} has been studied from different viewpoints.
Recent research studying the size-robustness 
has shown that graph neural networks (GNNs), 
typically used in supervised settings, struggle with size generalisation in certain instances, for example when learning solutions to combinatorial problems \citep{joshi2022learning}, or characterising geometric graphs across varying scales 
\citep{ballester2024mantra}. 
Generalisation across coarsening resolutions and graph sizes has further been shown to be difficult for GNNs applied to molecular graphs \citep{Koke26b}, or across specific graph distributions \citep{yehudai2021local}. 
%
%
Learning size-robust graph representations thus is a core topic of interest for graph learning. 

\textbf{Unsupervised methods.}
Size generalisation remains relevant for unsupervised methods, which are the focus of our work. 
As commonly used baselines, graph kernel methods based on path- or walk-based kernels are known to suffer from size-sensitivity~\citep{han2025multi}, making their embeddings less suitable in unsupervised settings. Similar size-sensitivity issues have been observed in \emph{unsupervised graph embeddings}. For instance, graph2vec compares subgraph collections~\citep{narayanan2017graph2vec}, which is more nuanced than directly comparing entire graphs, but can still be dominated by the vocabulary of larger graphs.
Spectral embeddings, such as NetLSD~\citep{tsitsulin2018netlsd} or SpectralZoo~\citep{jin2020spectral}, aim to characterise each graph via summarising the spectrum of eigenvalues of its normalised Laplacian. In practice, such methods still struggle with clustering graphs generated from the \emph{same} model across varying densities and sizes~\citep{tantardini2019comparing}. Other approaches like FEATHER~\citep{rozemberczki2020characteristic} leverage random walks, which are pooled to obtain graph-level embeddings, thus also implicitly incorporating size. 
Moreover, while contrastive learning can be used to adapt GNNs to 
unsupervised tasks \citep{you2020, Yin2022},  their reliance on graph augmentations and comparisons between positive and negative pairs limits their training efficiency and their training objective is sensitive to varying graph sizes. 

\textbf{Graph filtrations.}
To obtain interpretable representations, prior work on geometric and topological machine learning often employs 
the concept of \textit{graph filtrations}. Such filtrations utilise nested sequences of subgraphs derived from different graph descriptors to obtain a multiscale view of a graph. Across these filtrations, persistent homology (PH) tracks topological features~\citep{ballester2025expressivity, southern2023curvature}, while TopER reports the node and edge counts \citep{tola2025toper}, and filtration curves evaluate general graph descriptors \citep{obray2021filtration_curves}. 
Many of these approaches remain inherently size-sensitive because the subgraph filtrations and final quantities recorded often explicitly or implicitly depend on the size of the input graph \citep{chazal2014convergence_rates_persistence_diagram, Siu2025asymptotics_betti_numbers}. 
This motivates the need to develop size-robust geometric representations, potentially breaking with the ongoing paradigm of relying on subgraph filtrations, and considering novel construction processes. 

\textbf{Magnitude and spread.}
To find size-robust representations, we propose a novel approach  leveraging \emph{magnitude} and \emph{spread} as graph invariants that measure the structural diversity of graphs, i.e.\ the effective number of distinct vertices, based on node distances~\citep{leinster2013magnitude,willerton2015spread}. These measures were originally introduced to assess the diversity of a metric space across different scales of the metric; 
recently, they have seen use in deep learning applications, specifically for evaluating the diversity of latent representations across varying distance scales~\citep{limbeck2024metric}, or for obtaining geometry-aware edge pooling schemes~\citep{limbeck2025geometryaware}.
Our work uses the \emph{spread}, as it is computationally more efficient than magnitude, while being highly correlated \citep{limbeck2025geometryaware, willerton2015spread}.  
 Spread in particular has not yet been studied in the context of comparing graphs of various cardinalities, enabling us to advance beyond the previously proposed distance-scaled summaries of diversity towards developing novel size-aware geometric tools. 
%
%
%

\textbf{Graph coarsening.}
Our method leverages the use of a \emph{coarsening method} to track the diversity of graphs at different cardinalities. 
Graph coarsening is commonly used in hierarchical pooling layers 
of GNNs that aim to reduce the size of graphs while preserving graph structure \citep{grattarola2022understanding, liu2023graph_pooling}.  
We exploit this approach for our diversity curves, using \emph{edge contraction coarsening} specifically because it is an iterative pooling operation that preserves graph structure~\citep{limbeck2025geometryaware} while improving model expressivity~\citep{lachi2025expressive}. 
Combining diversity invariants together with edge pooling in particular offers promising potential for developing size-robust and expressive graph-level representations. 

\section{Methods}
We propose a framework to compare the structural diversity of graphs across a  coarsening process. The resulting diversity curves yield a vectorised representation for each graph, encoding their hierarchical geometry, which can be used for downstream analysis. 
\Cref{fig:overview} illustrates the construction of diversity curves,  \Cref{sec:methods} describes the process and \Cref{sec:theory} their properties.

\subsection{Diversity curves as graph representations}
\label{sec:methods}
\begin{figure}[t]
\includegraphics[width=1\textwidth]{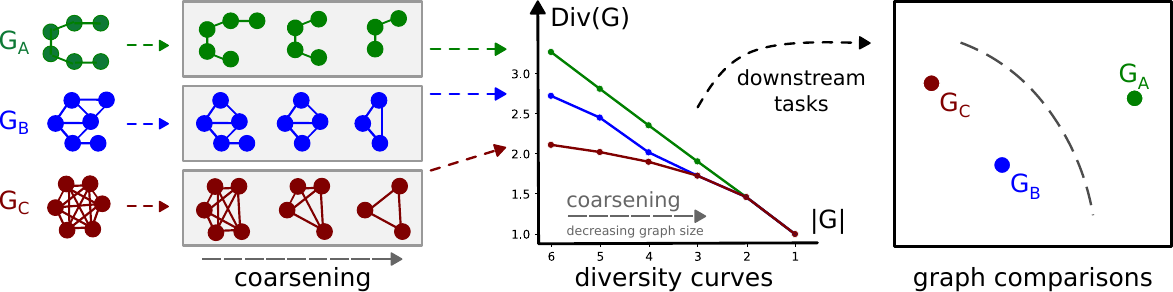}
\caption{Diversity curves for three example graphs. 
Our framework  first \emph{coarsens} the initial graphs via iterative edge contractions 
and tracks 
 structural invariants across the coarsening process.  Specifically, we compute structural diversity based on a choice of metric on the graph, e.g.\ shortest-path distances. 
 \emph{Diversity curves} then represent the structural diversity of a graph as a function of the number of nodes in the coarsened graph. 
 The resulting vector representations compare the coarsened geometry of graphs for unsupervised visualisation, clustering, classification, 
 or explanatory tasks.}
\label{fig:overview}
\end{figure}


\paragraph{Graph coarsening as a multiscale process.} To compare graphs across cardinalities, we coarsen or upsample graphs to the same range of node numbers. 
This requires a graph coarsening process that 
\begin{inparaenum}[(i)]
    \item allows us to 
    control the number of nodes in the coarsened graphs,
    \item is iterative,
    \item is applicable to unattributed graphs,
    \item is unsupervised, and
    \item preserves the connected components in the graph.
\end{inparaenum}
Subsequently, we use edge contraction coarsening, which satisfies our requirements, in particular permits controlling the size of coarsened graphs, is known to improve the expressivity of GNNs~\citep{lachi2025expressive}, is an interpretable operation, and does well at preserving the structure of graphs \citep{limbeck2025geometryaware}; refer to \Cref{app:coarsening} for an extended discussion on the coarsening operation.
Formally, we select a set of cardinalities, $\mathcal{I} \subset \mathbb{N}$, across which we want to compare the coarsened geometry of graphs and we consider a graph $G=(X,E)$ to consist of a set of vertices $X$ and a set of edges $E$. 
We use the notation $|G| = |X| =n$ for the cardinality of the graph 
and $m=|E|$ for the number of edges. 
Furthermore, we define a sequence $(e_i)_{1\leq i \leq k}$ of edges in $G$ that are to be contracted, for some $k \in \{1,\dots,m\}$ such that $n-k = \min_{i\in \mathcal{I}}i$. The edge sequence induces a sequence of coarsened graphs $(G_i)_{n-k\leq i \leq n}$ based on contracting these edges in order, so that $G_{n}=G$ and $G_{i-1}=G_i/{e_{n-i+1}}$. Note that $|G_i| = i$. 
To apply the coarsening operation in practice, we follow existing implementations of edge contraction methods \citep{limbeck2025geometryaware} to select the sequence of edges to contract. 
Specifically, we obtain this edge sequence by assigning a score to each edge and iteratively contracting the edges starting with the lowest scores, while contracting no edge adjacent to a previously contracted edge.
If needed, we reapply this procedure to coarsen graphs to the desired size and aggregate the features of any merged nodes, for instance by averaging. 
We consider scoring based on each edge's importance for the graph's structure \citep{limbeck2025geometryaware} or random scoring~(cf.\ \Cref{app:coarsening}).
In case a dataset contains small graphs, we further employ upsampling strategies, for example by uniformly upsampling a node and reconnecting it to all neighbours of the initial node. This strategy allows us to 
extend the sequence of coarsened graphs to $(G_i)_{n-k\leq i \leq N}$ for some $N>n$. 

\paragraph{Structural diversity for comparing graphs.} Structural diversity is inherently well-suited to compare graphs at the same cardinality. By measuring diversity as the \emph{effective number} of distinct nodes in a graph, the \emph{spread} of a graph is a novel and expressive invariant that encodes the graph's geometry based on a choice of distance between nodes. 
Throughout this work, we consider a graph \(G=(X, E)\) 
to be associated with a metric $d$, so that \((X, d)\) defines a finite metric space. 
Given 
such a graph, 
its structural diversity, as measured via the \emph{spread} of its metric space~\citep{willerton2015spread, limbeck2025geometryaware}, 
is defined as  \begin{equation}
\Div(G) := \sum_{x \in X} \frac{1}{\sum_{y \in X} \exp(-d(x,y))}.
\end{equation}
Note that \(\Div(G)\in[1,|X|]\) and that structural diversity can be directly interpreted as the effective number of nodes in the graph. 
Through coarsening, we aim to exploit the inherent link between diversity and this notion of effective size. Spread is more efficient to compute than alternative diversity measures like the magnitude of a metric space~\citep{leinster2013magnitude, willerton2015spread}, while allowing for a flexible choice of geometry determined by the metric on the graph. Further, spread can be directly compared across graphs as it is invariant to isometries and permutations between nodes \citep{limbeck2025geometryaware}. 
By default, we use shortest-path distances for diversity computations as we find that they perform well across our experiments and are faster to estimate than alternative graph metrics, say diffusion distances or resistance distances, as explored in the ablation study in 
\Cref{app:ablation_metric}. 
Coarsening allows structural invariants to be robust to varying cardinalities. In fact, coarsening counteracts one of the notable weaknesses of many geometric tools, diversity measures, or path-based summaries, which are known to struggle with directly comparing data of varying sizes \citep{limbeck2024metric}. 
Our framework thus still allows us to use shortest-path distances to make comparisons across  coarsened graphs. 
Extending our framework, other choices of graph invariants or summary statistics are possible, but we consider spread to be a useful and flexible default choice for the reasons described above.

\paragraph{Diversity curves as size-aware graph representations.} Our final graph representations, which we call \emph{diversity curves}, track the structural diversity of each graph as a function of the number of nodes in the coarsened graph. 
Let \(G\) be a graph associated with a metric \(d\). Further, fix a sequence of integers \(\mathcal{I}\) representing cardinalities, and a choice of down- or upsampling strategy. Using the down- or upsampling methods we construct a multiset of sequences of coarsened graphs, \(  \mathcal{G} = \{\{ (G_{i,j})_{i \in \mathcal{I}}\}\}_{j \in \mathcal{J}}  \) indexed by some set $\mathcal{J}$ 
such that $|G_{i,j}| = i$ and $G_{n,j} = G$ for all $j \in \mathcal{J}, i\in \mathcal{I}$. Then, we define the corresponding \textit{diversity curve}, \(\DivCurve(G)=\DivCurve(\mathcal{G})\in \mathbb{R}^{|\mathcal{I}|}\),  at cardinality $i \in \mathcal{I}$ by 
\begin{equation}
\DivCurve(G)_i := \frac{1}{|\mathcal{J}|} \sum_{j \in \mathcal{J}}\Div(G_{i,j}), 
\end{equation}
the mean diversity of all coarsened graphs in $\mathcal{G}$ of cardinality $i$.
This yields a vector representation for each graph, which describes the graph's coarsened geometry across resolutions. Intuitively, graphs with few localised differences will behave similarly when contracted and have more similar diversity curves than graphs with globally different structures. Diversity curves capture exactly these hierarchical trends. In fact, the value of the diversity curve at each scale directly summarises how structurally diverse a given graph is, leading to a clear interpretation as the effective number of distinct nodes in the coarsened graph. This gives a multiscale and size-aware summary of graph structure and diversity that is robust to differences in graph size. 
Our methods thus improve upon previous work on geometric machine learning 
by proposing the combination of coarsening and diversity as a size-aware alternative to existing graph filtrations. 
Diversity curves can be directly treated as functions or vectorised summaries, which can be stored efficiently, and allow for flexible comparisons between graphs. This offers an alternative to graph-level embeddings. Further, we can compute pairwise distances between curves via the 
$L^{p}$-norm between them, allowing us to make formal distributional comparisons between graphs. 

\paragraph{Default parameter choices.} We set the coarsening scales \(\mathcal{I}\) at which we compute diversity curves to 
all integers between $1$ and the maximum number of nodes in the dataset, but other choices are possible. For example, choosing fewer and lower evaluation scales can ensure efficient computations~(cf.\ \Cref{app:tradeoff_ev_scales}).
We compute the diversity curves by default over a set of three coarsening sequences using random edge contractions as we find the choice of edge scoring has no notable impact on performance~(cf.\ \Cref{app:ablation_metric}) and the diversity curve is reasonably stable with respect to the randomness of the edge score and the number of repeats~(cf.\ \Cref{app:ablation_stability}). This is motivated by the fact that while we do not immediately know which coarsening sequence best describes a given graph, we aim to capture key structural characteristics through repeated coarsening. In fact, we will prove how edge pooling improves expressivity. 
By default we compute diversity using shortest-path distances and set $\DivCurve(G)_1=1$ as a graph with one node effectively has one distinct point~\citep{roff2025small, willerton2015spread}. Further, 
we cannot use edge contraction pooling to reduce the size of a graph to less than the number of connected components. In these cases, we use linear interpolation of the diversity curves to estimate the value of diversity at lower sizes.
We publish code for our methods on \href{https://github.com/aidos-lab/diversity_curves/}{GitHub} and detail our algorithm in \Cref{app:extended_methods}.




\subsection{Properties of diversity curves}
\label{sec:theory}

In the following, we describe the theoretical properties that ensure diversity curves are well-behaved and expressive graph representations. Complete proofs and further results 
are provided in \Cref{app:theoretical results}. 
For some of the theoretical treatment, we consider the diversity curve of a graph over the multiset of all possible edge contraction sequences. 
This allows us to talk about the invariance and expressivity of diversity curves from a theoretical perspective without the need for a canonical labelling of the edges, a notoriously difficult problem related to the graph isomorphism problem \citep{babai1983canonical}.

\paragraph{Diversity curves as graph invariants.} 
Our measure of structural diversity is invariant under permutations of vertices and under isometries \citep{willerton2015spread}. For two isomorphic graphs $G=(X,E)$ and $H=(Y,F)$, with edges $e \in E$ and $f \in F$ that are identified under the isomorphism, the graphs $G/e$ and $H/f$ are also isomorphic and diversity curves are invariant in the following sense. 
\begin{restatable}{thm}{ThmInvariance}
\label{thm:invariance}
    For graphs $G\cong H$, their diversity curves over all possible edge contractions are equal.
\end{restatable}

\paragraph{Diversity curves collapse to the same diversity.} 
We investigate the limiting behaviour of diversity curves at low coarsening scales by considering the properties of structural diversity and the requirement posed on the coarsening method to preserve connectivity.
The result below shows that when coarsening graphs up to cardinality $1$, they eventually all have the same diversity. 
As such, diversity curves can distinguish the number of connected components of a graph and behave reasonably under increased coarsening.
\begin{restatable}{thm}{ThmCollapse}
\label{thm:collapse}
    Any two graphs $G$ and $H$ with the same number $c$ of connected components will collapse to the same graph of cardinality $c$ and their diversity curves agree for $i \in \{1,\dots c\}$. Concretely, $\DivCurve(G)_i = \DivCurve(H)_i = i$.  Furthermore, $\DivCurve(H)_{c+1} = \DivCurve(G)_{c+1} \neq c+1$.   
\end{restatable}


\paragraph{Diversity curves increase expressivity.} 
Expressivity, that is the ability to distinguish between graphs, is a key property of useful graph representations \citep{wang2023empirical}. 
To motivate the use of diversity as a graph invariant together with the coarsening process, we show that expressivity is improved when combining them, following the definition of increasing expressivity by 
\citet{lachi2025expressive}, restricting it to consider graphs of the same cardinality. Hence, we show that diversity curves improve upon the expressivity of diversity as measured by the spread of a graph through edge contraction pooling. 
\begin{restatable}{definition}{DefExpressivity}
    \label{def:expressivity}
    For two functions $\varphi$ and $\psi$ on the set of graphs of a fixed cardinality, we say that $\varphi$ is \emph{more expressive} than $\psi$ if for any graphs $G_A$, $G_B$ such that $\psi(G_A) \neq \psi(G_B)$, then also $\varphi(G_A) \neq \varphi(G_B)$, and there exist graphs $G_C$, $G_D$ such that $\varphi(G_C) \neq \varphi(G_D)$ and $\psi(G_C) = \psi(G_D)$.
\end{restatable}
\begin{restatable}{thm}{ThmCoarseningExpressivity}
\label{thm:coarsening_expressivity}
    The 
    diversity curves over all possible edge contractions is more expressive than spread.
\end{restatable}
The proof of \Cref{thm:coarsening_expressivity} consists of providing a pair of graphs with identical spread but distinguishable by their diversity curves. In \Cref{sec:expressivity} we further demonstrate empirically that this holds not just for one example, but improves expressivity across a wide selection of graphs.

\paragraph{Computational complexity. } 
\label{sec:costs_main}
The time complexity of computing the diversity curve of a graph $G=(X,E)$ is determined by $C_{\text{Coarse}} = \mathcal{O}(|E|(\log(|E|)+|X|)+ n_{\text{max}}^2)$, 
constructing one sequence of coarsened graphs using random edge contractions and upsampling, 
and $C_{\Div}=O(|\mathcal{I}|\cdot n_{\text{max}}^3)$, evaluating the structural diversity of all coarsened graphs, where $n_{max} = \max_{i \in \mathcal{I}} i$. 
Altogether, the time complexity of computing the diversity curve 
is given by 
$C_{\text{Coarse}} + C_{\Div}$. 
\Cref{app:computational_complexity} contains a more detailed 
description of the theoretical time complexity,
while \Cref{app:empirical_costs} evaluates the empirical costs; we conclude that  
our methods are efficient to compute while outperforming more  costly methods, such as the graph filtrations used by TopER~\citep{tola2025toper}, on large graphs. Because coarsening is an integral component of our framework, computational costs can be efficiently 
reduced by preprocessing graphs to smaller sizes, or evaluating diversity curves at fewer evaluation scales or lower cardinalities. Furthermore, the computation of the diversity at the different scales can be parallelised. We discuss different options for improving the efficiency of our methods in \Cref{app:computational_complexity}.

\section{Experiments}
\label{sec:experiments}

We show the following experimental results: \Cref{sec:pert} verifies that the changes
in diversity curves induced by edge perturbations correlate with the strength of
the perturbation. This confirms that our methods capture differences between
modified graphs. Next, \Cref{sec:sim} demonstrates that diversity curves perform
well at distinguishing graph distributions and parameter choices while being
robust to variations in graph sizes. Based on this, we investigate applications to
bioinformatics tasks. \Cref{sec:single} shows that diversity curves distinguish between
the coarse geometry of trajectory- and cluster-like single-cell graphs surpassing
more complex summaries. \Cref{sec:enzymes} uses diversity curves to characterise
similarities between enzyme graphs through statistical comparisons. \Cref{sec:mantra}
further confirms that diversity curves reach high performance at distinguishing
geometric shapes. \Cref{sec:expressivity} reports empirical expressivity results supporting our theoretical analysis. Overall, this demonstrates that diversity curves constitute useful and expressive methods for unsupervised tasks.

\subsection{Diversity curves track graph perturbations}
\label{sec:pert}

\begin{wrapfigure}[16]{r}{0.43\textwidth}
\centering
\vspace{-12pt}
\includegraphics[width=0.43\textwidth]{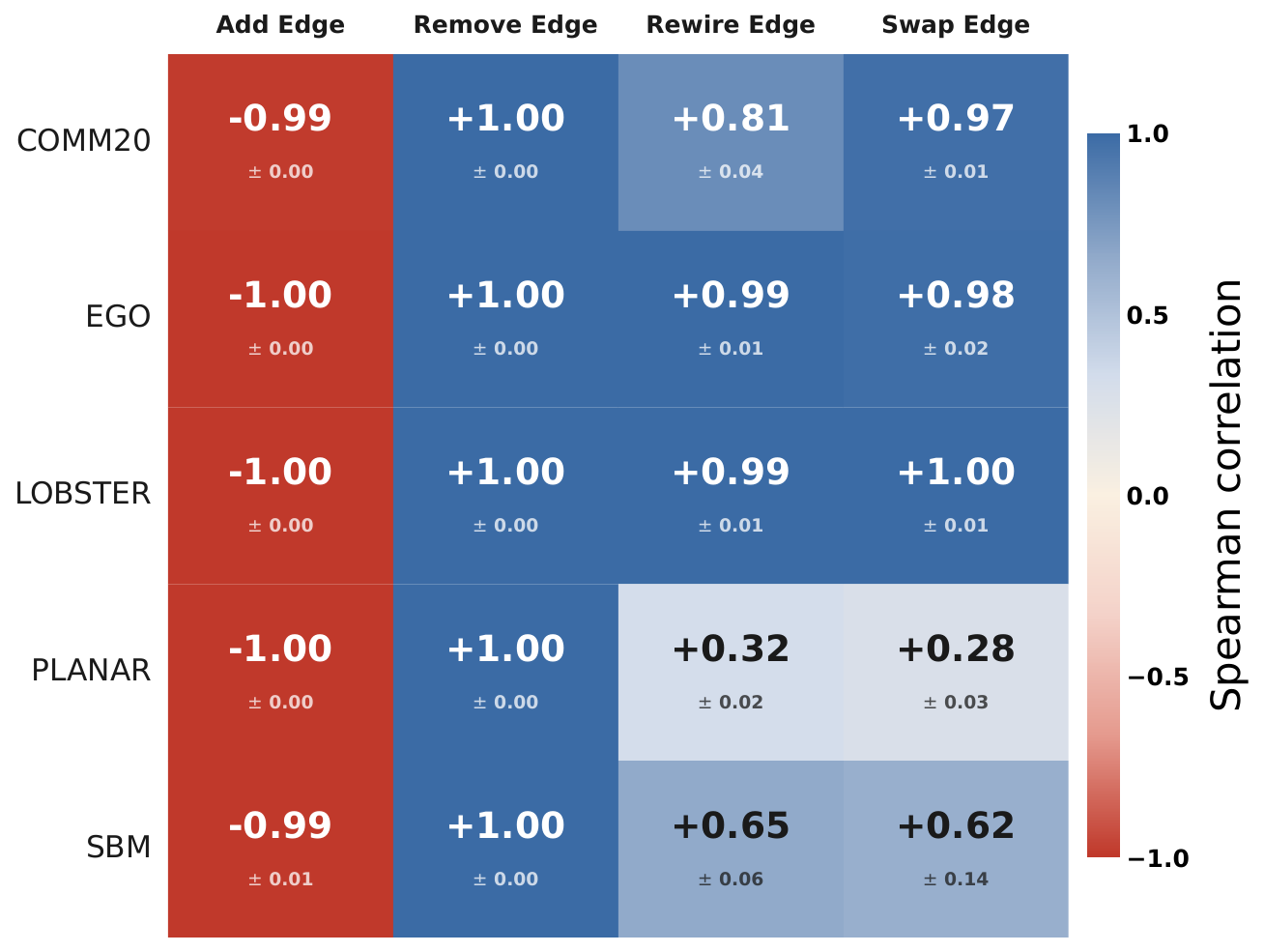}
\caption{Spearman correlation between the degree of perturbation and the change in diversity curve per dataset.}
\label{fig:corr}
\end{wrapfigure}

As a first correctness check, we investigate how diversity curves behave under perturbations of the input graphs~\citep{o2022evaluation}. We fix a degree of perturbation \(p \in [0,1]\) and apply edge perturbations by consecutively adding, removing, swapping, or rewiring edges, resulting in edits to all graphs in the \texttt{SBM}, \texttt{PLANAR}, \texttt{EGO}, \texttt{LOBSTER} and \texttt{COMMUNITIES} datasets \citep{o2022evaluation, krimmel2026polygraph}. 
For each dataset and degree of perturbation, we compute the mean norm across diversity curves, which summarises the structural diversity across graphs. 
We expect that the more graphs are perturbed, the larger the change in diversity curves should be. 
As demonstrated in \Cref{fig:corr}, this trend is true in practice: Across almost all datasets and perturbation types, the more edges get perturbed, the larger the change in diversity curves as measured by the Spearman correlation between the norm of the curves and the degree of perturbation. 
When adding edges, nodes become more similar and structural diversity decreases, hence we observe almost perfect \emph{negative correlation}. When removing edges, distances between nodes and the value of diversity curves increases, leading to perfect \emph{positive correlation}. Rewiring and swapping edges similarly tends to destroy graph structure and increase structural diversity with high correlations for almost all datasets other than \texttt{PLANAR}, which 
by construction contains more random connectivities. Additional results in \Cref{app:extended_experiments_gp} confirm that diversity curves detect small perturbations that result in structural changes to graphs. 









\subsection{Diversity curves distinguish graph distributions}
\label{sec:sim}







To investigate size-robustness, 
we test whether our method can distinguish graphs from different distributions, namely
Erd\H{o}s-R\'{e}nyi (\texttt{ER}), stochastic block model (\texttt{SBM}), random partition (\texttt{RP}), and random geometric (\texttt{RG}) 
graphs. 
For each model and graph sizes between $10$--$29$ nodes, we generate three graphs using a fixed parameter setting~(cf.\  \Cref{app:simulated_graphs}). 
We compute diversity curves using shortest-path distance, using the $L^2$-norm of the 
representations to visualise distributions with PCA. \Cref{fig:sim_var_dist_MDS} shows that our embedding clearly separates graphs from different distributions. 
Further, we use the diversity curves 
for $10$-fold stratified grouped CV using a kNN-classifier to distinguish models. 
We compare to relevant unsupervised graph representation methods, i.e.\ graph kernels~\citep{siglidis2020grakel}, spectral graph representations via SpectralZoo~\citep{jin2020spectral} and NetLSD~\citep{tsitsulin2018netlsd}, topology-inspired embeddings via TopER~\citep{tola2025toper} and vectorial representations of persistence images computed via persistent homology~\citep{ballester2025expressivity},
as well as a random walk based approach, FEATHER \citep{rozemberczki2020characteristic}. 
Results in \Cref{tab:sim_var_dist} demonstrate that diversity curves perform similarly to the best baselines in terms of classification accuracies. 
To further investigate the quality of each graph representation, \Cref{tab:sim_var_dist}  compares \emph{silhouette scores}, obtained from $50\%$ of the data over $100$ resamples, as a measure of clustering quality, for which diversity curves \emph{outperform} all other methods, confirming  that diversity curves achieve better clustering
than baseline methods and demonstrating 
superior performance at distinguishing graph distributions in an unsupervised manner while being robust to differences in graph sizes. 


\begin{minipage}[t]{0.61\textwidth}
\vspace{0pt} 
\captionof{table}{Performance (mean $\pm$ std) of unsupervised methods at distinguishing graphs from four different distributions.}
\centering
\label{tab:sim_var_dist}
\resizebox{\linewidth}{!}{  
\sisetup{
    separate-uncertainty,
    detect-all,
    retain-zero-uncertainty,
}
\begin{tabular}{l
    S[table-format=1.2(2)]
    S[table-format=-1.2(2)]}
\toprule
    Method & {Accuracy ($\uparrow$)} & {Silhouette Score ($\uparrow$)} \\
\midrule
    Diversity Curves (ours) & \bfseries 0.90 +- 0.06 & \bfseries 0.42 +- 0.03 \\
\midrule
    Kernel WL Opt.\ Assig. & \bfseries 0.90 +- 0.05 & 0.02 +- 0.00 \\
    Kernel Shortest Path   & 0.77 +- 0.07           & 0.28 +- 0.04 \\
    TopER                  & 0.70 +- 0.10           & -0.06 +- 0.02 \\
    NetLSD                 & \bfseries 0.90 +- 0.03           & 0.06 +- 0.04 \\
    FEATHER                & 0.75 +- 0.09           & 0.11 +- 0.02 \\
    SpectralZoo           & 0.72 +- 0.09           & 0.05 +- 0.03 \\
    PH Vietoris--Rips      & 0.78 +- 0.07           & -0.01 +- 0.03 \\
\bottomrule
\end{tabular}
}
\label{tab:placeholder}
\end{minipage}  \hfill
\begin{minipage}[t]{0.36\textwidth}
\vspace{0pt} 
\centering
\includegraphics[width=1\textwidth]{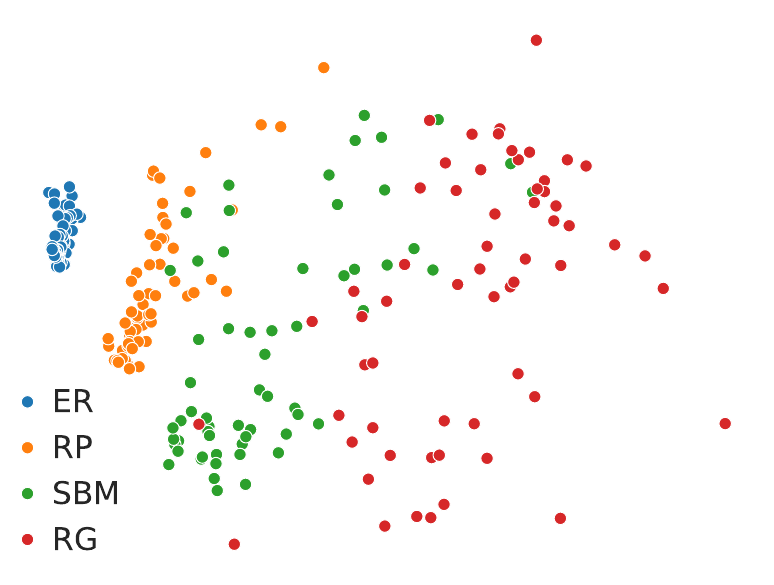}
\captionof{figure}{PCA of diversity curves for \texttt{ER}, \texttt{RP}, \texttt{SBM}, and \texttt{RG} models.
}
\label{fig:sim_var_dist_MDS}
\end{minipage}



Next, we test the ability of our method to characterise and distinguish between graphs generated from the \emph{same} underlying random graph model using three different parameter choices. 
For each graph distribution, parameter setting~(cf.\ Appendix ~\ref{app:simulated_graphs}), and graph sizes from $10$--$29$ nodes, we generate three graphs. 
We again use the Euclidean distances of the 
diversity curves with the shortest-path distance as input to a kNN-classifier with $10$-fold stratified grouped CV. The results reported in \Cref{tab:sim_acc_param} demonstrate that diversity curves achieve high classification accuracies. 
We compare silhouette scores in \Cref{fig:sim_silh_param} to our baselines, and achieve silhouette scores comparable to the best kernel baselines for the \texttt{ER} and \texttt{RP} distributions while obtaining the highest scores for \texttt{RG} and \texttt{SBM} graphs. This confirms that diversity curves outperform other unsupervised methods at distinguishing between parameters used in a random graph model, while being robust to size differences. 
Additional clustering quality metrics and extended results in Appendix \ref{app:simulated_graphs} confirm our findings.

\begin{minipage}[t]{0.6\textwidth}
\vspace{0pt} 
\captionof{table}{Accuracies (mean $\pm$ std) at distinguishing  parameter choices of random graph models. Diversity curves reach top performance scores across graph distributions.}
\centering
\label{tab:sim_acc_param}
\resizebox{\linewidth}{!}{  
\begin{tabular}{lllll}
\toprule
Method & ER & RP & RG & SBM \\
     \midrule
Diversity Curves (ours) & \textbf{1.00 $\pm $ 0.00} & \textbf{0.98 $\pm $ 0.04} & \textbf{0.97 $\pm $ 0.04} & \textbf{0.86 $\pm $ 0.08} \\
\midrule
Kernel Shortest Path & \textbf{1.00 $\pm $ 0.00} & 0.97 $\pm $ 0.04 & 0.89 $\pm $ 0.13 & 0.64 $\pm $ 0.08\\
Kernel WL Opt. Assig. & 0.98 $\pm $ 0.04 & 0.89 $\pm $ 0.08 & 0.93 $\pm $ 0.06 & 0.80 $\pm $ 0.08 \\
TopER & 0.89 $\pm $ 0.08 & 0.87 $\pm $ 0.09 & 0.80 $\pm $ 0.10  & 0.60 $\pm $ 0.07\\
NetLSD & 0.97 $\pm $ 0.05 & 0.97 $\pm $ 0.04 & 0.92 $\pm $ 0.07 & 0.83 $\pm $ 0.08\\
FEATHER & \textbf{1.00 $\pm $ 0.00} &  0.86 $\pm $ 0.08 & 0.94 $\pm $ 0.07 & 0.66 $\pm $ 0.12 \\
SpectralZoo & \textbf{1.00 $\pm $ 0.00} & 0.91 $\pm $ 0.07  & 0.91 $\pm $ 0.11 & 0.63 $\pm $ 0.08\\
PH Vietoris--Rips & \textbf{1.00 $\pm $ 0.00} & 0.89 $\pm $ 0.08  & 0.94 $\pm $ 0.08 & 0.61 $\pm $ 0.10\\

\bottomrule
\end{tabular}
}
\label{tab:sim_acc}
\end{minipage}  \hfill
\begin{minipage}[t]{0.39\textwidth}
\vspace{-0pt} 
\centering
\includegraphics[width=\linewidth]
{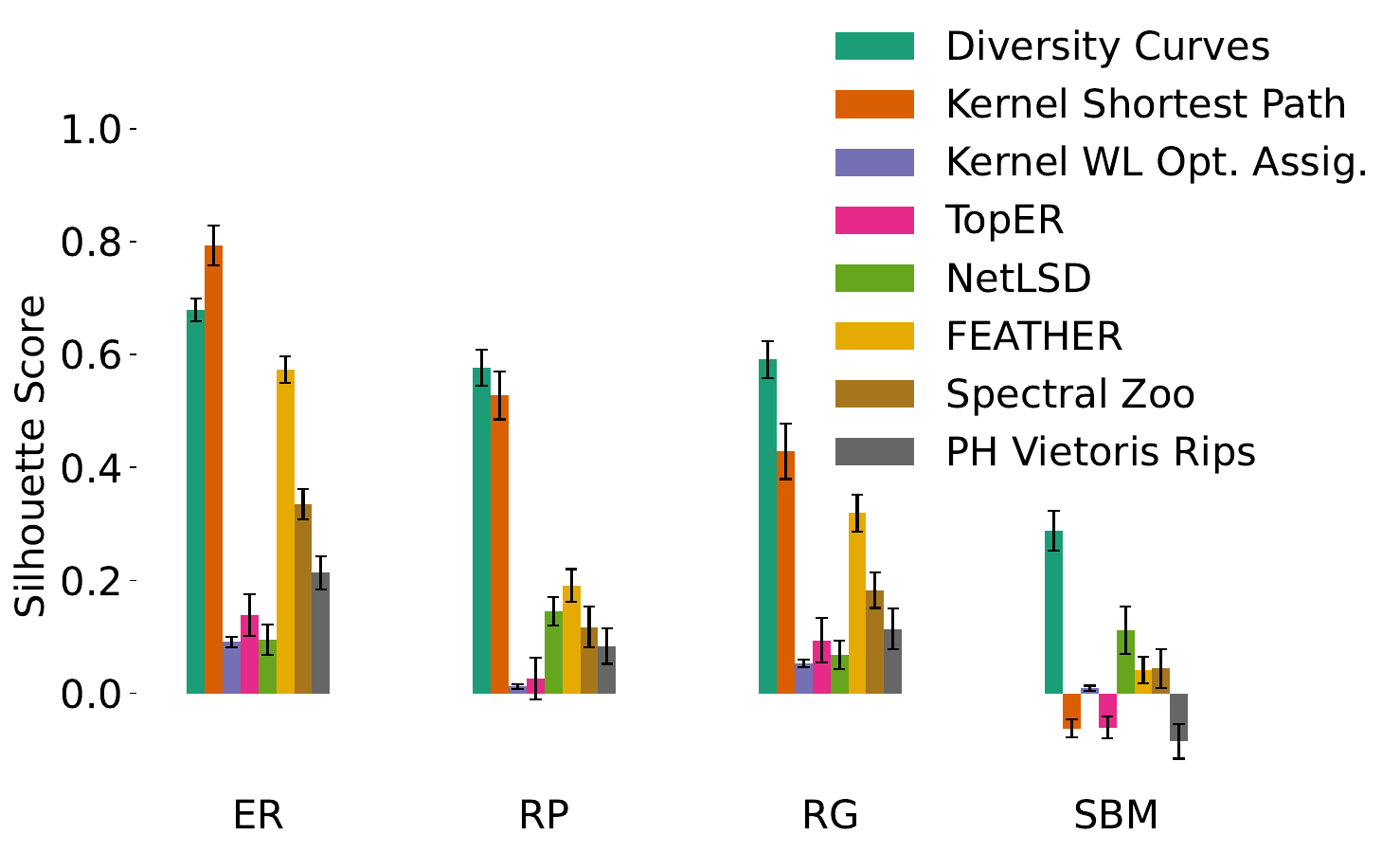}
\captionof{figure}{Silhouette scores ($\uparrow$) for distinguishing different parameter choices.}
\label{fig:sim_silh_param}
\end{minipage}

\subsection{Diversity curves distinguish trajectory or cluster structures in single-cell graphs}
\label{sec:single}
%

Comparing graphs of vastly different sizes is crucial when analysing single-cell transcriptomics data, where the number of cells often varies widely due to experimental factors. A key question to inform further downstream analysis is whether single-cell graphs show cluster- or trajectory-like structures. 
This inherently geometric question informs the biological interpretation of data as well as the choice of computational tools and visualisation methods designed for either type. To address this task, we consider a collection of 169 curated single-cell graphs that have been annotated as cluster- or trajectory-like using biological ground truths \citep{lim2024quantifying, saelens2019comparison}. For each graph processed as described in \Cref{app:sc}, we compute the diversity curve using shortest-path distances and the diversity curve using Euclidean distances between node features, concatenate both curves, and use these curves as our graph representations. 
We note that cluster-like graphs have on average higher structural diversity than trajectory-like graphs. 
The PCA plot 
in \Cref{fig:sc_curves} further shows that our graph representation separates well between both types of graph structure. 
We compare to the methods of \citet{lim2024quantifying}, who compute a set of five geometric scores: the entropy of pairwise diffusion pseudo-time geodesic distances, an entropy-based persistent homology score, a summary of vector norms, a summary of Ripley's K function, and a summary of reachability values. To our knowledge, these are the only methods that have so far been investigated for this task.
We use the geometric scores as inputs for 5-fold stratified CV using a kNN-classifier. \Cref{tab:sc} reports classification accuracies compared to 
the same classification results when using our diversity curves, mean pooled node features, or alternative graph embeddings as inputs. 
We achieve notably better accuracies 
using diversity curves 
than the collection of scores by \citet{lim2024quantifying}. Hence, we find that diversity curves provide more useful and more interpretable graph representations than existing methods designed for this task. Further, we confirm that our methods perform well compared to other unsupervised graph embedding methods. 
We thus show the potential of diversity curves for 
distinguishing the geometry of single-cell graphs, giving actionable insights into deciding further analysis steps. 

\begin{minipage}[t]{0.5\textwidth}
\vspace{0pt} 
\captionof{table}{Performance at distinguishing between trajectory- and cluster-like single-cell graphs.}
\centering
\label{tab:sc}
\resizebox{\linewidth}{!}{  
\begin{tabular}{llc}
\toprule
Dataset & Method & Test Accuracy ($\uparrow$) \\
\midrule
\multirow{2}{*}{Gold (87 Graphs)}
 & Diversity Curves (ours) & \textbf{0.860 $\pm$ 0.077} \\
 & NetLSD & 0.823 $\pm$ 0.047 \\
 & Scores (\cite{lim2024quantifying}) & 0.804 $\pm$ 0.053 \\
 & Best Kernel & 0.785 $\pm$ 0.079 \\
  & TopER & 0.672 $\pm$ 0.086 \\
  & Node Features & 0.661 $\pm$ 0.091 \\
\midrule
\multirow{3}{*}{All (169 Graphs)}
 & Diversity Curves (ours)& \textbf{0.759 $\pm$ 0.077} \\
 & NetLSD & 0.731 $\pm$ 0.047 \\
 & Scores (\cite{lim2024quantifying}) & 0.716 $\pm$ 0.053 \\
 & Best Kernel & 0.685 $\pm$ 0.092 \\
 & TopER  & 0.624 $\pm$ 0.086 \\
    & Node Features & 0.619 $\pm$ 0.091 \\
 & Landscape (\cite{lim2024quantifying}) & 0.570 $\pm$ 0.025 \\
\bottomrule
\end{tabular}}
\end{minipage}
\hfill
\begin{minipage}[t]{0.49\textwidth}
\vspace{12pt} 
\centering
\includegraphics[width=1\textwidth]{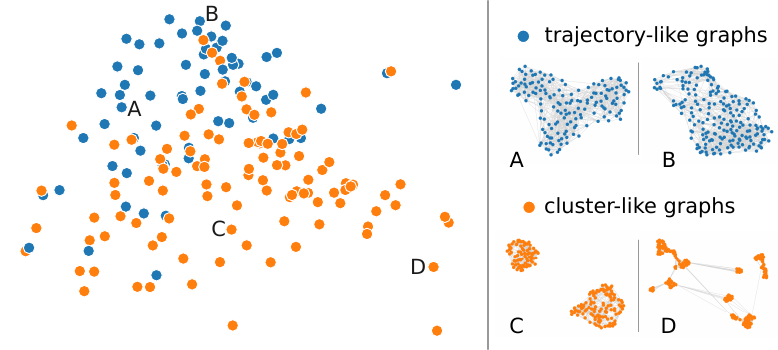}
\captionof{figure}{PCA visualisation of the diversity curves for all 169 single-cell graphs (left). Examples of two trajectory-like graphs and two cluster-like single-cell graphs (right).}
\label{fig:sc_curves}
\end{minipage}

\subsection{Diversity curves characterise molecular graph datasets}
\label{sec:enzymes}
\begin{wrapfigure}[16]{r}{0.5\textwidth}
\vspace{-16pt} 
\centering
\includegraphics[width=1\linewidth]{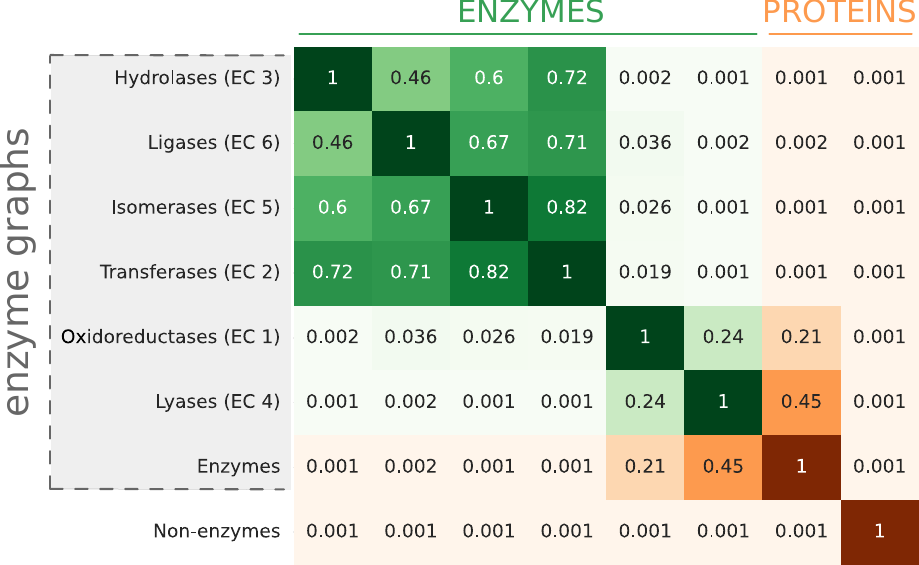}
\caption{Two sample permutation testing results (p-values) for testing the equality in mean diversity curves per class in the \texttt{PROTEINS} and \texttt{ENZYMES} datasets using the $L^{2}$-norm between curves.}
\label{fig:datasets}
\end{wrapfigure}

Characterising the coarse geometry of graphs, diversity curves can be used 
to compare and find structural similarities between graph datasets. 
Specifically, we explore common trends 
across classes of molecular graphs from the \texttt{ENZYMES} or \texttt{PROTEINS} datasets \citep{borgwardt2005protein},
where each graph represents a protein, with some proteins also being enzymes. We first compare the mean diversity curves per class, 
observing that non-enzyme graphs are notably distinct and have lower structural diversity on average than enzyme graphs. 
We further implement a non-parametric two-sample permutation test for the difference in mean diversity curves using the $L^2$-norm between the curves as a test statistic. For each comparison across classes, \Cref{fig:datasets} reports the $p$-values, which we computed across $1000$ random permutations of the data and by adding a pseudo-count for error control \citep{phipson2010permutation}. The results confirm that the set of non-enzymes is clearly distinguishable from all enzyme classes. Meanwhile, we find no evidence to indicate that enzymes from the \texttt{PROTEINS} dataset differ from the EC4 or EC1 classes in the \texttt{ENZYMES} dataset through their averaged diversity curves. This verifies that the geometry of these proteins is more coarsely similar to enzymes 
than to non-enzymes, even when comparing across datasets. 
Hence, diversity curves show that enzymes that share similar biological functions have the same structure across different coarsening scales and datasets, 
which is a crucial question for evaluating graph learning datasets~\citep{coupette2025no}, and a step towards quantifying structural similarities for transfer learning tasks~\citep{huang2023learning}.



\subsection{Diversity curves characterise geometric shapes}
\label{sec:mantra}
Having investigated bioinformatics applications, 
we further perform an experiment to support our claim that diversity curves are capable of describing purely geometric graphs.
To this end, we consider the \texttt{MANTRA} dataset~\citep{ballester2024mantra}, which contains, among others, triangulations of $2$-manifolds, i.e.\  Klein bottle,  torus, sphere, or the real projective plane.  The task of distinguishing these manifolds has been shown to be challenging for graph-based models in particular \citep{ballester2024mantra}. Following our established pipeline, 
we compute diversity curves using shortest-path distances on the $1$-skeleton graphs of the triangulation and use these representations as inputs for SVM classification. These graphs range from $3$ to $15$ nodes. We denote methods that take graphs or triangulations as inputs by $\mathit{G}$ and $\mathcal{T}$, respectively,  in \Cref{tab:mantra}. Because Klein bottles and tori have the same $1$-skeletons, an optimal graph classifier could reach up to 73.8\% accuracy.
Our method~(cf.\ \Cref{tab:mantra}) almost reaches this theoretical upper bound. 
Next, we extend our approach and compute diversity curves from heat kernel distances between faces computed from the Hodge Laplacian \citep{krahn2025heat}, which 
 encodes higher-order structures. 
 Following the same classification procedure as before, the combined 
 diversity curves almost completely solve the task and reach 99\% accuracy at distinguishing the homeomorphism types. This is superior to the best baseline 
 reported by \citet{ballester2024mantra}, i.e.\ the \emph{cellular transformer}~(CT), indicating 
 our representations can be used for mesh data and shape tasks. We further show a PCA plot of the combined 
 diversity curves in \Cref{fig:mantra}, which indicates that our methods effectively separate classes. 
 Finally, we investigate whether our representations encode intrinsic geometric information that is robust to varying graph sizes. 
 To do so, we train our classifier on the diversity curves of the original dataset, but test on a modified dataset created by barycentric subdivision of the original triangulations \citep{ballester2024mantra}, denoted by $\mathcal{T}_\text{up}$, giving graphs of cardinalities between $14$ and $165$. 
 As reported in \Cref{tab:mantra}, our methods \emph{retain} high classification performance surpassing alternative models at this arguably difficult experimental setup. 
 We conclude that diversity curves are expressive representations that 
effectively characterise geometric graphs across varying resolutions and 
sizes.


\begin{minipage}[t]{0.58\linewidth}
\vspace{0pt} 
\centering
\captionof{table}{Model performance (mean $\pm$ std) at distinguishing homeomorphism types of triangulations of $2$-manifolds from \texttt{MANTRA}.}
\resizebox{\textwidth}{!}{  
\begin{tabular}{llcc}
\toprule
Method & Version & Accuracy ($\uparrow$) & AUROC ($\uparrow$) \\
\midrule
Diversity Curve (ours) 
& $\mathcal{T}$ & \textbf{0.99 $\pm$ 0.00}  & \textbf{1.00 $\pm$ 0.00} \\
CT {\tiny(best result from \cite{ballester2024mantra})} & $\mathcal{T}$ & 0.91 $\pm$ 0.01  & 0.94 $\pm$ 0.00 \\
\midrule
Diversity Curve (ours) 
& $\mathit{G}$ & \textbf{0.74 $\pm$ 0.00} & \textbf{0.85 $\pm$ 0.00} \\
 Optimal Classifier 
 & $\mathit{G}$ &  \textbf{0.74 $\pm$ 0.00}  & 0.80 $\pm$ 0.00\\
\midrule
Diversity Curve (ours) 
& $\mathcal{T}_{\text{up}}$  & \textbf{0.83 $\pm$ 0.02}  & \textbf{0.96 $\pm$ 0.01} \\
CT {\tiny(best result from \cite{ballester2024mantra})} & $\mathcal{T}_{\text{up}}$  & 0.74 $\pm$ 0.00 &  0.83 $\pm$ 0.01 \\
\bottomrule
\end{tabular}}
\label{tab:mantra}
\end{minipage}
\hfill
\begin{minipage}[t]{0.38\linewidth}
\vspace{0pt}
\centering
\includegraphics[width=1\linewidth]{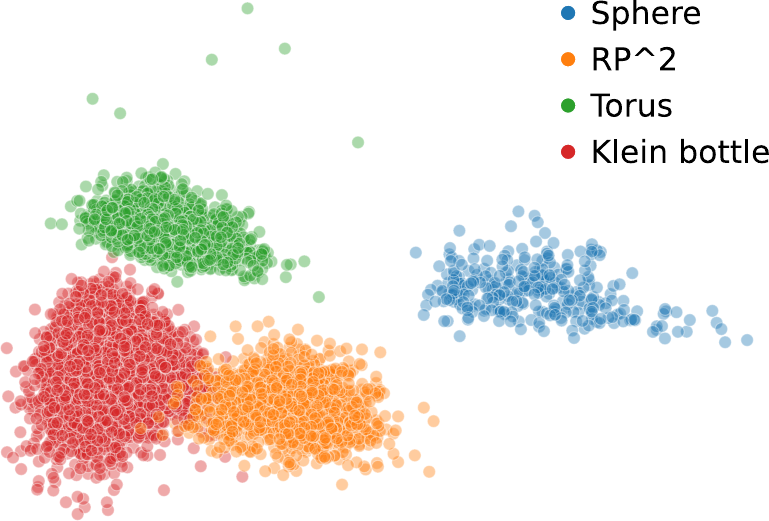}
\captionof{figure}{PCA visualisation of \mbox{diversity} curves for \texttt{MANTRA}.}
\label{fig:mantra}
\end{minipage}










\subsection{Diversity curves are expressive graph representations}
\label{sec:expressivity}


\begin{wraptable}[16]{r}{0.4\textwidth}
\centering
\vspace{-\baselineskip}
    \caption{Number of pairs from the whole \texttt{BREC} dataset distinguished via structural diversity or diversity curves. Compared to representative results from \citep{wang2023empirical,southern2023curvature, ballester2025expressivity}.}
\centering
\resizebox{0.4\textwidth}{!}{  
\begin{tabular}{lc}
\toprule
 Method  & All (400) \\
 \midrule
Diversity Curves & 203 (50.8\%) \\
Div(G) & 76 (19.0\%)\\
\midrule
3-WL \citep{wang2023empirical} & 270 (67.5\%)\\
1-WL \citep{southern2023curvature} & 0 (0.0\%)\\
\midrule
Graphormer \citep{wang2023empirical}  &  79 (19.8\%)\\
EdgePool \citep{lachi2025expressive}  & 23 (5.8\%)\\
PH Vietoris-Rips from \citep{ballester2025expressivity} & 84 (21.0\%) \\
NetLSD \citep{tsitsulin2018netlsd,rozembereczki2020karateclub} & 213 (53.2\%)\\
FEATHER  \citep{rozemberczki2020characteristic,rozembereczki2020karateclub} & 103 (25.8\%)\\
\bottomrule
\end{tabular}}
\label{tab:brec_main}
\end{wraptable}

To support our theoretical expressivity analysis, we investigate whether our methods can distinguish pairs of WL-indistinguishable graphs from the \texttt{BREC} datasets \citep{wang2023empirical, lachi2025expressive}. The results in \Cref{tab:brec_main} confirm that diversity curves improve the expressivity of structural diversity measures through graph coarsening. 
On its own, structural diversity based on shortest-path distances  can distinguish $19\%$ of graphs in \texttt{BREC} through their absolute difference in spread within an error tolerance of $10^{-5}$. Further, by comparing diversity curves evaluated for all possible ways of contracting one edge, we increase the proportion of graph pairs that can be distinguished to $50.8\%$. This demonstrates the practical utility of using edge pooling for increasing expressivity. Further, these insights highlight that diversity curves  successfully distinguish $1$-WL indistinguishable graphs and are comparable in expressivity on the \texttt{BREC} dataset to NetLSD while being clearly more expressive than persistent homology with Vietoris-Rips filtrations and FEATHER. Comparisons with further baselines and detailed results for the different classes of graphs in \texttt{BREC} can be found in \Cref{app:expressiv_experiment}. 

\subsection{Extended results and ablation studies}
\label{sec:extended}

In addition to our main results, we empirically evaluate the computational efficiency of our methods on large graphs in \Cref{app:empirical_costs}, noting that through coarsening, diversity curves remain scalable to graphs up to $2^{12}$ nodes. \Cref{app:extenden_large_datasets} shows classification results on \texttt{MCF-7} and \texttt{MOLT-4}, two large molecular graph datasets from TU datasets \citep{morris2020tudataset}, across which diversity curves show competitive performance compared to the unsupervised baseline methods evaluated throughout our work.

\begin{wraptable}[12]{r}{5.8cm}
\centering
\caption{Effect of the number of evaluation scales ($|\mathcal{I}|$) on runtime and accuracy.}
\label{tab:tradeof_ev_scales}
\begin{tabular}{rrr}
\toprule
$|\mathcal{I}|$ & Runtime (s) & Accuracy \\
\midrule
29 & 7.81 & 0.904 $\pm$ 0.059 \\
15 & 4.85 & 0.904 $\pm$ 0.067 \\
10 & 3.76 & 0.908 $\pm$ 0.058 \\
6  & 2.92 & 0.888 $\pm$ 0.065 \\
3  & 1.99 & 0.883 $\pm$ 0.069 \\
2  & 1.87 & 0.858 $\pm$ 0.060 \\
\bottomrule
\end{tabular}
\end{wraptable}

%
%
Further extending the results from \Cref{tab:sim_var_dist}, we 
observe that diversity curves clearly outperform simple structural baselines 
at distinguishing graph distributions as detailed in \Cref{app:simulated_graphs}, verifying that diversity curves capture information beyond basic graph statistics. 
To clarify the impact of different parameter choices on the performance of our methods, we provide further ablations, described in more detail in \Cref{app:ablation}. First, we investigate the impact of the selected coarsening strategy and graph metric on the experiment for distinguishing parameter choices as reported in \Cref{fig:sim_silh_param}. The results in \Cref{app:ablation_metric} show that diversity curves computed using shortest-path distances and random edge contraction coarsening reach high silhouette scores compared to alternative metric choices, justifying the utility of our default choice. 
Next, we investigate the influence of varying the number of evaluation scales on the experiment on distinguishing graph distributions from \Cref{tab:sim_var_dist}, noting that when we decrease the number of scales runtime decreases and accuracy remains stable, only decreasing for lower numbers of scales as reported in \Cref{tab:tradeof_ev_scales}. We further investigate the interplay between the up- and downsampling strategies by considering the operations individually in \Cref{app:up_downsampling} and find that 
the strength of our methods lies in being able to combine both up- and downsampling strategies to compare graphs in a dataset at the same cardinalities. Finally, in \Cref{app:ablation_stability}, we vary the number of edge contraction sequences over which we evaluate diversity curves for this experiment, noting that task performance and the diversity curves themselves are reasonably stable across different edge sequences. 



\section{Discussion}
\label{sec:discussion}

By introducing diversity curves as size-aware graph representations, our work contributes towards a unified framework for graph structural comparisons through coarsening. Diversity curves provably improve expressivity through edge contraction pooling and exhibit top performance in geometric tasks and unsupervised settings. 
Nevertheless, our approach has certain \textit{limitations}: We focus on structural comparisons rather than attributed graphs and do not aim to reach state-of-the-art performance in supervised settings. 
Instead, diversity curves offer interpretable and expressive tools for \emph{unsupervised} structural and distributional comparisons between graphs. Moreover, our methods implicitly rely on the quality and utility of the selected graph coarsening procedure. 
For example, some of our theoretical results assume we compute diversity curves exactly for all possible choices of edge contractions, which may not always be feasible. 
Building upon our work, there is potential to 
generalise the choice of coarsening process, or the graph invariant to track over the coarsening process. 
In addition, there is the possibility to change the underlying metric for diversity computations as discussed in \Cref{app:metrics} and used in \Cref{sec:mantra}.
Finally, diversity curves could be integrated into supervised graph learning architectures as expressive structural encodings. 

\label{sec:limitations}




\begin{ack}
This work has received funding from the Swiss State Secretariat for Education, Research, and Innovation~(SERI). G.\,W.\ was partially funded by Humboldt Research Fellowship, CIFAR AI Chair, NSERC Discovery grant 03267, FRQNT grant 343567, and NSF grant DMS-2327211. The content provided here is solely the responsibility of the authors and does not necessarily represent the views of the funding agencies. The authors declare \emph{no} competing interests. The funders had no role in the preparation of the manuscript or the decision to publish.
\end{ack}

\clearpage

\bibliography{main}
\bibliographystyle{abbrvnat}

\newpage
\appendix
\onecolumn

\counterwithin*{figure}{part}
\stepcounter{part}
\renewcommand{\thefigure}{S.\arabic{figure}}

\counterwithin*{table}{part}
\stepcounter{part}
\renewcommand{\thetable}{S.\arabic{table}}

\startcontents
\printcontents{}{1}{{%
    \vskip10pt\hrule
    \large\textbf{Appendix~(Supplementary Materials)}\vskip3pt\hrule\vskip5pt}
}

\clearpage

\section{Changelog}
We provide a brief list of changes we implemented after the reviewing process. Overall, we also improved the text with minor stylistic choices and fixing of typos. 

\begin{itemize}
    \item 
Changes in the main text:
    \begin{itemize}
        \item In \Cref{sec:intro}, we clarify the contributions of our method.
        \item \Cref{sec:relatedwork} has been rewritten with the feedback of the reviewers in mind to strengthen the context and novelty of our work with respect to existing work.
        \item We add clarifications in \Cref{sec:methods}, i.e.\ we argue what properties are needed in the coarsening procedure and rephrase the formal exposition of the edge contraction sequence to make it more concise.
        \item We redefine the definition of a diversity curve in \Cref{sec:methods} to be the average over a multiset of graph coarsening sequences, unifying the way we discuss diversity curves. We reformulate \Cref{sec:theory} to reflect this change. We further address the discussion about the theoretical and practical evaluation of diversity curves brought up during the review process. 
        \item In \Cref{sec:methods}, we remove the paragraph ``Applications to graph-level tasks'' by moving the key points to ``Diversity curves as size-aware graph representation.''
        \item The theoretical time complexity computation in \Cref{sec:costs_main} has been made more concise.
        \item We clarify in \Cref{sec:theory} that the setting for  \Cref{thm:coarsening_expressivity} considers graphs of the same cardinality.
        \item We discuss the possibility to use parallelisation to increase the computational efficiency in \Cref{sec:costs_main}.
        \item \Cref{sec:sim} clarifies how the silhouette scores have been computed.
        \item In \Cref{sec:single}, we add comparisons to feature-based baselines in order to assess what information the features only provide.
        \item We specify the number of nodes of the graphs for the experiment on geometric shapes in \Cref{sec:mantra}.
        \item We have included a short version of the results table for \Cref{sec:expressivity}. 
        \item We add a brief description of extended results and ablation experiments in \Cref{sec:extended}. 
    \end{itemize}
    
    \item Changes in the appendix:
    \begin{itemize}
        \item In \Cref{app:coarsening}, we give a more detailed discussion on what properties we require in the coarsening operation and why these requirements lead us to choose edge contraction pooling. We detail in \Cref{thm:collapse} for which theoretical results the requirements are needed.
        \item We include a more detailed discussion on the choice of metric in \Cref{app:metrics}.
        \item We make \Cref{app:extended_methods} and \Cref{app:theoretical results} consistent with the adapted version of the definition of diversity curves. In \Cref{app:theoretical results}, we include a section discussing the new formulation of ``the diversity curves over all possible edge contractions'' and why it was chosen over the previous definition of ``the multiset of diversity curves over all possible edge contractions.''
        \item We fix some typos in \Cref{app:theoretical results} in the proof of \Cref{thm:collapse}.
        \item We add structural baseline comparisons for distinguishing simulated graph distributions in \Cref{app:simulated_graphs}. 
        \item In \Cref{app:sc}, we add comparisons to feature-based baselines in order to assess what information the features only provide.
        \item We emphasise in \Cref{app:empirical_costs} how we can flexibly define the cardinalities at which we compute the diversity in order to compute diversity curves of larger graphs efficiently and add an experiment for distinguishing graph distributions with graphs between cardinalities ranging from 16 to 4096.
        \item For the empirical cost in \Cref{app:empirical_costs}, we include a speed-up for the empirical computational complexity of diversity curves.
        \item We add further ablations to \Cref{app:ablation} as discussed during the review process. Specifically, we provide empirical stability results of diversity curves with respect to varying the contraction sequence and the size of the set over which we compute the diversity curves as discussed with the reviewers. In \Cref{app:tradeoff_ev_scales}, we add an experiment on the trade-off between the number of evaluation scales, expressiveness, and the computational cost. We include the ablation study of only using up- or only downsampling in \Cref{app:up_downsampling}.
        \item \Cref{app:extenden_large_datasets} gives additional results on the \texttt{MCF-7} and \texttt{MOLT-4} datasets.

    \end{itemize}

\end{itemize}

\section{Extended methods}
\label{app:extended_methods}

To support the introduction of diversity curves in \Cref{sec:methods}, we give a more detailed description of our proposed methods, as well as the steps involved in computing diversity curves, and the outline of the final algorithm. Further, we briefly discuss the compute resources used for our experiments, and give a brief introduction to unsupervised baseline methods, with which we compare diversity curves throughout our main experiments.

\subsection{Diversity curves}
\label{app:diversity_curves_methods}
In the following, we define how to compute diversity curves in detail, describe the algorithm used to compute them, and discuss specific parameter choices, such as the choice of pooling operation and the choice of graph metric. 

\subsubsection{Coarsening operations}
\label{app:coarsening}
Throughout our work, we explore edge contraction pooling to downsample the size of graphs. Further, we use upsampling operations through upsampling nodes in a graph or triangles in a mesh. These operations are detailed below.
The coarsening operator used to compute diversity curves should satisfy the following conditions, previously summarised in \Cref{sec:methods}:
\begin{enumerate}[label=\roman*)]
    \item The coarsening operator is a function $\mathrm{Coarse}(G, i)$ that takes as inputs a graph $G$ with $c$ connected components and cardinality $n$ and any integer $i$ such that $c \leq i \leq n$ and 
    outputs a graph with cardinality $|\mathrm{Coarse}(G, i)| = i$.
    \item The coarsening is iterative. That is
    for any $i < j < |G|$ we can obtain the coarsening to $i$ by first coarsening to $j$, concretely $\mathrm{Coarse}(G,i) = \mathrm{Coarse}(\mathrm{Coarse}(G,j),i)$.
    \item The coarsening is applicable to unattributed graphs.
    \item The coarsening is unsupervised.
    \item The coarsening process does not change the connected components of the graph. That is, for two nodes in a graph $G$ of the same connected component, their respective supernodes in the pooled graph belong to the same connected component.
\end{enumerate}
Note that these requirements allow us to consider specific sequences of sequentially coarsened graphs and control precisely the cardinality in order to evaluate the diversity curves. Edge contraction coarsening as described in more detail below has been studied in previous works \citep{limbeck2025geometryaware} and satisfies all of the conditions above. 
In contrast, many other coarsening operations on graphs fail to meet these requirements. Pooling layers as for example discussed by \citet{grattarola2022understanding}, might either not directly induce a stepwise sequence of coarsened graphs, or even if they produce a stepwise coarsening hierarchy not enable precise control over the number of nodes in the coarsened graphs. 
Methods that cluster nodes, such as Graclus or NMF, 
do not directly allow us to iteratively pool graphs to any desired size (requirement i and ii). Other operations that delete nodes, such as TopK or SAGPool, might disconnect the graph (requirement v). Further, we explicitly choose not to apply supervised pooling methods (requirement iv), excluding a wide range of trainable operations, such as TopK, SAGPool, DiffPool or MinCutPool. Hence, considering related work on graph pooling and graph coarsening methods, we do find that the iterative edge-coarsening approach detailed below is particularly suitable for the purpose of our work.

\paragraph{Edge pooling}
\label{app:methods edge pooling}
We downsample graphs by performing edge contractions closely following the pooling method proposed by \citet{limbeck2025geometryaware}. For a graph $G=(X,E)$ and an edge $e = (u,v) \in E$, we define the graph $G/e$ obtained from $G$ by contracting the edge $e$ to have vertex set $X \setminus \{v\}$ and edge set $E \setminus \{(v, x)|x\in \mathcal{N}_G(v)\} \cup \{(u,x)| x \in \mathcal{N}_G(v)\}$, where $\mathcal{N}_G(v) = \{x \in G|(v,x) \in E\}$ denotes the neighbourhood of the vertex $v$ in $G$. Note that we do not introduce any multiple edges or a self-loop on the remaining node $u$. 
To describe our downsampling process, let $G=(X,E)$ be a graph of size $n = |X|$ with $m=|E|$ edges that we want to downsample to the size $n-k$, where $0 < k \leq m$. To do this, we generate a sequence $\mathcal{E} = (e_i)_{1\leq i \leq k}$ of edges in $G$ to be contracted in that order. We choose this edge sequence based on edge scores that are either drawn uniformly at random from the unit interval 
for each edge, or we use edge scores based on a diversity measure as proposed by \citet{limbeck2025geometryaware}. In the latter case, the score for an edge $e \in E$ is the difference in diversity $|\Div(G) - \Div(G/e)|$. After the scores are computed, we iteratively add the edge with the lowest score to the sequence $\mathcal{E}$ if that edge is not adjacent to a previously contracted edge. If there are no more such edges, we recompute all edge scores and repeat the procedure. 
With the generated edge sequence, we define the sequence of coarsened graphs $(G_i)_{n-k \leq i \leq n}$ where $G_n = G$ and for $n-k < i \leq n$ we set $G_{i-1} = G_{i}/e_{n-i+1}$. If node features are available, we average the node features of all original nodes that have been merged through these edge contractions. Further, we note that edge contraction pooling extends to mesh data and higher order shapes and is in fact often used to coarsen mesh data in practice \citep{ollivier2003coarsening}. Hence, we also apply edge contraction pooling to the triangulations in \Cref{sec:mantra} while taking care to track triangles as well as edges through the coarsening process. Edge contraction pooling is thus used throughout all of our experiments.

\paragraph{Node upsampling}
For upsampling a graph $G = (X,E)$ of size $n=|X|$ to size $N > n$, we generate a sequence $\mathcal{V} = (v_i)_{1 \leq i \leq N-n}$ of $N-n$ nodes from $G$. This sequence determines the nodes that we upsample and the order in which we upsample them. We construct the sequence in such a manner that the difference between the number of times nodes are selected is no more than $1$. To do this, let $m,r \in \mathbb{Z}_{\geq 0}$ be such that $N-n = m\cdot n + r$ and choose the first $m\cdot n$ nodes in $\mathcal{V}$ by doing $m$ permutations of all nodes in the graph. For the remaining $r$ nodes we choose them from the set of nodes $X$ without replacement. This ensures each node is picked at least once before a previously upsampled node is chosen again. 
Going through the sequence $\mathcal{V}$ of nodes to upsample, for every $0 \leq i \leq N-n$ we add a new node $u$ with the same neighbourhood as the sampled node $v_i$ and also connect the sampled node to the new node. That is, we add an edge between the new node $u$ and every node in $\mathcal{N}_G(v_i) \cup \{v_i\}$. If node features are available, we set the node feature of $u$ to be the same as the node feature of $v_i$. This node-based upsampling procedure is used throughout all of our experiments other than in \Cref{sec:mantra}. Note that the node-based upsampling strategy aims to mimic and reverse the way edge contraction pooling downsamples nodes in a graph. 

\paragraph{Triangular upsampling} Given a triangular mesh, such as in \Cref{sec:mantra}, we might not want to just upsample nodes in the graphs as this loses information on the inherent shape of the data. Instead, we employ triangular upsampling and upsample each 3-clique in the graph, or each triangle in the triangulation. Specifically, we uniformly select triangles making sure that each triangle is picked at least once before a previously upsampled triangle could be selected again. We then subdivide the selected triangles in order by adding a central vertex and connecting it to all corners. This way, we uniformly subdivide the surface of the meshes ($\mathcal{T}$), or the 3-cliques in the graphs ($\mathit{G}$). This upsampling strategy is used to compute diversity curves in \Cref{sec:mantra} and account for the triangular nature of the data.


\subsubsection{Metric choices}
\label{app:metrics}
To compute diversity curves, we consider the structural diversity of a graph to be dependent on a metric between nodes. Here, we will define and detail the distance choices explored in our paper. First, we define shortest-path distances, which are our default choice across this paper, as they are faster to compute than other structure-based metrics, such as heat kernel or diffusion distances, while performing well across our experiments and ablation studies.


\begin{definition}[Shortest-path distance]
    Given an undirected graph $G=(X, E)$, the \emph{shortest-path distance} $d$ is a metric that counts the minimum number of edges in any path between two nodes, $x_i, x_j \in X$, so that $d(x_i, x_j)=\infty$ if no such path exists and otherwise
    \begin{equation}
        d(x_i, x_j) := \min_k\{k \in \mathbb{N} | \exists \text{ path } x_i = v_0, v_1, \dots, v_k = x_j \text{ where } (v_l, v_{l+1}) \in E \text{ and } v_l \in X\}.
    \end{equation}
\end{definition}


When dealing with attributed graphs we can further define distances based on node features:
\begin{definition}[Euclidean distance]
    Given a graph $G=(X, E)$ associated with features $\mathbf{Y} \in \mathbb{R}^{|X| \times M }$ the \emph{Euclidean distance} $d$ between two nodes $x_i, x_j \in X$ with feature vectors $\mathbf{y_i},\mathbf{y_j} \in \mathbb{R}^M$ is computed as 
    \begin{equation}
        d(x_i, x_j) := ||\mathbf{y_i} - \mathbf{y_j}||_2 = \sqrt{\sum_{k=1}^{M} (\mathbf{Y}_{ik} - \mathbf{Y}_{jk})^{2}}.
    \end{equation}
\end{definition}


Next, we define alternative structure-based metrics, which are computed from a graph's normalised Laplacian and have been used throughout related literature to study the geometry of graphs \citep{coupette2025no, limbeck2025geometryaware}. Let $G=(X, E)$ be a graph, $A$ its adjacency matrix, and $D$ its degree matrix, whose diagonal entries are given by $D_{ii}=\sum_{j=1}^{|X|}A_{ij}$. The normalised graph Laplacian of $G$ is given by $\hat{L}=D^{-\frac{1}{2}}(D-A)D^{-\frac{1}{2}}$. 

\begin{definition}[Diffusion distance]
    Given a graph $G=(X, E)$, its normalised Laplacian, \(\hat{L}\), has positive eigenvalues \(2=\lambda_0 > \lambda_1 > \lambda_2 > \cdots > \lambda_{n-1} \geq 0\) with eigenvectors \(\lbrace \psi_l\rbrace_{l}\). At time $t$ this provides an embedding in Euclidean space given by 
    \begin{equation}
        \Phi_t(x) = (\lambda_1^t \psi_1(x), \cdots, \lambda_{n-1}^t \psi_{n-1}(x)) \text{ for } x\in X.
    \end{equation}
    The \emph{diffusion distance} $d$ between two nodes $x_i, x_j \in X$ at time $t$ is defined as 
    \begin{equation}
        d(x_i, x_j) := ||\Phi_t(x_i) - \Phi_t(x_j)||_2.
    \end{equation}
\end{definition}

\begin{definition}[Heat kernel distance]
    Given a graph $G=(X, E)$, its normalised Laplacian, \(\hat{L}\), has positive eigenvalues \(2=\lambda_0 > \lambda_1 > \lambda_2 > \cdots > \lambda_{n-1} \geq 0\) with eigenvectors \(\lbrace \psi_l\rbrace_{l}\). The \emph{heat kernel distance} $d$ between two nodes $x_i, x_j \in X$ at time $t$ is defined as 
    \begin{equation}
        d(x_i, x_j) := \sum_{l=0}^{n-1} \exp(-\lambda_lt\psi_l(x_i)\psi_l(x_j)).
    \end{equation}
\end{definition}

Unless otherwise stated, we set the time parameter to $t=1$ throughout our experiments. 
Further, we define a generalisation of the heat kernel distances to higher-dimensional shapes using the Hodge Laplacian on a simplicial complex \citep{krahn2025heat}. We use existing code for computing the Hodge Laplacian.\footnote{\href{https://github.com/tsitsvero/hodgelaplacians}{https://github.com/tsitsvero/hodgelaplacians} available under an MIT licence.}

\begin{definition}[Hodge heat kernel distance]
    Given a simplicial complex $\mathcal{C}$, its $k$-th Hodge Laplacian, \(\hat{L}\), has eigenvalues \(\lambda_0 > \lambda_1 > \lambda_2 > \cdots > \lambda_{s-1}\) with eigenvectors \(\lbrace \psi_l\rbrace_{l}\). We define the \emph{topological heat kernel distance} $d$ between two $k$-simplices $c_i, c_j \in \mathcal{C}$ at time $t$ as 
    \begin{equation}
        d(c_i, c_j) := \sum_{l=0}^{s-1} \exp(-\lambda_lt\psi_l(c_i)\psi_l(c_j)).
    \end{equation}
\end{definition}

Different choices of metric allow us to explore different notions of geometry and structural diversity on graphs. In practice, we evaluate this choice of metrics across multiple experiments. We choose shortest-path distances as default choices due to their good empirical performance at distinguishing simulated graph distributions (\Cref{fig:ablation_silh}) and scalability compared to eigenvalue-based distances (\Cref{app:computational_complexity}). In other applications, we find that a combination of shortest-path and feature-based Euclidean distances performs best at distinguishing the geometry of single-cell graphs (\Cref{tab:sc_all_app}), that a combination of shortest-path and higher-order heat kernel distances can distinguish meshes (\Cref{tab:mantra}), or that diffusion-based distances can be more expressive than shortest-path distances despite being more computationally expensive (\Cref{tab:brec}). Overall, while there are advantages and disadvantages to each distance choice, the shortest-path metric gives a reasonable and generalist summary of graph structure, while alternative choices can reveal different geometries on the graph. In practice, one can select the default shortest-path metric to make general structural comparisons across graphs, explore application-specific distance choices, or automatically learn the distance choices that perform best for a specific application.



\subsubsection{Choice of coarsening levels}



Let $\mathcal{D}$ be a set of graphs for which we want to compute the diversity curves and compare them between each other. We select a set $\mathcal{I}$ of integers, which will represent the cardinalities at which we evaluate the diversity curves. That is, for any graph $G=(X,E) \in \mathcal{D}$ in our dataset we construct a sequence $(G_i)_{i \in \mathcal{I}}$ of graphs with $|G_i| = i$ based on the edge contraction and node upsampling described above and compute the diversity measure for every graph in the sequence $(G_i)_{i\in \mathcal{I}}$ to obtain the diversity curve. 
By default, we choose the cardinalities $\mathcal{I} = \{1,\dots, \max_{G \in \mathcal{D}}|G|\}$ in our experiments. Note that for a graph $G$ with precisely $c$ connected components, we can not coarsen the graph to less nodes than $c$ by edge contractions alone. In the case where $c>1$, to still obtain values for every cardinality of the diversity curve, we linearly interpolate the diversity curve between $\DivCurve(G)_1 = 1$ and the lowest $i \in \mathcal{I}$ to which we could downsample the graph by edge contractions.
This setup allows for a lot of flexibility, in particular, for datasets in which there are graphs of large cardinalities for which it is not feasible to compute the diversity measure directly, we can choose the maximum cardinality to be much smaller than the largest cardinality in the dataset. Furthermore, we can also choose to evaluate the diversity curves only at certain cardinalities if we are not interested in the finest possible resolution.

\subsubsection{Algorithm}


We base our computations of structural diversity on \texttt{magnipy}\footnote{Available at \hyperlink{https://github.com/aidos-lab/magnipy}{https://github.com/aidos-lab/magnipy} under a BSD 3-Clause Licence.}, a Python library for the computation of metric space diversity measures \citep{limbeck2024metric}, and \texttt{mag\_edge\_pool}\footnote{Available at \hyperlink{https://github.com/aidos-lab/mag_edge_pool}{https://github.com/aidos-lab/mag\_edge\_pool} under a BSD 3-Clause Licence.}, which provides code for computing the spread of a metric space on graphs, and for computing edge contraction pooling operations \citep{limbeck2025geometryaware}. For using and manipulating graphs, we use the Python package NetworkX~\citep{hagberg2008networkx} released under a BSD-3-Clause Licence. Then, we implement the computations of diversity curves as described in \Cref{sec:methods} via \Cref{alg:diversity_estimation}. 
A code implementation of our methods is available on GitHub.\footnote{The code for computing diversity curves is available at \href{https://github.com/aidos-lab/diversity_curves/}{https://github.com/aidos-lab/diversity\_curves/}.}




\begin{algorithm}
\caption{Diversity Curve Computation}
\label{alg:diversity_estimation}
\begin{algorithmic}[1]
\REQUIRE Graph $G = (X, E)$; evaluation scales $\mathcal{I} = \{n_{\min}, \ldots, n_{\max}\}$; upsampling method $\mathrm{Up}$; downsampling method $\mathrm{Coarse}$; distance function $d(\cdot, \cdot)$; diversity measure $\text{Div}(\cdot, d)$; number of repetitions $R \in \mathbb{N}$
\ENSURE Diversity curve $\mathbf{v} \in \mathbb{R}^{|\mathcal{I}|}$
\STATE $\mathbf{v} \leftarrow \mathbf{0}^{|\mathcal{I}|}$, \quad $n \leftarrow |X|$
\FOR{$r = 1, \ldots, R$}
    \STATE $\mathcal{I}_{\uparrow} \leftarrow \{i \in \mathcal{I} : i > n\}$ in ascending order
    \STATE $\mathcal{I}_{\downarrow} \leftarrow \{i \in \mathcal{I} : i \leq n\}$ in descending order
    \STATE $G_{\uparrow} \leftarrow G$
    \FOR{each $i \in \mathcal{I}_{\uparrow}$}
        \STATE $G_{\uparrow} \leftarrow \mathrm{Up}(G_{\uparrow}, i)$ \COMMENT{Upsample to $i$ nodes}
        \STATE $\mathbf{v}[i] \leftarrow \mathbf{v}[i] + \text{Div}(G_{\uparrow}, d)$
    \ENDFOR
    \STATE $G_{\downarrow} \leftarrow G$
    \FOR{each $i \in \mathcal{I}_{\downarrow}$}
        \STATE $G_{\downarrow} \leftarrow \mathrm{Coarse}(G_{\downarrow}, i)$ \COMMENT{Coarsen to $i$ nodes}
        \STATE $\mathbf{v}[i] \leftarrow \mathbf{v}[i] + \text{Div}(G_{\downarrow}, d)$
    \ENDFOR
\ENDFOR

\STATE $\mathbf{v} \leftarrow \mathbf{v} / R$
\STATE  \textbf{return} $\mathbf{v}$
\end{algorithmic}
\end{algorithm}

As default settings, we select the following choices: We set the evaluation interval $\mathcal{I}=\{1, 2, \dots, n_{\text{max}}\}$ with $n_{\text{max}}$ being the maximum number of nodes in a given graph dataset. Further, we use random edge contraction pooling to downsample the graph, and random node-based upsampling to upsample small graphs to larger sizes as described in \Cref{app:coarsening}. For diversity computations, we choose the shortest-path metric between nodes as a distance function, and the spread of a graph for computing the diversity. Further, we compute the diversity curves by default over a set of three random coarsening sequences (that is, $R=3$)
, unless specified otherwise. Finally, whenever the number of connected components in a graph is larger than the minimum evaluation scale, we use linear interpolation to interpolate the value of diversity as described in \Cref{sec:methods}.


\subsection{Compute resources}
\label{app:compute}
We run our experiments in a compute node with $\times 2$ CPU \textsc{AMD EPYC 7763 64-Core Processor} and  1 TiB \textsc{DRAM}. We use these resources to speed up computation, however running experiments on a personal computer is also possible and does not represent a computational bottleneck.

\subsection{Baseline methods}

For computing the graph embeddings NetLSD \citep{tsitsulin2018netlsd}, FEATHER \citep{rozemberczki2020characteristic}, and Graph2Vec \citep{narayanan2017graph2vec} we use the `Karate Club' Python package~\citep{rozembereczki2020karateclub} published under the MIT licence (GPLv3).
The SpectralZoo embeddings we compute via the approximation used by~\citet{jin2020spectral}\footnote{Made available at \hyperlink{https://github.com/shengminjin/EstimateSpectralMoments}{https://github.com/shengminjin/EstimateSpectralMoments} with the publication of their paper \citep{jin2020spectral}. 
}.
For the computation of the kernel methods we use the Python package GraKel (available under a BSD 3-clause Licence)~\citep{siglidis2020grakel}.
For the computation of the TopER\footnote{Available at \hyperlink{https://github.com/AstritTola/TopER}{https://github.com/AstritTola/TopER} under the MIT Licence} method we use the implementation that has been made available with the publication \citep{tola2025toper}.
We use the same approach for the persistent homology based graph embeddings as described in the classification experiment by \citet{ballester2025expressivity} and use their provided implementation.\footnote{Available at \hyperlink{https://github.com/aidos-lab/PH\_expressivity}{https://github.com/aidos-lab/PH\_expressivity} under a BSD 3-clause Licence.}

\section{Proofs and theoretical results}
\label{app:theoretical results}


Here, we give an overview of all relevant theorems and proofs that detail and support our theoretical results. Specifically, we detail the properties of diversity curves in terms of their computational complexity
and their expressivity at distinguishing graphs.




\subsection{Computational complexity} 
\label{app:computational_complexity}

Extending our analysis of the time complexity of diversity curve computations in \Cref{sec:theory}, we outline the computational complexity of the specific steps involved in computing our methods for one edge contraction sequence. Note that the cardinality of the multiset of coarsening sequences over which we compute the diversity curves is constant in practice. Hence, we do not take the averaging over repeated evaluations into account for reporting the total time complexity.






\paragraph{Computing the diversity measure}
Let $C_{Div}$ be the time complexity for computing the diversity measure of a graph $G=(X,E)$, we conduct our analysis with the spread as it was used in this work. The complexity $C_{Div}$ consists of the time complexity $C_{dist}$ for computing the pairwise distances between all vertices in the graph, and computing the spread. We use the shortest-path distance, which has time complexity $\mathcal{O}(|X|^3)$. In practice, shortest-path distances can be more efficiently computed than alternative graph distances, such as diffusion distances defined in \Cref{app:metrics} and previously used by \citet{limbeck2025geometryaware}. 
Further, computing the spread has time complexity $\mathcal{O}(|X|^2)$. This is notably faster than computing alternative diversity measures, such as the magnitude of a graph, which has a time complexity $\mathcal{O}(|X|^3)$. 
Overall, for a general distance we get $C_{Div} = C_{dist}+\mathcal{O}(|X|^2)$, which in the case of shortest-path distance gives $C_{Div} = \mathcal{O}(|X|^3)$ using Floyd-Warshall
algorithm. 
For our work, we choose to use exact computations of both metric space spread and shortest-path distances as described above to ensure computational accuracy. Nevertheless, we note that 
costs could further be reduced by approximating the shortest-path distance  
in around $C_{dist}=O(|X|^{2.032 })$ time using a 2-approximation of all shortest-paths \citep{bermanr2007distapprox, dory2024fast}.
Further, the costs of spread could be reduced via iterative normalisation or mini-batching using $s$ subsets of size |$X_s$| reducing the time complexity of approximating spread to $C_{Div} = C_{dist}+O(s \cdot |X_s| \cdot |X|)$ \citep{dunne2024efficiently, limbeck2025geometryaware}. 

\paragraph{Coarsening the graph} 
Given a graph $G=(X,E)$, let $C_{Coarse}$ be the time complexity for computing the sequence of coarsened graphs $(G_i)_{i \in \mathcal{I}}$ for a fixed set $\mathcal{I}$ of integers representing the cardinalities, so $|G_i|=i$ for any $i \in \mathcal{I}$. We denote by $n_{max} = \max_{i\in \mathcal{I}} i$ the maximum cardinality at which we compute the diversity and $n_{min} = \min_{i\in \mathcal{I}} i$ the minimum cardinality. We obtain the coarsened graphs by doing repeated edge contractions in $G$ and upsampling nodes. The time complexity $C_{Coarse}$ is composed of the time for computing the edge scores, sorting the edges according to their scores, and contracting the edges. 
\begin{description}[left=1em]
    \item[Computing the edge scores] If we choose random edge scores, we need $\mathcal{O}(|E|)$ to get all of the edge scores. 
    When using spread-based scores for determining the order of the edge contractions, it takes $\mathcal{O}(|E|\cdot C_{Div})$ time to compute all scores once \citep{limbeck2025geometryaware}.
    \item[Sorting the edges] Sorting the edges according to their scores takes $\mathcal{O}(|E|\log(|E|))$ time.
    \item[Contracting the edges]  Contracting one edge has time complexity $\mathcal{O}(|X|)$ and we do this at most $\min(|E|,|X|-n_{min})$ times.
    \item[Repeating the edge scoring] We might have to repeat the edge score computation and sorting of the edges since we restrict the edge contractions such that no two adjacent edges get contracted in the same iteration of the process of computing edge scores, sorting the edges, and contracting. Let $r$ denote the number of times we have to recompute the edge scores and sort them. This number depends on $n_{min}$ and how many edges can be contracted per iteration. If we want to coarsen the graph such that every edge is contracted, the number of iterations is $r=\log(n/\min(|E|,|X|-n_{min}))$ in the best case for which we can halve the size of the graph in each iteration, and $r=\min(|E|,|X|-n_{min})$ in the worst case where we can contract only one edge per iteration. 
    \item[Upsampling a node]  We might also have to upsample graphs. Upsampling one node in a graph of cardinality $n$ has time complexity $\mathcal{O}(n)$. We have to upsample $u=\max(0,n_{max} - |X|)$ nodes for a graph size of at most $n_{max}$.
\end{description}

Altogether, considering the steps detailed above we arrive at the following costs: 
For spread-based edge scores, we get a time complexity of \(C_{Coarse}= \mathcal{O}(r(|E|\cdot C_{Div}+|E|\log(|E|)) + |E||X| + u \cdot n_{max}) = 
\mathcal{O}(r|E|( C_{Div}+\log(|E|)) + |E||X| + n_{max}^2).
\)
This is more expensive than random edge selection. In fact, 
when generating and sorting random edge scores, we obtain the time complexity \(
C_{Coarse} = \mathcal{O}(r(|E|+|E|\log(|E|)) + |E||X| + u \cdot n_{max}) = 
\mathcal{O}(|E|(r\cdot \log(|E|) + |X|) + n_{max}^2)
\). 

Note that we can further speed up the coarsening process for random edge scores by not first generating random edge scores and sorting them, but rather directly randomly permuting the edges. Then we can continue as above with the ordering of the edges given by the random permutation. This will improve the time complexity for computing the sequence of coarsened graphs using random edge scores to $C_{Coarse} = \mathcal{O}(|E|(r + |X|)+ u \cdot n_{max})= \mathcal{O}(|E|(r + |X|)+ n_{max}^2)$. Note that in case $n_{max} <|X|$, we do not need to do any upsampling, and the time complexity for random edge contractions further reduces to $C_{Coarse}=\mathcal{O}(|E|(r + |X|))$.

   
\paragraph{Computing the diversity curves}
To compute diversity curves, we calculate the diversity measure at every cardinality $i \in \mathcal{I}$ for graphs with at most $n_{max}$ nodes, which takes $\mathcal{O}(|\mathcal{I}|\cdot n_{max}^3)$ time when using spread. 
In total, the time complexity of computing the diversity curves is given by  
\begin{itemize}[noitemsep,topsep=0pt,parsep=0pt,partopsep=0pt,leftmargin=*]
    \item 
$C_{Coarse} + C_{Div} = \mathcal{O}(r|E|( C_{Div}+\log(|E|)) + |E||X|+|\mathcal{I}|\cdot n_{max}^3)$ using spread-based edge scores,

\item $C_{Coarse} + C_{Div} = \mathcal{O}(|E|(r\cdot \log(|E|) + |X|) + |\mathcal{I}|\cdot n_{max}^3)$ when sorting random edge scores, or 

\item $C_{Coarse} + C_{Div} = \mathcal{O}(|E|(r + |X|) + |\mathcal{I}|\cdot n_{max}^3)$ when permuting edges randomly.

\end{itemize}

Overall, we find that by using coarsening as one of their core components, diversity curves remain reasonably efficient to compute in practice. Further, our framework allows for investigating even faster coarsening methods or graph invariants in future work.

\subsection{Diversity curves collapse to the same diversity} 
In this subsection, we describe how diversity curves behave if we coarsen graphs as much as possible 
by performing consecutive edge contractions until there are no more edges to contract. This illustrates that eventually, all diversity curves will reach the same values at low cardinalities. The proof of this theorem uses that our chosen graph coarsening procedure does not change the connectivity of the graph.

\ThmCollapse*

\begin{proof}
 Consider a graph $G$ with exactly $c$ connected components. Doing edge contractions in a graph does not change its number of connected components. For any possible sequence of edge contractions, we will arrive at the graph $\overline{K_c}$, the graph on  $c$ vertices with no edges. The spread of this graph is $\Div(\overline{K_c}) = c$. In particular, a connected graph will collapse to $\overline{K_1}$, the graph on $1$ vertex, with spread $\Div(\overline{K_1}) = 1$. For graphs with more than one connected component, 
 we perform the linear interpolation, which gives us values of $\DivCurve(\overline{K_c})_i=i$ for $1 \leq i < c$. For any two graphs $G$ and $H$ with the same number $c$ of connected components we thus have $\DivCurve(G)_i = \DivCurve(H)_i = i$ for $i \in \{1,\dots c\}$. 
 Note also that since the number of connected components is $c$, for a graph with $c+1$ vertices this is possible only if there exists precisely one edge. The spread of this graph is $c-1 + \frac{2}{1+\exp(-1)}$.  
 Hence, we have shown that $\DivCurve(H)_{c+1} = \DivCurve(G)_{c+1} \neq c+1$.
\end{proof}

\subsection{Diversity curves over all possible edge contractions}
\label{app:all_possible_contractions}
For the remaining two sections we study diversity curves over all possible edge contraction sequences. That is, the multiset over which we average the diversity curves contains all possible coarsenings obtained from edge-contraction sequences as described in the paragraph below.
Theoretically, this allows us to study diversity curves without the need for a canonical labelling of the edges, a notoriously difficult problem related to the graph isomorphism problem \citep{babai1983canonical}. 
In practice, the evaluation of all possible edge contraction sequences becomes infeasible. Hence, we empirically compute diversity curves over a small number of edge contraction sequences. \Cref{app:ablation} verifies that this empirical approach is nevertheless stable with respect to the chosen edge contraction sequence and the number of evaluated sequences. 

For a graph $G$ and a sequence $s_e = (e_i)_{1 \leq i \leq k}$ of $m\geq k\geq 0$ of edges in $G$, we denote by $G/(e_i)_{1 \leq i\leq k}$ the graph obtained from $G$ after iteratively contracting the edges in the sequence $s_e$. We say that the sequence of coarsened graphs $\mathcal{G}_\mathcal{I} = (G_i)_{i \in \mathcal{I} }$ is induced by the edge sequence $(e_i)_{i \leq k}$ if $\mathcal{I} = \{n-k,\dots,n\}$ and $G_i = G/(e_j)_{1 \leq j\leq n-i}$ for any $i\in \mathcal{I}, i < n$ and $G_n = G$. We consider the multiset of all valid edge contraction sequences in the graph $G$ of length $0\leq k \leq m$ by $\mathcal{E}_k(G) = \{\{ \; (e_i)_{i\leq k} | \forall i \leq k \; e_i \in E(G/(e_j)_{j\leq i-1}\}\}$. 
In a previous version of the manuscript, we considered the multiset of diversity curves associated to each sequence of $k$ edge contractions in order to study the theoretical properties, that is $\{\{\DivCurve(\mathcal{G}_\mathcal{I}) |  \exists (e_i)_{i\in\mathcal{I}} \in \mathcal{E}_k \text{ such that } \mathcal{G}_\mathcal{I} \text{ is induced by } (e_i)_{i\in\mathcal{I}} \}\}$, denoted by the multisets of diversity curves under all possible edge contractions. From this formulation we obtain the diversity curve over all possible edge contractions under the current definition by taking for each cardinality the average with multiplicities over all values of the diversity curves appearing in the multiset in that cardinality. The current formulation allows us to treat diversity curves in a more unified way.

\subsection{Diversity curves as graph invariants}
\label{app:invariant}

\ThmInvariance*
\begin{proof}
    The statement follows immediately from the fact that the spread of a graph is invariant under graph isomorphism, and that for two isomorphic graphs  $G \cong H$ where we denote by $\psi\colon G \to H$ an isomorphism, contracting an edge $e$ in $G$ and the corresponding edge $\psi(e)$ in $H$ gives two isomorphic graphs. 
    Hence for any edge sequence $(e_i)_{i \leq k}$ of $G$ and the corresponding coarsening sequence of $G$, the edge sequence $(\psi(e_i))_{i \leq k}$ defines a sequence of coarsened graphs that are isomorphic to the coarsened graphs of $G$ for every $i \leq k$. Hence the diversity at each cardinality is equal and because of the isomorphism we obtain a bijection between the two multisets of possible edge contractions, which implies the average diversity at each cardinality is equal.
\end{proof}

Even when not considering the diversity curves over all possible edge contractions, we can state the following:

\begin{prop}
    Let $G$ and $H$ be two isomorphic graphs and $\sigma$ a multiset of edge sequences in $G$ and the induced multiset of coarsened graphs $\mathcal{G}$.
    There exists a multiset of edge sequences in $H$ with induced multiset of coarsened graphs $\mathcal{H}$ such that $\DivCurve(\mathcal{G}) = \DivCurve(\mathcal{H})$.
\end{prop}
\begin{proof}
     Let $\psi\colon G \to H$ be an isomorphism between the two graphs. 
     The multiset $\sigma_H = \{\{(\psi(e_i))_{1\leq i \leq k\}}|(e_i)_{1 \leq i \leq k} \in \sigma\}\}$ of edge sequences in $H$ ensures that for every edge sequence in $\sigma$ there is exactly one edge sequence in $\sigma_H$ where the coarsened graphs at each cardinality are equal and vice versa. Hence the diversity curve of $H$ over $\sigma_H$ equals the diversity curve of $G$ over $\sigma$.      
\end{proof}
Hence, even if we can not compute all possible edge sequences, there always exists a corresponding multiset of edge sequences such that the diversity curves are equal, and if we choose the edges in a way that is respecting an isomorphism, we will find the corresponding sequence.

\subsection{Diversity curves increase expressivity} 


   
In this section, we want to study the expressivity of the diversity curves and spread, combining properties of edge contraction coarsening and the graph invariant itself. Note that it is known that several classes of graphs can be reconstructed based on their multiset of single edge contractions~\citep{poirier2018reconstructing}, while the spread of graphs has empirically demonstrated the potential for being an expressive graph invariant \citep{limbeck2025geometryaware}. Altogether, as detailed in \Cref{sec:theory} and \Cref{app:invariant}, we consider the diversity curves over all possible edge contractions for this discussion. Further, we define what it means to increase expressivity in \Cref{def:expressivity} following existing work by  \citet{lachi2025expressive}, who previously showed that edge contraction pooling layers do not only preserve, but also increase expressivity. We do restrict to consider graphs of the same cardinality. 
Our first result gives an example which shows that tracking the spread over a coarsening of graphs increases the expressivity compared to just considering the spread at the original cardinality of the graphs.

\ThmCoarseningExpressivity*
\begin{proof}
    It is clear that diversity curves are at least as expressive as spread, meaning that for any two graphs $G$ and $H$ of the same cardinality with $\Div(G) \neq \Div(H)$, also $\DivCurve(G) \neq \DivCurve(H)$ since $\DivCurve(G)_{|G|} = \Div(G) \neq \Div(H) = \DivCurve(H)_{|H|}$ for any coarsening of $G$ and $H$ respectively, which implies that the diversity curves over all possible edge contractions differ, since every diversity curve over a coarsening of $G$ differs in at least one component from a diversity curve over any coarsening of $H$. \\
    Consider the two graphs $G_A, G_B$ in \Cref{fig:graphs_examples}, where $G_A$ is the graph on the left, and $G_B$ is the so-called house graph on the right. They have equal spread $\Div(G_A) = \Div(G_B) = \frac{2}{1+3\exp(-1)+\exp(-2)} + \frac{3}{1+2\exp(-1)+2\exp(-2)} \approx 2.39$. The one-edge contractions of $G_A$ are all isomorphic to the diamond graph $D$, the four-cycle with one diagonal edge, which has spread $\Div(D) = \frac{2}{1+2\exp(-1)+\exp(-2)} + \frac{2}{1+3\exp(-1)} \approx 2.02$, which consequently is the diversity curve of $G_A$ over all possible edge contractions at cardinality $4$. The one edge contractions of $G_B$ consist of the four-cycle $C_4$ with multiplicity $2$ and spread $\frac{4}{1+2\exp(-1)+\exp(-2)}\approx 2.14$, the diamond graph $D$ with multiplicity $3$ and the three-cycle with an additional edge to a new vertex with multiplicity $1$ and spread $\frac{1}{1+\exp(-1)+2\exp(-2)} + \frac{2}{1+2\exp(-1)+\exp(-2)} + \frac{1}{1+3\exp(-1)} \approx 2.15$, giving us an average in diversity over all one-edge contractions of $\approx 2.08 > \Div(D)$, showing that the diversity curves over all possible edge contractions differ already at cardinality $4$.
    The diagram on the right of \Cref{fig:graphs_examples} shows the different diversity curves over all possible edge contractions at each cardinality for the two graphs $G_A$ and $G_B$. 
\end{proof}

\begin{figure}[tbh]
    \centering
    \includegraphics[width=0.6\linewidth, clip, trim=0cm 0cm 0cm 1cm]{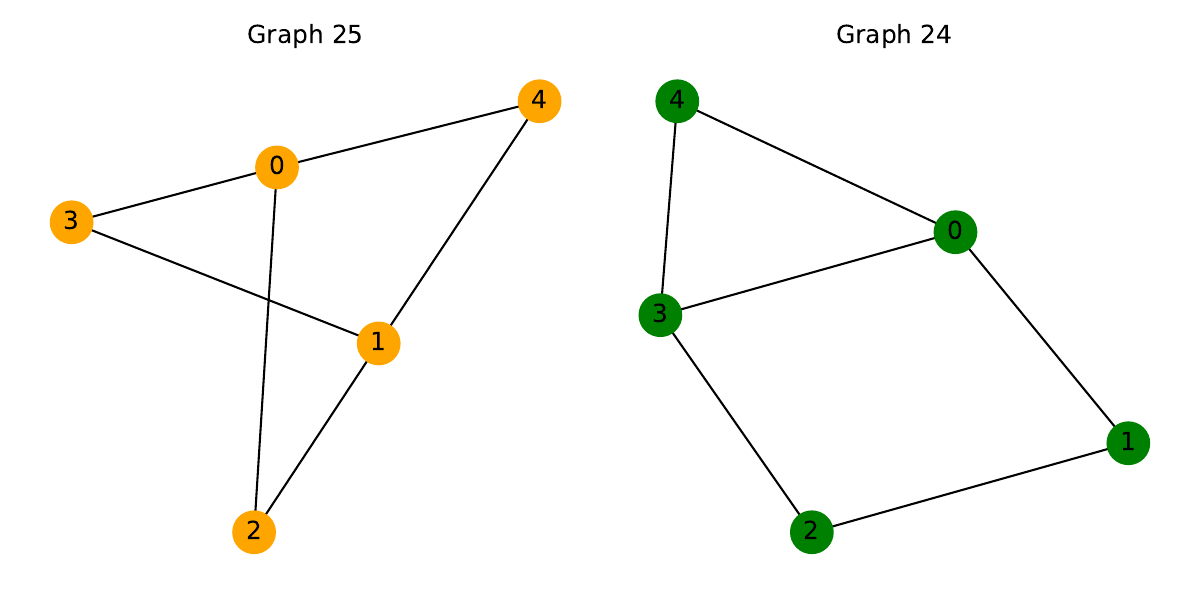}
    \includegraphics[width=0.35\linewidth]{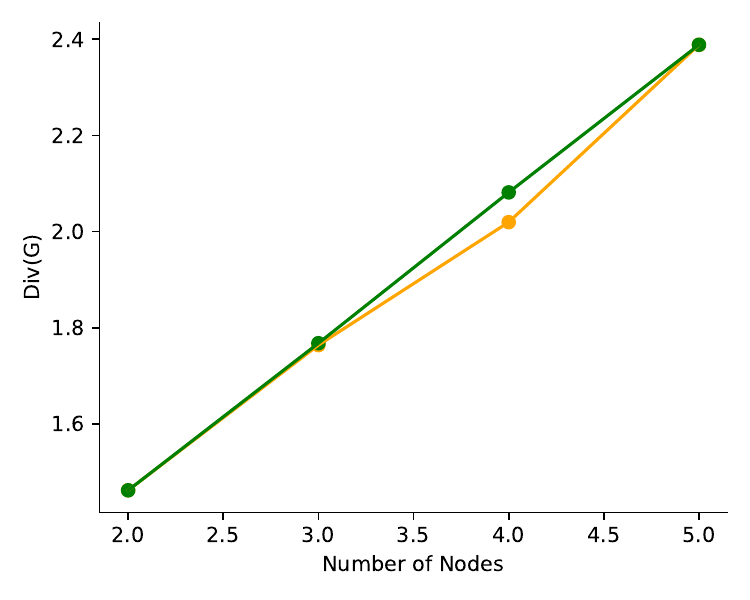}
    \caption{Graphs with identical diversity, $\Div(G)$ as computed via spread using shortest-path distances, that can be distinguished through edge pooling by their diversity curves, $\DivCurve(G)$. This example is used in  \Cref{thm:coarsening_expressivity} and \Cref{prop:wl_spread_incompatible}.}
    \label{fig:graphs_examples}
\end{figure}

\begin{figure}[tbh]
    \centering
    \includegraphics[width=0.6\linewidth, clip, trim=0cm 0cm 0cm 1cm]{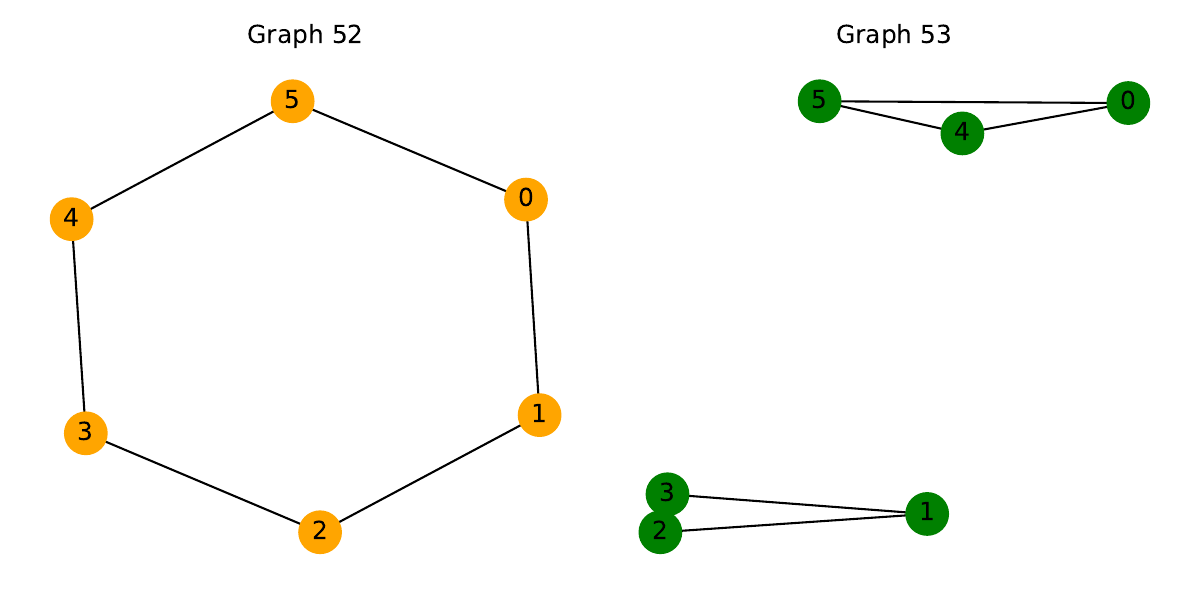}
    \includegraphics[width=0.35\linewidth]{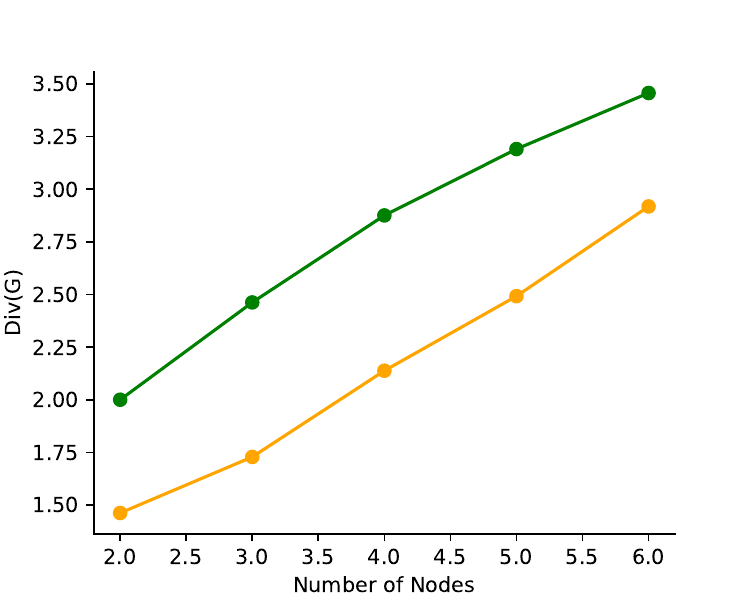}
    \caption{Graphs that are 1-WL indistinguishable, but can be distinguished through their diversity, $\Div(G)$, or by their diversity curves, $\DivCurve(G)$. This example is used in \Cref{prop:wl_spread_incompatible}.}
    \label{fig:graphs_examples_16}
\end{figure}

\begin{figure}[tbh]
    \centering
    \includegraphics[width=0.6\linewidth, clip, trim=0cm 0cm 0cm 1cm]{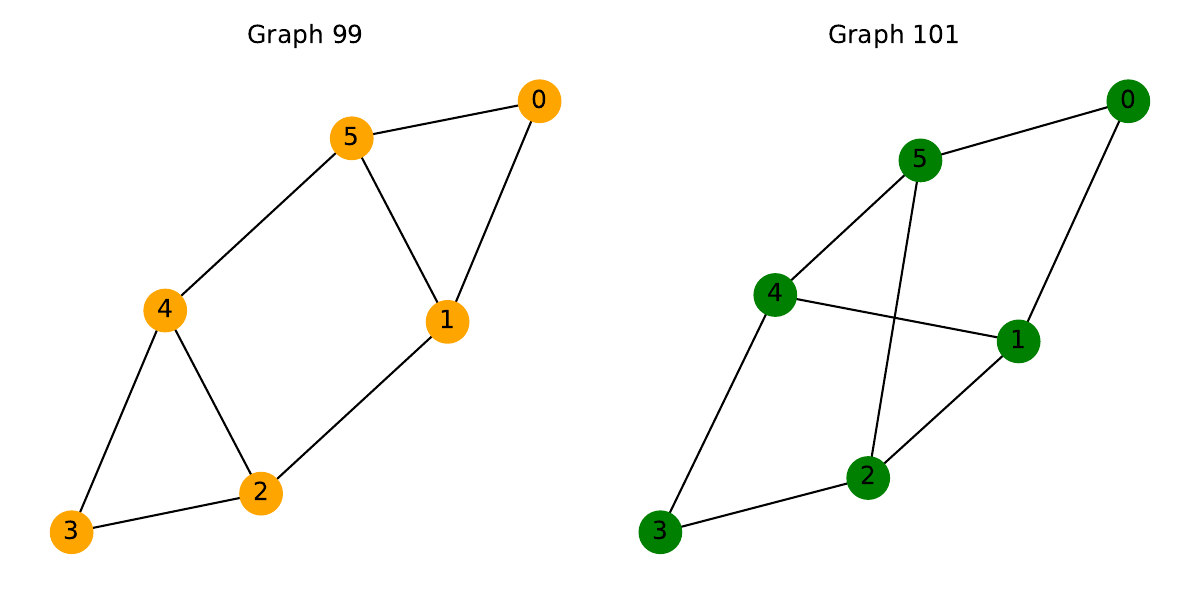}
    \includegraphics[width=0.35\linewidth]{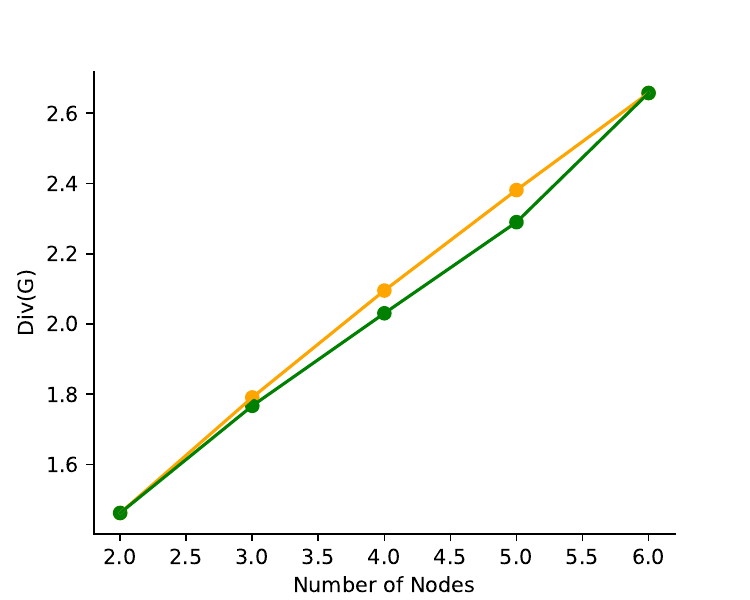}
    \caption{Graphs that are 1-WL indistinguishable and have identical diversity, $\Div(G)$ as computed via spread using shortest-path distances, that can be distinguished through edge pooling by their diversity curves, $\DivCurve(G)$. This example is used in \Cref{prop:wl_ind}.}
    \label{fig:graphs_examples_2}
\end{figure}

For our next statement, we first introduce another notion of comparing expressivity of two functions. 
\begin{definition}
     For two functions $\varphi$ and $\psi$ on the set of graphs, we say that $\varphi$ and $\psi$ are \emph{incomparable} if there exist graphs $G_A, G_B, G_C$, and $G_D$ such that $\varphi(G_A) \neq \varphi(G_B)$ and $\psi(G_A) = \psi(G_B)$, and $\varphi(G_C) = \varphi(G_D)$ and $\psi(G_C) \neq \psi(G_D)$. 
\end{definition}


\begin{prop}
    The spread and the Weisfeiler-Leman (1-WL) test are incomparable in terms of expressivity.
    \label{prop:wl_spread_incompatible}
\end{prop}
\begin{proof}
    We give concrete examples of two pairs of graphs. First, consider the disjoint union of two three-cycles $C_3 \sqcup C_3$ and the six-cycle $C_6$, which are known to be indistinguishable by 1-WL. The spread of the disjoint union of two three-cycles is $\Div(C_3 \sqcup C_3) = \frac{6}{1+2\exp(-1)} \approx 3.46$, whereas the spread of the $6$-cycle is $\Div(C_6) = \frac{6}{1+2\exp(-1)+2\exp(-2)+\exp(-3)} \approx 2.92$. \Cref{fig:graphs_examples_16} illustrates this example.\\
    Second, consider again the two graphs in \Cref{fig:graphs_examples}, which have the same spread as we have seen in the proof of \Cref{thm:coarsening_expressivity}. Those two graphs are distinguishable by 1-WL, since the degrees of the neighbourhoods of the vertices of degree $3$ differ in the two graphs. Hence, we show that the spread and 1-WL are incomparable in terms of expressivity. 
\end{proof}

\begin{prop}
    There exists a pair of graphs that is 1-WL-indistinguishable, indistinguishable by their spread alone, but distinguishable by their diversity curves over all possible edge contractions.
    \label{prop:wl_ind}
\end{prop}
\begin{proof}
    Consider the two graphs $G_C$ and $G_D$ in \Cref{fig:graphs_examples_2}, where $G_C$ is the $6$-cycle graph with the additional edges $(1,5)$ and $(2,4)$ on the left, and $G_D$ is the $6$-cycle graph with the additional edges $(1,4)$ and $(2,5)$ on the right. They are indistinguishable by the $1$-WL test as the colouring obtained from the degrees of each node is already a stable colouring. Both graphs have the same spread, concretely $\Div(G_C) = \Div(G_D) = \frac{2}{1+2\exp(-1)+2\exp(-2)+\exp(-3)} + \frac{4}{1+3\exp(-1)+2\exp(-2)} \approx 2.66$. As is shown in \Cref{fig:graphs_examples_2} on the right, the average spread over all possible edge contractions differs already after one edge contraction in each graph. 
\end{proof}
In total, these results show that if we have a pipeline that already distinguishes graphs on the level of 1-WL, then adding the diversity curves will increase the expressivity. Empirically, we demonstrate this in \Cref{app:expressiv_experiment} and \Cref{sec:expressivity}.

\section{Extended experiments}
\label{app:extended_experiments}
In the following section, we describe extended experimental details of all of our experiments, which describe how to reproduce our results. Further, we provide empirical expressivity results and ablation studies that support our main results.

\subsection{Graph perturbations}
\label{app:extended_experiments_gp}

\begin{figure}[b]
    \centering
    \includegraphics[width=\textwidth]{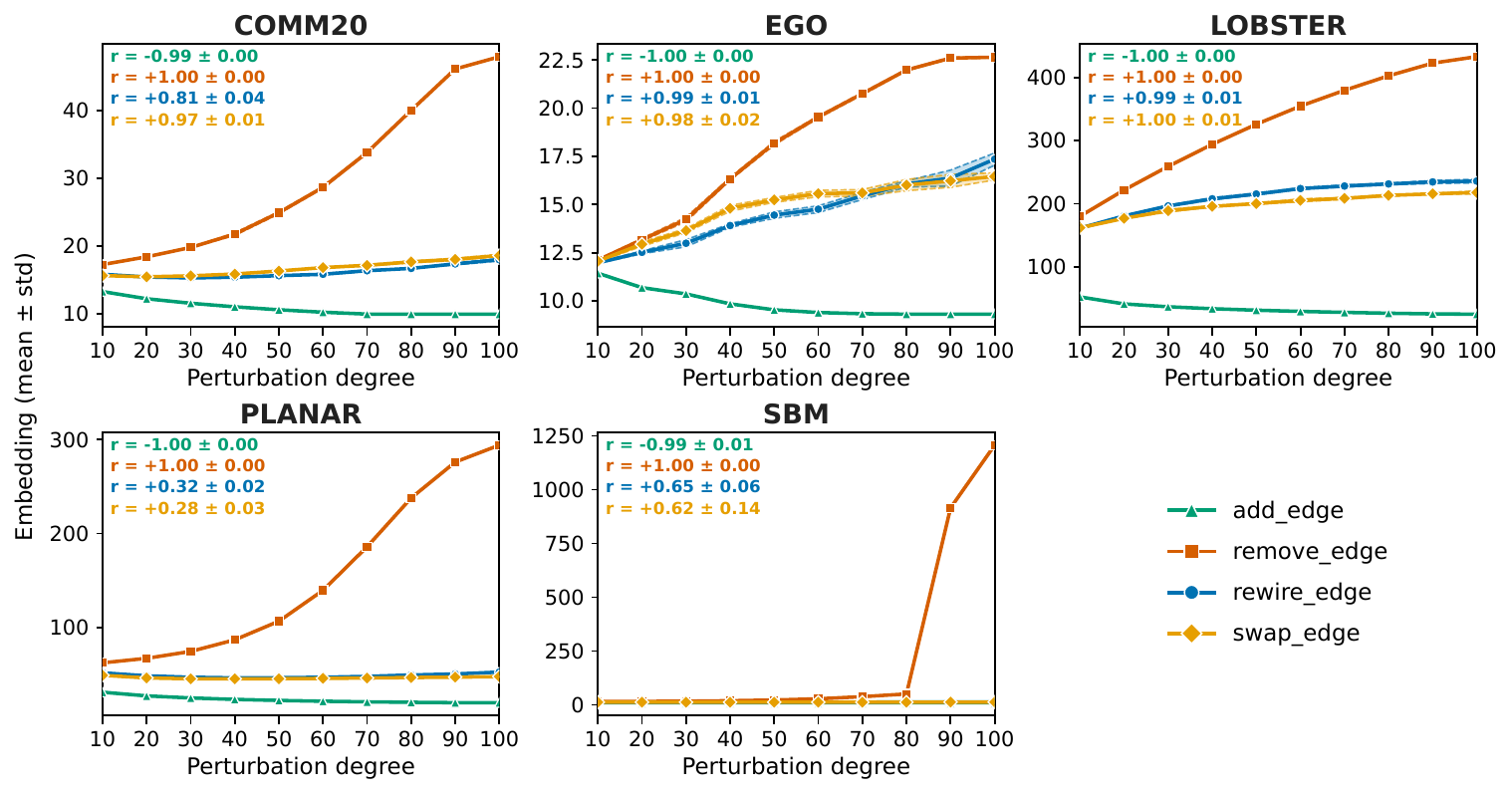}
    \caption{Diversity curves track the change in graph structure across edge perturbations. We report the norm of the mean diversity curve per dataset by perturbation degree.}
    \label{fig:perturbation_curves}
\end{figure}

The perturbation experiment in \Cref{sec:pert} applies one of the following perturbation scenarios:
 \begin{inparaenum}[(i.)]
    \item \textsc{add\_edge} (\textcolor{green}{$\blacktriangle$}) addition of an edge,
    \item \textsc{remove\_edge} (\textcolor{orange}{$\filledmedsquare$}) removal of a non-bridge edge
    \item \textsc{rewire\_edge} (\textcolor{blue}{$\bullet$}) connect one of the endpoints of an edge to another vertex and
    \item \textsc{swap\_edge} (\textcolor{yellow}{$\blacklozenge$}) given two edges $(v_0, v_1)$ and $(v_2, v_3)$ construct two new edges as $(v_0, v_2)$ and $(v_1, v_3)$. 
\end{inparaenum} The perturbation degree refers to the percentage of edges that were perturbed. In each case this represents a different quantity. In (i.) the degree $p$ is the percentage of edges remaining before the graph is fully-connected, i.e. $p=1$ means enough edges have been added so that the graph is fully-connected. In (ii.) $p$ is the percentage of edges that are not bridges, i.e. edges that do not result in the creation of more connected components when removed. In cases (iii.) and (iv.), $p$ is the percentage of edges in each graph that are altered. 
We uniformly sample edges from the pool of available ones.
Across each perturbation scenario and perturbation degree $p\in\{0.1, 0.2, \dots,1\}$, we apply the perturbation and track the norm of the average diversity curve across graphs, computed using shortest-path distances at node numbers between one and the maximum number of nodes in each dataset. 
Specifically, we consider the following datasets: \texttt{PLANAR} consists of 200 graphs with 64 nodes each, \texttt{LOBSTER} consists of 128 graphs with parameters $p_1 = p_2 =0.7$  with a minimum number of 10 nodes and a maximum number of 100 nodes, \texttt{SBM} consists of 200 graphs with 20-40 nodes per community with 2-5 communities subset to all graphs with less than 60 nodes, and \texttt{COMM20} consists of 200 graphs of a stochastic block model with 2 communities  \citep{martinkus2022spectre}\footnote{We take the code for generating the datasets from \url{https://github.com/KarolisMart/SPECTRE} available under an MIT licence.}. The additional dataset \texttt{EGO} consists of 757 graphs extracted from Citeseer \citep{jang2024a} \footnote{We take the code for generating these datasets from \url{https://github.com/yunhuijang/GEEL}}. We take a maximum of 80 graphs from each dataset, which is the test split except for the case of \texttt{SBM}, for which we consider 40 graphs with less than 60 nodes.
\Cref{fig:perturbation_curves} and \Cref{fig:corr}
show the results averaged over 5 random seeds along with the standard deviation.


\subsection{Simulated graph distributions}
\label{app:simulated_graphs}

Following our motivation of characterising graphs drawn from the same underlying distribution but across varying cardinalities, we elaborate on our results from \Cref{sec:sim}.

\subsubsection{Distinguishing graph distributions}
In \Cref{sec:sim}, we investigate how well diversity curves can distinguish between graphs generated from different random graph distributions. Specifically, we generate graphs with different sizes using the Python package NetworkX~\citep{hagberg2008networkx} released under a BSD-3-Clause Licence.
We choose the following four random graph distributions, for which we each generate $3$ graphs for every cardinality $n \in \{10, \dots ,29\}$: 
\begin{itemize}
    \item Erd\H{o}s-R\'{e}nyi (\texttt{ER}) graphs $G_{n,p}$ with edge probability parameter $p=0.75$.
    \item Random partition (\texttt{RP}) graphs on $n$ vertices, which have communities of fixed sizes and edge probabilities $p_{in}$ for nodes in the same community and $p_{out}$ for nodes in different communities. We fix the community sizes of the \texttt{RP} graphs such that we have $3$ communities, each of about $1/3$ the size of the graph and we set $p_{in} = 0.9$, and $p_{out} = 0.1$.
    \item Random geometric (\texttt{RG}) graphs, which are generated by choosing $n$ nodes uniformly at random in the unit square and adding an edge for two nodes if their Euclidean distance is less than or equal to a parameter $r$, which we set to $r=0.25$.
    \item Stochastic block model (\texttt{SBM}) of $n$ vertices, which partitions the nodes into communities and adds edges according to the community the vertices belong to. As in the \texttt{RP} model, we choose probabilities $p_{in} = 0.8$ and $p_{out} = 0.05$ which specify the edge probabilities between vertices in the same cluster and different clusters respectively. For the communities we partition the node set into $4$ communities each of about $1/4$ of the graph's size.
\end{itemize}

Example graphs generated by each distribution are shown in \Cref{fig:graphs_examples_sim}. Based on the graphs generated as detailed above, we aim to assess the ability of our method to distinguish between graph distributions. 
For computing the classification accuracies, we use a kNN-classifier with fixed $k=5$ and $10$-fold stratified grouped cross validation, where we group the graphs of the same cardinalities together. 
To assess the clustering qualities, we report the silhouette score, Calinski--Harabasz index, and Davies--Bouldin index with the true class labels, as well as the adjusted rand indices (ARI) based on $k$-means clustering and a hierarchical clustering approach. For the $k$-means clustering we specify the correct number of clusters and use `k-means++' for the initialisation method. For the hierarchical clustering we use agglomerative clustering where we specify the correct number of clusters and as a linkage criterion we use the average over all distances between the points in two clusters. All parameters which we do not specify here are used with their default setting as implemented in scikit-learn (version 1.5.2). To obtain uncertainties for the clustering metrics, we take the mean and standard deviation over $100$ random resamples using $50\%$ of the data points each.

\begin{figure}[b]
    \centering
    \includegraphics[width=0.24\linewidth]{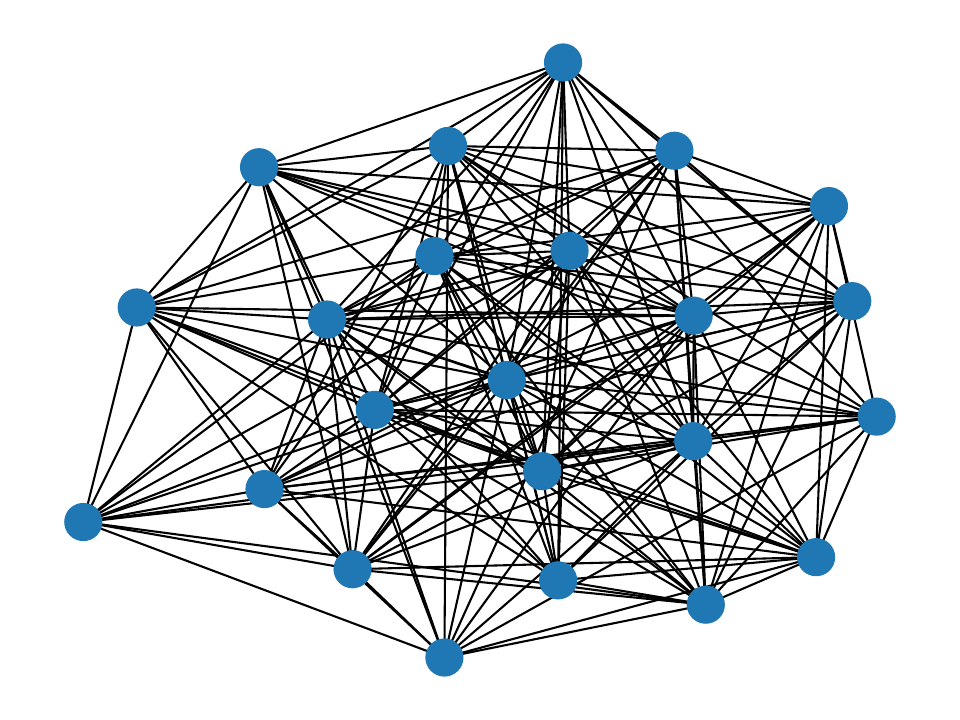}
    \includegraphics[width=0.24\linewidth]{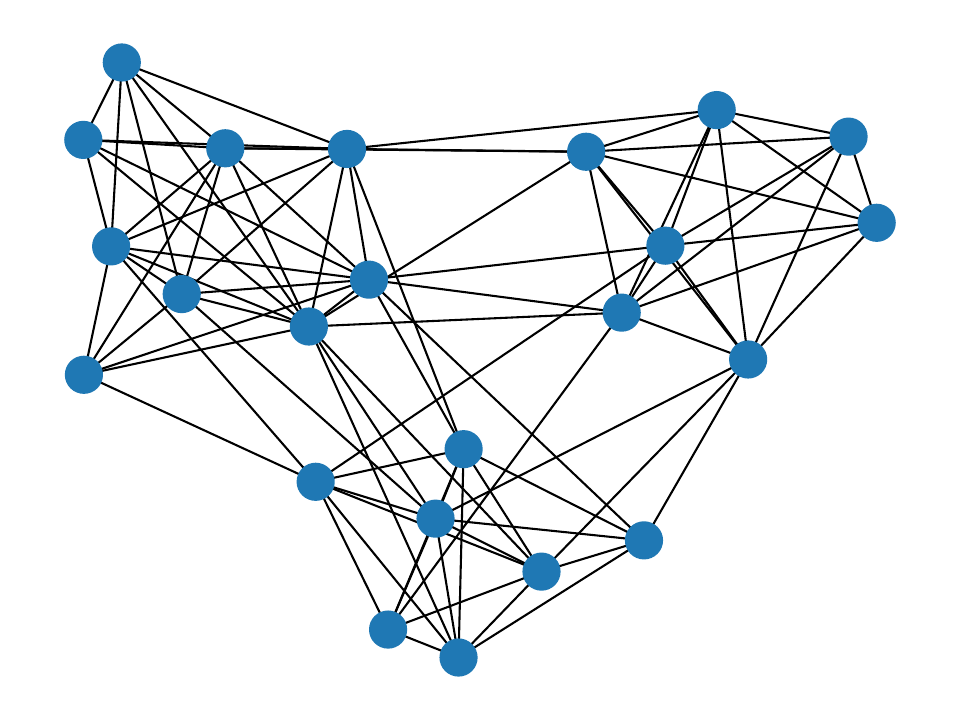}
    \includegraphics[width=0.24\linewidth]{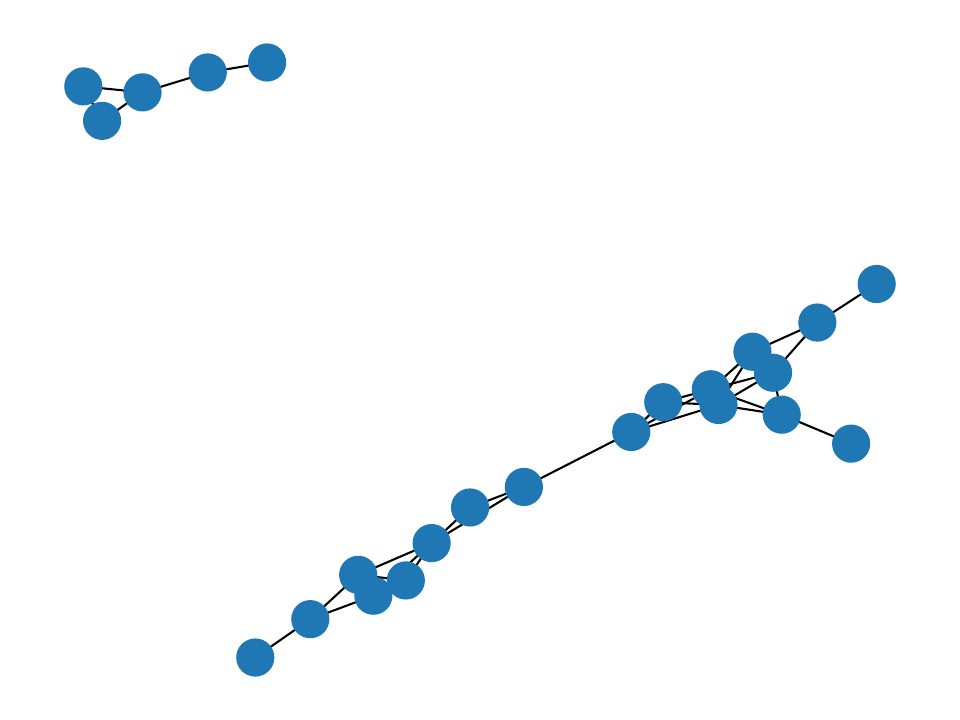}
    \includegraphics[width=0.24\linewidth]{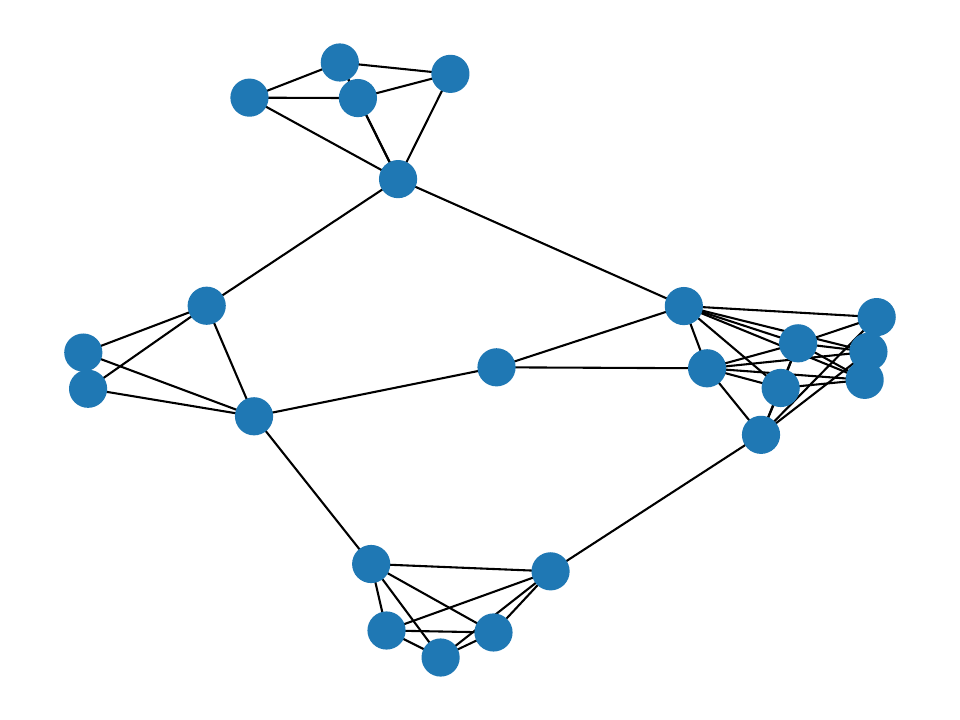}
    \caption{Graphs drawn from the distributions ER, \texttt{RP}, \texttt{RG}, \texttt{SBM} used in the experiment to distinguish different distributions from each other.}
    \label{fig:graphs_examples_sim}
\end{figure}

\paragraph{Diversity curve computations.} The kNN-classifier as well as the distances with which we compute clustering quality scores use the $L^2$-norm between the diversity curves viewed as vectorial embedding in $\mathbb{R}^{29}$ across evaluation scales from 1 to 29. Results are reported in \Cref{tab:sim_var_dist}. For visualisation as in \Cref{fig:sim_var_dist_MDS}, we perform PCA on the diversity curves using the $L^2$-norm.

\paragraph{Spread functions as baselines.} The spread function, as introduced by~\citet{willerton2015spread}, measures how the spread of a metric space changes with respect to a scale factor $t$ of the distance. Concretely, for a graph $G$ viewed as a metric space with shortest-path distance $d_G$ and a parameter $t > 0$, we define a metric space $tG$ with the same underlying points as $G$ and with the scaled metric $t\cdot d_G$. The spread function is the spread $\Div(tG)$ as a function of $t>0$. 
As an ablation, we compare our diversity curves to this approach of the diversity function. For each graph, we compute the spread function at the graph's original size for $1 \leq t \leq 5$ on $20$ evenly spaced evaluation scales and report the classification accuracies and clustering metrics using the $L^2$-norm between the spread functions.

\paragraph{Graph embedding baselines.} We compare our results to relevant groups of unsupervised graph representation methods, namely graph kernels, spectral graph representations given by NetLSD~\citep{tsitsulin2018netlsd} and SpectralZoo~\citep{jin2020spectral}, the topological methods TopER~\citep{tola2025toper} and a persistent homology based embedding \citep{ballester2025expressivity}, as well as the graph embeddings FEATHER~\citep{rozemberczki2020characteristic} and Graph2Vec~\citep{narayanan2017graph2vec}. 
For all the graph embeddings, NetLSD, FEATHER, and Graph2Vec, we use the vectorial representations with Euclidean distances for computing the clustering accuracies and scores. 

\paragraph{Kernel baselines.} For the graph kernels, we compute the Kernel matrix $K$ and use $-K + \max_{i,j} K_{i,j}$ as a pairwise distance to obtain clustering accuracies, silhouette scores, and the adjusted rand index with agglomerative clustering. We also use these pairwise distances for dimensionality reduction with MDS and compute all of our clustering metrics and accuracies also on the transformed data. 

\begin{table}[t]
    \centering
    \caption{Accuracies and clustering scores for distinguishing between the graph distributions \texttt{ER}, \texttt{RP}, \texttt{SBM}, and \texttt{RG}.
    }
    \resizebox{\textwidth}{!}{
\begin{tabular}{llllllll}
\toprule
 &  & Accuracy & Silhouette Score & Davies-Bouldin & Calinski-Harabasz & ARI k-means & ARI agglomerative \\
 & Method &  &  &  &  &  &  \\
\midrule
Diversity Curves & Shortest Path & \textbf{0.904 $\pm $ 0.059} & 0.423 $\pm $ 0.034 & 0.753 $\pm $ 0.060 & 191.237 $\pm $ 33.144 & 0.495 $\pm $ 0.071 & 0.313 $\pm $ 0.111 \\
 & Shortest Path PCA & 0.892 $\pm $ 0.065 & \textbf{0.432 $\pm $ 0.035} & \textbf{0.737 $\pm $ 0.060} & \textbf{194.064 $\pm $ 34.006} & \textbf{0.500 $\pm $ 0.075} & 0.323 $\pm $ 0.123 \\
 \midrule
Diversity Function & Spread Function Shortest Path & 0.829 $\pm $ 0.029 & 0.033 $\pm $ 0.019 & 3.949 $\pm $ 0.390 & 12.517 $\pm $ 2.420 & 0.097 $\pm $ 0.050 & 0.095 $\pm $ 0.049 \\
\midrule
Embeddings & NetLSD & 0.900 $\pm $ 0.033 & 0.055 $\pm $ 0.038 & 1.332 $\pm $ 0.112 & 38.799 $\pm $ 7.908 & 0.150 $\pm $ 0.054 & 0.092 $\pm $ 0.045 \\
 & FEATHER & 0.750 $\pm $ 0.091 & 0.108 $\pm $ 0.023 & 1.914 $\pm $ 0.287 & 74.719 $\pm $ 8.443 & 0.293 $\pm $ 0.033 & 0.277 $\pm $ 0.038 \\
 & SpectralZoo & 0.721 $\pm $ 0.088 & 0.050 $\pm $ 0.031 & 2.875 $\pm $ 2.025 & 33.794 $\pm $ 4.979 & 0.257 $\pm $ 0.048 & 0.140 $\pm $ 0.082 \\
 & Graph2Vec & 0.504 $\pm $ 0.094 & -0.005 $\pm $ 0.001 & 7.365 $\pm $ 0.291 & 1.405 $\pm $ 0.102 & 0.023 $\pm $ 0.024 & 0.002 $\pm $ 0.008 \\
 \midrule
Kernel & Kernel WL Optimal Assignment & 0.900 $\pm $ 0.046 & 0.019 $\pm $ 0.004 & - & - & - & 0.259 $\pm $ 0.046 \\
 & Kernel Core Framework & 0.842 $\pm $ 0.064 & 0.249 $\pm $ 0.026 & - & - & - & 0.446 $\pm $ 0.077 \\
 & Kernel Core Framework MDS & 0.821 $\pm $ 0.038 & 0.266 $\pm $ 0.026 & 1.414 $\pm $ 0.178 & 91.392 $\pm $ 11.984 & 0.459 $\pm $ 0.049 & \textbf{0.472 $\pm $ 0.060} \\
 & Kernel WL Optimal Assignment MDS & 0.796 $\pm $ 0.101 & 0.045 $\pm $ 0.020 & 3.491 $\pm $ 0.551 & 13.624 $\pm $ 2.199 & 0.185 $\pm $ 0.045 & 0.202 $\pm $ 0.057 \\
 & Kernel WL MDS & 0.779 $\pm $ 0.072 & 0.239 $\pm $ 0.034 & 2.531 $\pm $ 0.443 & 21.090 $\pm $ 4.602 & 0.429 $\pm $ 0.061 & 0.273 $\pm $ 0.094 \\
 & Kernel Shortest Path & 0.771 $\pm $ 0.073 & 0.284 $\pm $ 0.044 & - & - & - & 0.426 $\pm $ 0.057 \\
 & Kernel Shortest Path MDS & 0.767 $\pm $ 0.046 & 0.274 $\pm $ 0.038 & 1.597 $\pm $ 0.169 & 57.105 $\pm $ 11.254 & 0.444 $\pm $ 0.065 & 0.423 $\pm $ 0.064 \\
 & Kernel WL & 0.746 $\pm $ 0.094 & 0.213 $\pm $ 0.029 & - & - & - & 0.031 $\pm $ 0.090 \\
 & Kernel Pyramid Match & 0.671 $\pm $ 0.039 & 0.002 $\pm $ 0.009 & - & - & - & 0.017 $\pm $ 0.016 \\
 & Kernel Pyramid Match MDS & 0.588 $\pm $ 0.051 & -0.009 $\pm $ 0.012 & 7.657 $\pm $ 3.929 & 7.242 $\pm $ 1.145 & 0.065 $\pm $ 0.038 & 0.045 $\pm $ 0.032 \\
 & Kernel Neighbourhood Hash & 0.500 $\pm $ 0.062 & -0.032 $\pm $ 0.006 & - & - & - & 0.002 $\pm $ 0.008 \\
 & Kernel Neighbourhood Hash MDS & 0.421 $\pm $ 0.066 & -0.066 $\pm $ 0.008 & 20.290 $\pm $ 10.776 & 0.683 $\pm $ 0.411 & -0.007 $\pm $ 0.008 & -0.003 $\pm $ 0.007 \\
 \midrule
TopER & TopER Degree & 0.758 $\pm $ 0.091 & -0.064 $\pm $ 0.021 & 5.288 $\pm $ 1.902 & 31.191 $\pm $ 6.050 & 0.089 $\pm $ 0.019 & 0.089 $\pm $ 0.024 \\
 & TopER Closeness & 0.746 $\pm $ 0.066 & -0.093 $\pm $ 0.019 & 8.445 $\pm $ 7.621 & 25.482 $\pm $ 5.115 & 0.084 $\pm $ 0.020 & 0.076 $\pm $ 0.020 \\
 & TopER All Filtrations & 0.700 $\pm $ 0.096 & -0.055 $\pm $ 0.020 & 3.255 $\pm $ 0.499 & 27.534 $\pm $ 5.504 & 0.101 $\pm $ 0.025 & 0.073 $\pm $ 0.028 \\
 & TopER Popularity & 0.696 $\pm $ 0.093 & -0.022 $\pm $ 0.037 & 1.344 $\pm $ 0.129 & 29.667 $\pm $ 7.188 & 0.113 $\pm $ 0.076 & 0.062 $\pm $ 0.019 \\
 & TopER Forricci & 0.646 $\pm $ 0.088 & -0.067 $\pm $ 0.022 & 2.950 $\pm $ 1.701 & 31.450 $\pm $ 5.224 & 0.112 $\pm $ 0.024 & 0.112 $\pm $ 0.027 \\
 & TopER Olricci & 0.592 $\pm $ 0.069 & -0.099 $\pm $ 0.015 & 10.200 $\pm $ 9.354 & 13.468 $\pm $ 3.562 & 0.059 $\pm $ 0.020 & 0.043 $\pm $ 0.017 \\
 \midrule
Persistent Homology & PH Vietoris Rips & 0.783 $\pm $ 0.067 & -0.010 $\pm $ 0.031 & 2.227 $\pm $ 0.175 & 16.156 $\pm $ 2.031 & 0.085 $\pm $ 0.028 & 0.053 $\pm $ 0.029 \\
 & PH Laplacian & 0.783 $\pm $ 0.074 & -0.132 $\pm $ 0.018 & 2.668 $\pm $ 0.247 & 11.701 $\pm $ 2.184 & 0.055 $\pm $ 0.021 & 0.034 $\pm $ 0.019 \\
 & PH Alpha Complex & 0.742 $\pm $ 0.103 & 0.014 $\pm $ 0.021 & 1.992 $\pm $ 0.182 & 43.715 $\pm $ 6.740 & 0.202 $\pm $ 0.034 & 0.160 $\pm $ 0.054 \\
 & PH Fricci & 0.742 $\pm $ 0.067 & -0.099 $\pm $ 0.020 & 2.498 $\pm $ 0.220 & 30.874 $\pm $ 4.917 & 0.090 $\pm $ 0.031 & 0.043 $\pm $ 0.029 \\
 & PH Degree & 0.725 $\pm $ 0.062 & -0.160 $\pm $ 0.019 & 3.016 $\pm $ 0.353 & 8.580 $\pm $ 1.651 & 0.040 $\pm $ 0.020 & 0.018 $\pm $ 0.014 \\
 & PH Olricci & 0.696 $\pm $ 0.085 & -0.125 $\pm $ 0.023 & 2.945 $\pm $ 0.473 & 34.161 $\pm $ 5.567 & 0.092 $\pm $ 0.031 & 0.060 $\pm $ 0.023 \\
\bottomrule
\end{tabular}
    }
    \label{tab:simulations results ER RPG SBM RG}
\end{table}

\paragraph{Topological baselines.} For TopER we report for each of the filtrations described by \citet{tola2025toper} as sublevel filtrations individually and also combine all of them. We use the vectorial representations with Euclidean distances to compute clustering accuracies and scores. 
We use the same approach for the persistent homology based graph embeddings as described in the classification experiment by \citet{ballester2025expressivity}. Namely, we use different filtrations to compute persistence homology up to dimension two on the clique complexed obtained from the graphs by filling in every $3$-clique, so we obtain a simplicial $2$-complex. We transform the resulting persistence diagrams into a vectorial representation using their persistent images.  
We again use the Euclidean distances of the embeddings to compute clustering accuracies and scores.

The complete results can be found in \Cref{tab:simulations results ER RPG SBM RG}, which is an extended version of \Cref{tab:sim_var_dist}. Diversity curves perform similar in accuracy as the best performing baselines, and outperform the baselines in terms of almost all reported clustering metrics.

\paragraph{Comparison to simple structural baselines}
As a sanity check, we compare the performance of diversity curves in the simulation experiment from \Cref{sec:sim} as reported in \Cref{tab:ablation_dist_structturebaseline} against simple structural baselines. Concretely, we consider the number of nodes, number of edges, number of connected components, average clustering coefficient, the maximum diameter over all connected components, the average shortest path distance (averaged over  all connected components), the degree statistic (degree sequence padded with zeros to obtain the same length for all graphs), and a vectorisation of the histogram of the normalised Laplacian spectrum, and use each of these inputs to perform the kNN-classification and compute the clustering scores. Additionally, we also report the accuracy and clustering performance when collecting all the structural invariants in a summary statistic. 
As expected, we see that the classification based on the number of nodes is based on chance, confirming that by construction the number of nodes is uninformative for the tasks at hand. 
\begin{table}[t]
    \centering
    
    \caption{Accuracies and clustering scores of baseline structural summaries for distinguishing between the graph distributions \texttt{ER}, \texttt{RP}, \texttt{SBM}, and \texttt{RG}.
    }
    \resizebox{\textwidth}{!}{

\begin{tabular}{llllllll}
\toprule
         &                 &                     Accuracy &             Silhouette Score &               Davies-Bouldin &               Calinski-Harabasz &                  ARI k-means &            ARI agglomerative \\
{} & Method &                              &                              &                              &                                 &                              &                              \\
\midrule
Diversity Curves & Shortest Path &  \textbf{0.904 $\pm $ 0.059} &           0.423 $\pm $ 0.034 &           0.753 $\pm $ 0.060 &           191.237 $\pm $ 33.144 &           0.495 $\pm $ 0.071 &           0.313 $\pm $ 0.111 \\
         & Shortest Path PCA &           0.892 $\pm $ 0.065 &  \textbf{0.432 $\pm $ 0.035} &  \textbf{0.737 $\pm $ 0.060} &  \textbf{194.064 $\pm $ 34.006} &  \textbf{0.500 $\pm $ 0.075} &  \textbf{0.323 $\pm $ 0.123} \\
         \midrule
Baseline & Summary Statistic &           0.838 $\pm $ 0.029 &          -0.053 $\pm $ 0.017 &           3.196 $\pm $ 0.640 &             31.287 $\pm $ 5.296 &           0.111 $\pm $ 0.024 &           0.090 $\pm $ 0.033 \\
         & Laplacian Spectrum (binned) &           0.838 $\pm $ 0.060 &           0.058 $\pm $ 0.011 &           2.785 $\pm $ 0.237 &             20.417 $\pm $ 2.458 &           0.215 $\pm $ 0.051 &           0.150 $\pm $ 0.066 \\
         & Degree Sequence &           0.833 $\pm $ 0.046 &          -0.028 $\pm $ 0.012 &           3.459 $\pm $ 0.308 &             23.817 $\pm $ 3.310 &           0.135 $\pm $ 0.044 &           0.091 $\pm $ 0.034 \\
         & Average Shortest Path &           0.762 $\pm $ 0.059 &           0.187 $\pm $ 0.040 &         23.435 $\pm $ 90.534 &             30.409 $\pm $ 7.871 &           0.378 $\pm $ 0.044 &           0.305 $\pm $ 0.050 \\
         & Clustering Coefficient &           0.617 $\pm $ 0.095 &           0.041 $\pm $ 0.031 &           4.252 $\pm $ 4.872 &             34.702 $\pm $ 5.473 &           0.235 $\pm $ 0.055 &           0.129 $\pm $ 0.091 \\
         & Max Diameter &           0.579 $\pm $ 0.101 &           0.243 $\pm $ 0.043 &         15.235 $\pm $ 31.362 &             33.867 $\pm $ 6.923 &           0.378 $\pm $ 0.091 &           0.169 $\pm $ 0.072 \\
         & Number of Edges &           0.483 $\pm $ 0.121 &          -0.094 $\pm $ 0.021 &           3.666 $\pm $ 2.338 &             33.508 $\pm $ 5.866 &           0.113 $\pm $ 0.024 &           0.089 $\pm $ 0.029 \\
         & Number of Connected Components &           0.454 $\pm $ 0.060 &          -0.134 $\pm $ 0.055 &           1.713 $\pm $ 0.409 &            68.197 $\pm $ 14.380 &           0.168 $\pm $ 0.032 &           0.186 $\pm $ 0.032 \\
         & Number of Nodes &           0.250 $\pm $ 0.000 &          -0.083 $\pm $ 0.013 &        99.298 $\pm $ 196.613 &              0.532 $\pm $ 0.474 &          -0.012 $\pm $ 0.006 &          -0.012 $\pm $ 0.007 \\
\bottomrule
\end{tabular}

    } 
    \label{tab:ablation_dist_structturebaseline}
\end{table}

We further observe that all tested structural baselines perform worse in terms of the clustering metrics compared to diversity curves. Also in terms of accuracies, diversity curves outperform all structural baselines, with the normalised Laplacian spectrum achieving the highest accuracy amongst the structural baselines after the summary statistic. Hence, we verify that diversity curves provide information beyond standard graph summary statistics for these tasks. 


\subsubsection{Distinguishing parameter choices}
Next, we evaluate how well our methods can distinguish between different parameters used to generate graphs. We look at one random graph model at a time and vary the parameter choices within the model. We again generate $3$ graphs per class for every node size $n \in \{10, \dots ,29\}$ to understand how well the diversity curves can distinguish between these more subtle changes in the distribution while being robust to variations of the sizes of the graphs. We use the same graph models as in \Cref{app:simulated_graphs} and vary the parameters as follows:
\begin{itemize}
    \item Erd\H{o}s-R\'{e}nyi (\texttt{ER}) $G_{n,p}$ where we vary the edge probability $p \in \{0.1, 0.5, 0.9\}$;
    \item Random partition (\texttt{RP}) graphs where we vary the edge probabilities $(p_{in}, p_{out}) \in \{(0.9, 0.01), (0.9, 0.1), (0.9, 0.5)\}$ and fix the number of communities to be $3$;
    \item Random geometric (\texttt{RG}) graphs where we vary the radius $r \in \{0.25, 0.5, 0.75\}$ determining which pairs of nodes are connected by an edge;
    \item Stochastic block model (\texttt{SBM}), where we vary the number of communities between $2$, $3$, and $4$ roughly equally sized communities. The edge probabilities $p_{in} = 0.6$ and $p_{out} = 0.01$ are fixed.
\end{itemize}

For each random graph model, we compute the classification accuracies on the test set and clustering scores as described in \Cref{app:simulated_graphs} for the same set of baseline methods. The full results can be found in Tables \ref{tab:simulation results ER}, \ref{tab:simulations results RPG}, \ref{tab:simulation results RG}, and \ref{tab:simulation results SBM}. 
Furthermore, we show the PCA plot of the diversity curves for each distribution in Figure \ref{fig:sim_PCA_var_param}, which shows that diversity curves are able to cluster the graphs belonging to the same parameter setting clearly. 

\begin{figure}[tbh]
    \centering
    \includegraphics[width=0.32\linewidth]{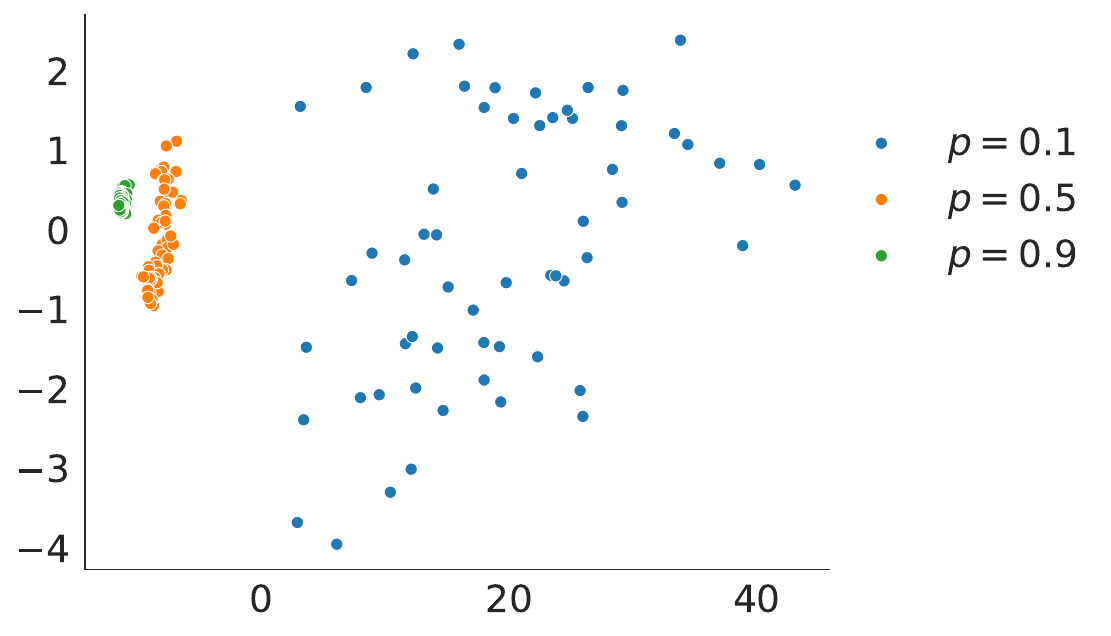}
    \includegraphics[width=0.35\linewidth]{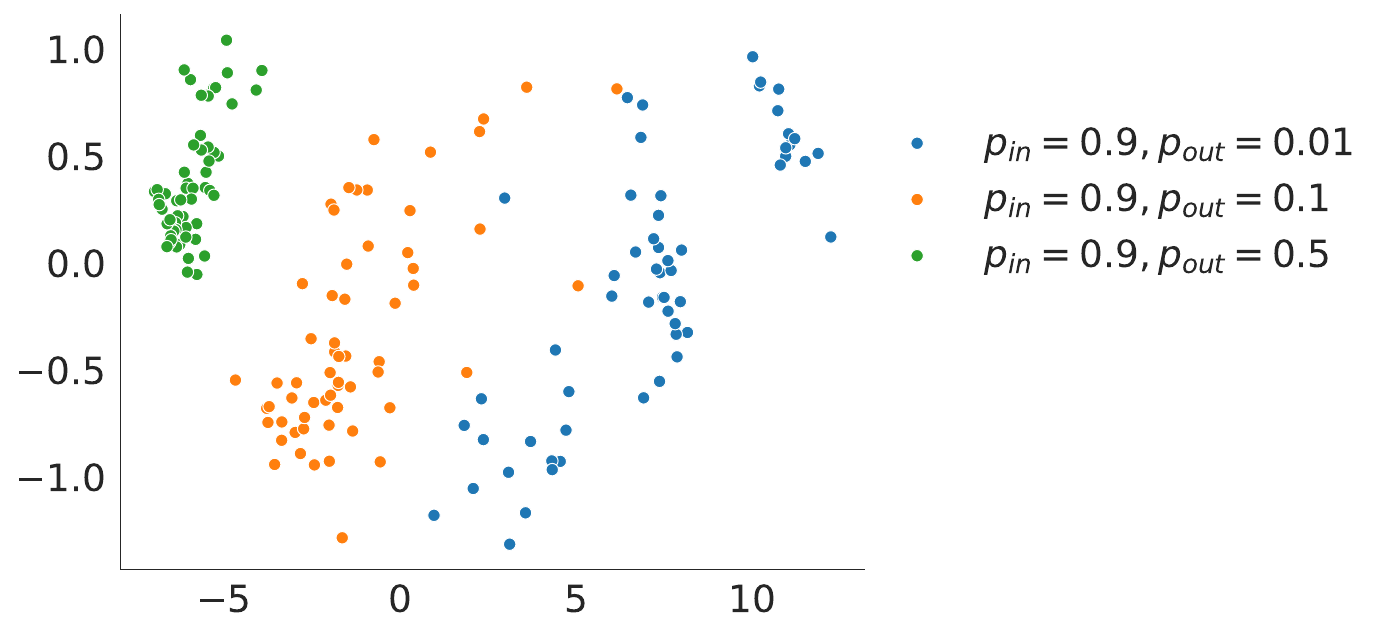}
    \includegraphics[width=0.32\linewidth]{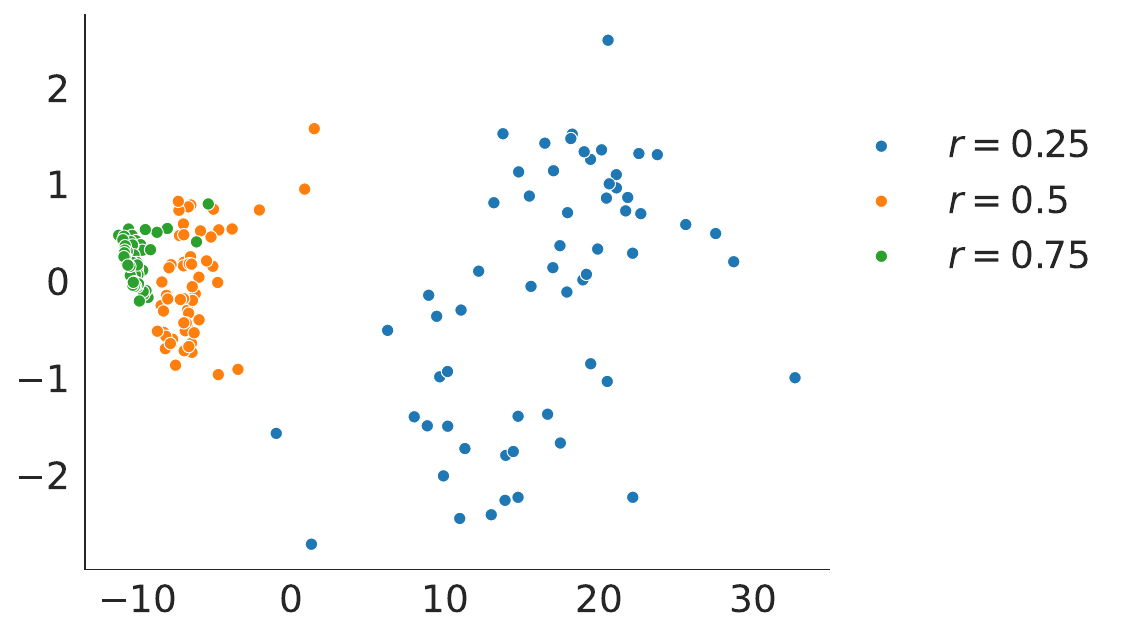}
    \includegraphics[width=0.35\linewidth]{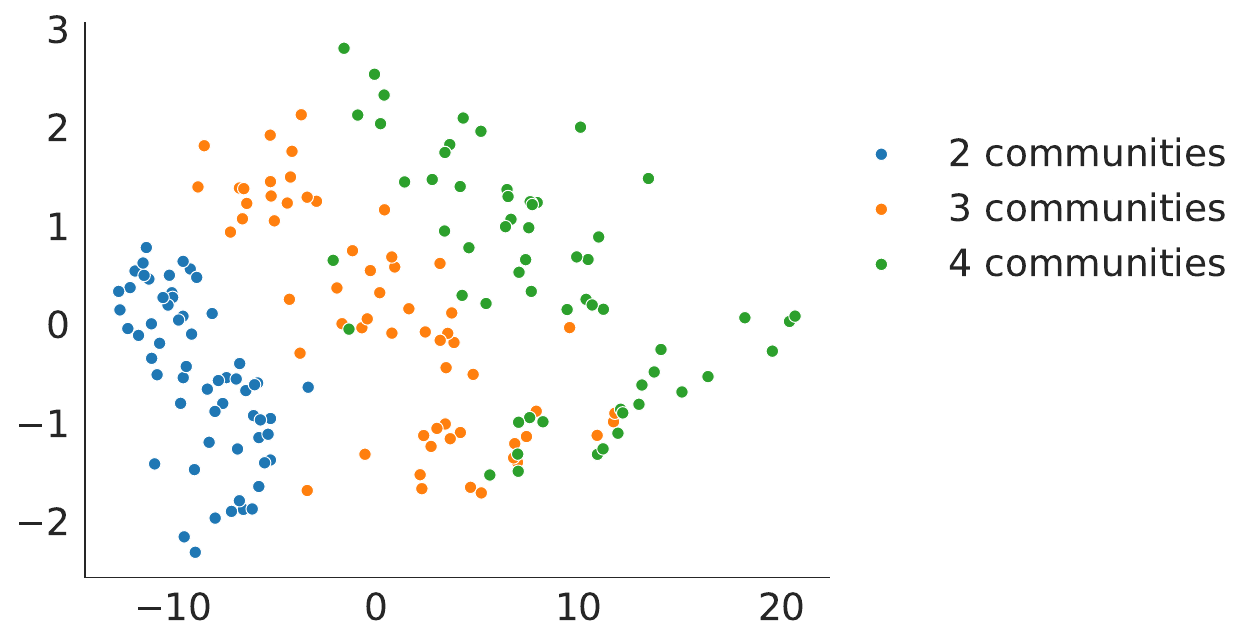}
    \caption{ PCA plots of the diversity curves for \texttt{ER}, \texttt{RP}, \texttt{RG}, and \texttt{SBM} graphs.}
    \label{fig:sim_PCA_var_param}
\end{figure}

\begin{table}[tbh]
    \centering
    
    \caption{Accuracies and clustering scores for distinguishing parameter choices in  \texttt{ER} graphs
    }
    \label{tab:simulation results ER}
    \resizebox{\textwidth}{!}{
\begin{tabular}{llllllll}
\toprule
 &  & Accuracy & Silhouette Score & Davies-Bouldin & Calinski-Harabasz & ARI agglomerative & ARI k-means \\
 & Method &  &  &  &  &  &  \\
\midrule
Diversity Curves & Shortest Path & \textbf{1.000 $\pm $ 0.000} & 0.679 $\pm $ 0.021 & 0.346 $\pm $ 0.018 & 272.892 $\pm $ 54.881 & 0.486 $\pm $ 0.053 & 0.444 $\pm $ 0.056 \\
 & Shortest Path PCA & \textbf{1.000 $\pm $ 0.000} & 0.693 $\pm $ 0.021 & \textbf{0.328 $\pm $ 0.019} & 277.053 $\pm $ 56.295 & 0.485 $\pm $ 0.052 & 0.445 $\pm $ 0.056 \\
 \midrule
Diversity Function & Spread Function Shortest Path & 0.956 $\pm $ 0.069 & 0.133 $\pm $ 0.019 & 2.209 $\pm $ 0.327 & 29.456 $\pm $ 5.114 & 0.198 $\pm $ 0.097 & 0.189 $\pm $ 0.067 \\
\midrule
Embeddings & FEATHER & \textbf{1.000 $\pm $ 0.000} & 0.574 $\pm $ 0.024 & 0.698 $\pm $ 0.053 & 186.237 $\pm $ 23.639 & 0.811 $\pm $ 0.134 & 0.800 $\pm $ 0.074 \\
 & SpectralZoo & \textbf{1.000 $\pm $ 0.000} & 0.335 $\pm $ 0.027 & 1.039 $\pm $ 0.107 & 161.364 $\pm $ 26.514 & 0.453 $\pm $ 0.115 & 0.439 $\pm $ 0.023 \\
 & NetLSD & 0.972 $\pm $ 0.051 & 0.096 $\pm $ 0.027 & 2.044 $\pm $ 0.196 & 45.743 $\pm $ 8.983 & 0.076 $\pm $ 0.062 & 0.194 $\pm $ 0.062 \\
 & Graph2Vec & 0.589 $\pm $ 0.083 & -0.001 $\pm $ 0.002 & 7.240 $\pm $ 0.335 & 1.328 $\pm $ 0.107 & 0.001 $\pm $ 0.005 & 0.011 $\pm $ 0.024 \\
 \midrule
Kernel & Kernel Shortest Path & \textbf{1.000 $\pm $ 0.000} & \textbf{0.793 $\pm $ 0.035} & - & - & 0.857 $\pm $ 0.066 & - \\
 & Kernel Core Framework & \textbf{1.000 $\pm $ 0.000} & 0.644 $\pm $ 0.019 & - & - & 0.778 $\pm $ 0.235 & - \\
 & Kernel Shortest Path MDS & 0.994 $\pm $ 0.017 & 0.696 $\pm $ 0.034 & 0.482 $\pm $ 0.072 & 180.703 $\pm $ 47.049 & 0.758 $\pm $ 0.114 & 0.781 $\pm $ 0.064 \\
 & Kernel Core Framework MDS & 0.994 $\pm $ 0.017 & 0.670 $\pm $ 0.018 & 0.445 $\pm $ 0.026 & \textbf{356.006 $\pm $ 44.846} & \textbf{0.940 $\pm $ 0.093} & \textbf{0.942 $\pm $ 0.100} \\
 & Kernel WL MDS & 0.994 $\pm $ 0.017 & 0.557 $\pm $ 0.032 & 0.678 $\pm $ 0.098 & 97.450 $\pm $ 22.580 & 0.477 $\pm $ 0.044 & 0.765 $\pm $ 0.144 \\
 & Kernel WL & 0.989 $\pm $ 0.022 & 0.513 $\pm $ 0.030 & - & - & 0.459 $\pm $ 0.060 & - \\
 & Kernel WL Optimal Assignment & 0.978 $\pm $ 0.037 & 0.091 $\pm $ 0.009 & - & - & 0.586 $\pm $ 0.080 & - \\
 & Kernel WL Optimal Assignment MDS & 0.917 $\pm $ 0.067 & 0.241 $\pm $ 0.030 & 1.973 $\pm $ 0.410 & 29.540 $\pm $ 4.918 & 0.478 $\pm $ 0.122 & 0.414 $\pm $ 0.136 \\
 & Kernel Pyramid Match & 0.811 $\pm $ 0.132 & 0.035 $\pm $ 0.011 & - & - & 0.017 $\pm $ 0.028 & - \\
 & Kernel Pyramid Match MDS & 0.756 $\pm $ 0.106 & 0.056 $\pm $ 0.015 & 11.737 $\pm $ 9.093 & 6.665 $\pm $ 1.152 & 0.050 $\pm $ 0.031 & 0.036 $\pm $ 0.045 \\
 & Kernel Neighbourhood Hash & 0.611 $\pm $ 0.122 & -0.007 $\pm $ 0.007 & - & - & 0.006 $\pm $ 0.014 & - \\
 & Kernel Neighbourhood Hash MDS & 0.556 $\pm $ 0.082 & -0.032 $\pm $ 0.010 & 14.111 $\pm $ 14.751 & 2.214 $\pm $ 1.022 & 0.012 $\pm $ 0.020 & 0.019 $\pm $ 0.025 \\
 \midrule
TopER & TopER Forricci & 0.906 $\pm $ 0.066 & 0.263 $\pm $ 0.044 & 1.196 $\pm $ 0.185 & 56.032 $\pm $ 10.377 & 0.180 $\pm $ 0.097 & 0.292 $\pm $ 0.077 \\
 & TopER Closeness & 0.894 $\pm $ 0.058 & 0.057 $\pm $ 0.037 & 2.578 $\pm $ 0.876 & 20.609 $\pm $ 4.163 & 0.101 $\pm $ 0.053 & 0.133 $\pm $ 0.047 \\
 & TopER All Filtrations & 0.894 $\pm $ 0.084 & 0.139 $\pm $ 0.037 & 2.062 $\pm $ 0.474 & 30.453 $\pm $ 5.249 & 0.145 $\pm $ 0.065 & 0.213 $\pm $ 0.067 \\
 & TopER Degree & 0.883 $\pm $ 0.063 & 0.061 $\pm $ 0.036 & 2.729 $\pm $ 0.930 & 20.703 $\pm $ 4.083 & 0.110 $\pm $ 0.059 & 0.143 $\pm $ 0.044 \\
 & TopER Olricci & 0.883 $\pm $ 0.072 & 0.162 $\pm $ 0.043 & 1.850 $\pm $ 0.481 & 34.409 $\pm $ 6.646 & 0.160 $\pm $ 0.064 & 0.202 $\pm $ 0.055 \\
 & TopER Popularity & 0.856 $\pm $ 0.094 & 0.178 $\pm $ 0.042 & 2.006 $\pm $ 0.608 & 35.827 $\pm $ 5.863 & 0.211 $\pm $ 0.105 & 0.259 $\pm $ 0.070 \\
 \midrule
Persistent Homology & PH Fricci & \textbf{1.000 $\pm $ 0.000} & 0.128 $\pm $ 0.045 & 1.227 $\pm $ 0.090 & 18.835 $\pm $ 3.434 & 0.057 $\pm $ 0.029 & 0.105 $\pm $ 0.063 \\
 & PH Vietoris Rips & \textbf{1.000 $\pm $ 0.000} & 0.213 $\pm $ 0.030 & 1.264 $\pm $ 0.075 & 29.377 $\pm $ 3.326 & 0.090 $\pm $ 0.061 & 0.139 $\pm $ 0.048 \\
 & PH Olricci & 0.989 $\pm $ 0.022 & 0.176 $\pm $ 0.049 & 0.919 $\pm $ 0.047 & 40.144 $\pm $ 6.448 & 0.081 $\pm $ 0.048 & 0.159 $\pm $ 0.046 \\
 & PH Laplacian & 0.900 $\pm $ 0.074 & 0.034 $\pm $ 0.045 & 2.152 $\pm $ 0.176 & 5.495 $\pm $ 1.112 & 0.021 $\pm $ 0.016 & 0.032 $\pm $ 0.020 \\
 & PH Degree & 0.883 $\pm $ 0.091 & 0.029 $\pm $ 0.047 & 2.237 $\pm $ 0.127 & 5.371 $\pm $ 0.788 & 0.011 $\pm $ 0.013 & 0.037 $\pm $ 0.028 \\
 & PH Alpha Complex & 0.767 $\pm $ 0.108 & 0.067 $\pm $ 0.029 & 2.016 $\pm $ 0.321 & 28.349 $\pm $ 6.228 & 0.124 $\pm $ 0.078 & 0.175 $\pm $ 0.048 \\
\bottomrule
\end{tabular}
    }
\end{table}
\begin{table}[tbh]
    \centering
    \caption{Accuracies and clustering scores for distinguishing parameter choices in \texttt{RP} graphs.
    }
    \resizebox{\textwidth}{!}{
\begin{tabular}{llllllll}
\toprule
 &  & Accuracy & Silhouette Score & Davies-Bouldin & Calinski-Harabasz & ARI agglomerative & ARI k-means \\
 & Method &  &  &  &  &  &  \\
\midrule
Diversity Curves & Shortest Path & \textbf{0.978 $\pm $ 0.037} & 0.576 $\pm $ 0.032 & 0.501 $\pm $ 0.038 & 285.485 $\pm $ 47.777 & 0.448 $\pm $ 0.062 & \textbf{0.674 $\pm $ 0.139} \\
 & Shortest Path PCA & \textbf{0.978 $\pm $ 0.037} & \textbf{0.593 $\pm $ 0.033} & \textbf{0.474 $\pm $ 0.041} & \textbf{295.427 $\pm $ 50.529} & 0.458 $\pm $ 0.088 & 0.673 $\pm $ 0.141 \\
 \midrule
Diversity Function & Spread Function Shortest Path & 0.928 $\pm $ 0.070 & 0.035 $\pm $ 0.017 & 4.557 $\pm $ 0.679 & 6.023 $\pm $ 1.646 & 0.030 $\pm $ 0.032 & 0.021 $\pm $ 0.025 \\
\midrule
Embeddings & NetLSD & 0.972 $\pm $ 0.037 & 0.145 $\pm $ 0.025 & 1.385 $\pm $ 0.088 & 32.110 $\pm $ 3.923 & 0.036 $\pm $ 0.056 & 0.382 $\pm $ 0.097 \\
 & SpectralZoo & 0.906 $\pm $ 0.066 & 0.117 $\pm $ 0.036 & 1.678 $\pm $ 0.327 & 31.343 $\pm $ 6.921 & 0.065 $\pm $ 0.064 & 0.229 $\pm $ 0.056 \\
 & FEATHER & 0.861 $\pm $ 0.083 & 0.191 $\pm $ 0.029 & 1.502 $\pm $ 0.134 & 41.949 $\pm $ 7.816 & 0.175 $\pm $ 0.073 & 0.212 $\pm $ 0.051 \\
 & Graph2Vec & 0.622 $\pm $ 0.060 & 0.003 $\pm $ 0.002 & 6.308 $\pm $ 0.298 & 2.655 $\pm $ 0.291 & 0.024 $\pm $ 0.044 & 0.176 $\pm $ 0.067 \\
 \midrule
Kernel & Kernel WL MDS & 0.972 $\pm $ 0.037 & 0.340 $\pm $ 0.036 & 1.317 $\pm $ 0.149 & 22.752 $\pm $ 4.198 & 0.109 $\pm $ 0.119 & 0.349 $\pm $ 0.128 \\
 & Kernel Shortest Path & 0.967 $\pm $ 0.037 & 0.527 $\pm $ 0.042 & - & - & \textbf{0.510 $\pm $ 0.088} & - \\
 & Kernel Shortest Path MDS & 0.967 $\pm $ 0.027 & 0.434 $\pm $ 0.036 & 0.990 $\pm $ 0.075 & 40.063 $\pm $ 5.645 & 0.133 $\pm $ 0.106 & 0.469 $\pm $ 0.118 \\
 & Kernel Core Framework & 0.961 $\pm $ 0.050 & 0.345 $\pm $ 0.027 & - & - & 0.244 $\pm $ 0.104 & - \\
 & Kernel Core Framework MDS & 0.950 $\pm $ 0.072 & 0.268 $\pm $ 0.035 & 1.484 $\pm $ 0.233 & 36.632 $\pm $ 6.913 & 0.289 $\pm $ 0.129 & 0.424 $\pm $ 0.077 \\
 & Kernel WL & 0.939 $\pm $ 0.063 & 0.236 $\pm $ 0.033 & - & - & 0.016 $\pm $ 0.021 & - \\
 & Kernel WL Optimal Assignment & 0.889 $\pm $ 0.082 & 0.012 $\pm $ 0.004 & - & - & 0.171 $\pm $ 0.065 & - \\
 & Kernel Pyramid Match & 0.783 $\pm $ 0.098 & 0.038 $\pm $ 0.012 & - & - & 0.010 $\pm $ 0.022 & - \\
 & Kernel WL Optimal Assignment MDS & 0.783 $\pm $ 0.098 & 0.061 $\pm $ 0.022 & 8.316 $\pm $ 6.006 & 8.633 $\pm $ 2.437 & 0.201 $\pm $ 0.101 & 0.189 $\pm $ 0.081 \\
 & Kernel Pyramid Match MDS & 0.767 $\pm $ 0.121 & 0.026 $\pm $ 0.018 & 5.064 $\pm $ 0.996 & 5.336 $\pm $ 1.275 & 0.040 $\pm $ 0.027 & 0.040 $\pm $ 0.055 \\
 & Kernel Neighbourhood Hash & 0.667 $\pm $ 0.090 & -0.025 $\pm $ 0.007 & - & - & 0.007 $\pm $ 0.013 & - \\
 & Kernel Neighbourhood Hash MDS & 0.611 $\pm $ 0.086 & -0.050 $\pm $ 0.011 & 11.336 $\pm $ 6.653 & 1.569 $\pm $ 0.900 & 0.006 $\pm $ 0.018 & 0.020 $\pm $ 0.029 \\
 \midrule
TopER & TopER All Filtrations & 0.872 $\pm $ 0.093 & 0.026 $\pm $ 0.037 & 1.914 $\pm $ 0.364 & 26.117 $\pm $ 5.753 & 0.147 $\pm $ 0.042 & 0.139 $\pm $ 0.037 \\
 & TopER Closeness & 0.783 $\pm $ 0.135 & 0.023 $\pm $ 0.038 & 1.954 $\pm $ 0.446 & 23.042 $\pm $ 5.265 & 0.112 $\pm $ 0.050 & 0.141 $\pm $ 0.043 \\
 & TopER Popularity & 0.706 $\pm $ 0.138 & -0.025 $\pm $ 0.047 & 2.128 $\pm $ 0.522 & 35.234 $\pm $ 8.414 & 0.158 $\pm $ 0.037 & 0.151 $\pm $ 0.034 \\
 & TopER Degree & 0.678 $\pm $ 0.099 & -0.028 $\pm $ 0.031 & 2.310 $\pm $ 0.930 & 23.952 $\pm $ 5.574 & 0.085 $\pm $ 0.041 & 0.111 $\pm $ 0.037 \\
 & TopER Forricci & 0.672 $\pm $ 0.110 & -0.043 $\pm $ 0.022 & 3.328 $\pm $ 1.586 & 11.954 $\pm $ 4.146 & 0.058 $\pm $ 0.029 & 0.060 $\pm $ 0.028 \\
 & TopER Olricci & 0.572 $\pm $ 0.075 & -0.050 $\pm $ 0.022 & 3.725 $\pm $ 1.580 & 9.520 $\pm $ 3.548 & 0.035 $\pm $ 0.029 & 0.052 $\pm $ 0.028 \\
 \midrule
Persistent Homology & PH Alpha Complex & 0.894 $\pm $ 0.072 & 0.183 $\pm $ 0.031 & 1.313 $\pm $ 0.142 & 49.927 $\pm $ 8.714 & 0.200 $\pm $ 0.090 & 0.319 $\pm $ 0.079 \\
 & PH Vietoris Rips & 0.889 $\pm $ 0.079 & 0.083 $\pm $ 0.031 & 2.219 $\pm $ 0.219 & 15.271 $\pm $ 2.565 & 0.049 $\pm $ 0.040 & 0.077 $\pm $ 0.044 \\
 & PH Laplacian & 0.889 $\pm $ 0.102 & -0.092 $\pm $ 0.019 & 3.261 $\pm $ 0.337 & 11.802 $\pm $ 2.092 & 0.015 $\pm $ 0.015 & 0.073 $\pm $ 0.037 \\
 & PH Fricci & 0.828 $\pm $ 0.088 & -0.045 $\pm $ 0.022 & 2.239 $\pm $ 0.203 & 16.281 $\pm $ 3.202 & 0.035 $\pm $ 0.038 & 0.094 $\pm $ 0.035 \\
 & PH Degree & 0.822 $\pm $ 0.096 & -0.095 $\pm $ 0.019 & 3.074 $\pm $ 0.351 & 7.049 $\pm $ 1.181 & 0.008 $\pm $ 0.012 & 0.047 $\pm $ 0.036 \\
 & PH Olricci & 0.817 $\pm $ 0.096 & -0.014 $\pm $ 0.026 & 1.862 $\pm $ 0.140 & 23.389 $\pm $ 4.529 & 0.081 $\pm $ 0.039 & 0.105 $\pm $ 0.033 \\
\bottomrule
\end{tabular}
    }
    \label{tab:simulations results RPG}
\end{table}
\begin{table}[tbh]
    \centering
    \caption{Accuracies and clustering scores for distinguishing parameter choices of \texttt{RG} graphs.
    }
    \resizebox{\textwidth}{!}{
\begin{tabular}{llllllll}
\toprule
 &  & Accuracy & Silhouette Score & Davies-Bouldin & Calinski-Harabasz & ARI agglomerative & ARI k-means \\
 & Method &  &  &  &  &  &  \\
\midrule
Diversity Curves & Shortest Path & \textbf{0.967 $\pm $ 0.044} & 0.591 $\pm $ 0.032 & 0.483 $\pm $ 0.046 & 427.861 $\pm $ 97.530 & 0.486 $\pm $ 0.056 & 0.478 $\pm $ 0.129 \\
 & Shortest Path PCA & \textbf{0.967 $\pm $ 0.044} & \textbf{0.600 $\pm $ 0.033} & \textbf{0.468 $\pm $ 0.046} & \textbf{434.844 $\pm $ 100.391} & 0.486 $\pm $ 0.057 & 0.480 $\pm $ 0.134 \\
 \midrule
Diversity Function & Spread Function Shortest Path & 0.889 $\pm $ 0.066 & 0.093 $\pm $ 0.017 & 2.660 $\pm $ 0.426 & 20.325 $\pm $ 3.851 & 0.110 $\pm $ 0.059 & 0.126 $\pm $ 0.059 \\
\midrule
Embeddings & FEATHER & 0.939 $\pm $ 0.072 & 0.319 $\pm $ 0.033 & 1.173 $\pm $ 0.129 & 87.254 $\pm $ 13.845 & 0.456 $\pm $ 0.073 & 0.489 $\pm $ 0.063 \\
 & NetLSD & 0.922 $\pm $ 0.067 & 0.068 $\pm $ 0.025 & 2.388 $\pm $ 0.317 & 61.887 $\pm $ 12.320 & 0.138 $\pm $ 0.057 & 0.232 $\pm $ 0.071 \\
 & SpectralZoo & 0.906 $\pm $ 0.111 & 0.183 $\pm $ 0.032 & 1.305 $\pm $ 0.187 & 77.332 $\pm $ 13.485 & 0.223 $\pm $ 0.145 & 0.365 $\pm $ 0.037 \\
 & Graph2Vec & 0.578 $\pm $ 0.062 & 0.000 $\pm $ 0.002 & 7.192 $\pm $ 0.355 & 1.483 $\pm $ 0.132 & 0.012 $\pm $ 0.029 & 0.041 $\pm $ 0.038 \\
 \midrule
Kernel & Kernel Core Framework MDS & 0.939 $\pm $ 0.076 & 0.373 $\pm $ 0.030 & 0.891 $\pm $ 0.063 & 99.746 $\pm $ 12.842 & 0.460 $\pm $ 0.151 & \textbf{0.574 $\pm $ 0.146} \\
 & Kernel Core Framework & 0.939 $\pm $ 0.107 & 0.386 $\pm $ 0.032 & - & - & 0.486 $\pm $ 0.139 & - \\
 & Kernel WL Optimal Assignment & 0.928 $\pm $ 0.056 & 0.053 $\pm $ 0.007 & - & - & 0.398 $\pm $ 0.115 & - \\
 & Kernel WL MDS & 0.911 $\pm $ 0.103 & 0.226 $\pm $ 0.037 & 3.921 $\pm $ 2.201 & 14.689 $\pm $ 4.101 & 0.025 $\pm $ 0.058 & 0.362 $\pm $ 0.130 \\
 & Kernel WL Optimal Assignment MDS & 0.906 $\pm $ 0.075 & 0.175 $\pm $ 0.026 & 2.739 $\pm $ 0.885 & 22.119 $\pm $ 3.766 & 0.358 $\pm $ 0.123 & 0.350 $\pm $ 0.119 \\
 & Kernel WL & 0.889 $\pm $ 0.129 & 0.185 $\pm $ 0.030 & - & - & 0.001 $\pm $ 0.005 & - \\
 & Kernel Shortest Path & 0.889 $\pm $ 0.127 & 0.428 $\pm $ 0.049 & - & - & \textbf{0.489 $\pm $ 0.080} & - \\
 & Kernel Shortest Path MDS & 0.861 $\pm $ 0.115 & 0.320 $\pm $ 0.043 & 1.768 $\pm $ 0.438 & 38.574 $\pm $ 7.686 & 0.321 $\pm $ 0.198 & 0.486 $\pm $ 0.055 \\
 & Kernel Pyramid Match & 0.811 $\pm $ 0.075 & 0.040 $\pm $ 0.010 & - & - & 0.028 $\pm $ 0.038 & - \\
 & Kernel Pyramid Match MDS & 0.672 $\pm $ 0.088 & 0.004 $\pm $ 0.014 & 9.754 $\pm $ 6.360 & 3.794 $\pm $ 1.051 & 0.041 $\pm $ 0.030 & 0.025 $\pm $ 0.037 \\
 & Kernel Neighbourhood Hash & 0.578 $\pm $ 0.083 & -0.014 $\pm $ 0.007 & - & - & 0.002 $\pm $ 0.008 & - \\
 & Kernel Neighbourhood Hash MDS & 0.500 $\pm $ 0.119 & -0.044 $\pm $ 0.007 & 16.075 $\pm $ 8.087 & 0.807 $\pm $ 0.521 & -0.004 $\pm $ 0.011 & -0.005 $\pm $ 0.010 \\
 \midrule
TopER & TopER Popularity & 0.833 $\pm $ 0.129 & 0.153 $\pm $ 0.046 & 2.013 $\pm $ 0.627 & 33.288 $\pm $ 6.509 & 0.189 $\pm $ 0.104 & 0.270 $\pm $ 0.055 \\
 & TopER Degree & 0.806 $\pm $ 0.083 & 0.048 $\pm $ 0.039 & 2.340 $\pm $ 0.874 & 21.922 $\pm $ 4.646 & 0.099 $\pm $ 0.039 & 0.151 $\pm $ 0.039 \\
 & TopER All Filtrations & 0.800 $\pm $ 0.100 & 0.094 $\pm $ 0.039 & 2.353 $\pm $ 0.745 & 25.906 $\pm $ 4.954 & 0.136 $\pm $ 0.075 & 0.212 $\pm $ 0.065 \\
 & TopER Closeness & 0.800 $\pm $ 0.097 & 0.058 $\pm $ 0.039 & 2.473 $\pm $ 1.004 & 22.667 $\pm $ 4.479 & 0.105 $\pm $ 0.048 & 0.166 $\pm $ 0.053 \\
 & TopER Forricci & 0.800 $\pm $ 0.051 & 0.119 $\pm $ 0.041 & 2.321 $\pm $ 0.838 & 32.456 $\pm $ 7.086 & 0.189 $\pm $ 0.078 & 0.235 $\pm $ 0.067 \\
 & TopER Olricci & 0.650 $\pm $ 0.119 & 0.029 $\pm $ 0.043 & 7.906 $\pm $ 5.057 & 15.313 $\pm $ 3.525 & 0.073 $\pm $ 0.045 & 0.118 $\pm $ 0.043 \\
 \midrule
Persistent Homology & PH Vietoris Rips & 0.939 $\pm $ 0.084 & 0.114 $\pm $ 0.036 & 1.849 $\pm $ 0.119 & 19.354 $\pm $ 2.626 & 0.044 $\pm $ 0.030 & 0.115 $\pm $ 0.040 \\
 & PH Olricci & 0.900 $\pm $ 0.074 & 0.138 $\pm $ 0.050 & 1.124 $\pm $ 0.075 & 28.997 $\pm $ 5.001 & 0.043 $\pm $ 0.036 & 0.146 $\pm $ 0.052 \\
 & PH Fricci & 0.872 $\pm $ 0.079 & 0.101 $\pm $ 0.047 & 1.568 $\pm $ 0.118 & 13.261 $\pm $ 3.103 & 0.008 $\pm $ 0.014 & 0.097 $\pm $ 0.057 \\
 & PH Laplacian & 0.811 $\pm $ 0.106 & 0.034 $\pm $ 0.047 & 2.030 $\pm $ 0.170 & 6.883 $\pm $ 1.267 & 0.010 $\pm $ 0.014 & 0.043 $\pm $ 0.029 \\
 & PH Alpha Complex & 0.778 $\pm $ 0.090 & 0.105 $\pm $ 0.033 & 1.723 $\pm $ 0.272 & 32.251 $\pm $ 6.547 & 0.130 $\pm $ 0.065 & 0.233 $\pm $ 0.078 \\
 & PH Degree & 0.739 $\pm $ 0.108 & 0.030 $\pm $ 0.046 & 2.535 $\pm $ 0.272 & 4.823 $\pm $ 0.962 & 0.012 $\pm $ 0.016 & 0.029 $\pm $ 0.024 \\
\bottomrule
\end{tabular}
    }
    \label{tab:simulation results RG}
\end{table}
\begin{table}[tbh]
    \centering
    \caption{Accuracies and clustering scores for distinguishing parameter choices in \texttt{SBM} graphs.
    }
    \resizebox{\textwidth}{!}{
\begin{tabular}{llllllll}
\toprule
 &  & Accuracy & Silhouette Score & Davies-Bouldin & Calinski-Harabasz & ARI agglomerative & ARI k-means \\
 & Method &  &  &  &  &  &  \\
\midrule
Diversity Curves & Shortest Path & \textbf{0.856 $\pm $ 0.079} & 0.288 $\pm $ 0.036 & 1.062 $\pm $ 0.115 & 98.923 $\pm $ 15.813 & 0.397 $\pm $ 0.064 & \textbf{0.427 $\pm $ 0.068} \\
 & Shortest Path PCA & 0.844 $\pm $ 0.092 & \textbf{0.296 $\pm $ 0.037} & \textbf{1.038 $\pm $ 0.113} & \textbf{101.116 $\pm $ 16.296} & \textbf{0.406 $\pm $ 0.067} & 0.426 $\pm $ 0.070 \\
 \midrule
Diversity Function & Spread Function Shortest Path & 0.794 $\pm $ 0.079 & -0.000 $\pm $ 0.014 & 5.739 $\pm $ 1.063 & 3.491 $\pm $ 1.383 & 0.007 $\pm $ 0.021 & 0.002 $\pm $ 0.015 \\
\midrule
Embeddings & NetLSD & 0.828 $\pm $ 0.080 & 0.112 $\pm $ 0.043 & 1.899 $\pm $ 0.432 & 20.365 $\pm $ 3.643 & 0.044 $\pm $ 0.027 & 0.095 $\pm $ 0.067 \\
 & FEATHER & 0.656 $\pm $ 0.116 & 0.041 $\pm $ 0.024 & 2.383 $\pm $ 0.462 & 26.831 $\pm $ 5.108 & 0.137 $\pm $ 0.059 & 0.156 $\pm $ 0.040 \\
 & SpectralZoo & 0.633 $\pm $ 0.083 & 0.044 $\pm $ 0.035 & 2.676 $\pm $ 0.855 & 24.065 $\pm $ 4.737 & 0.104 $\pm $ 0.053 & 0.157 $\pm $ 0.043 \\
 & Graph2Vec & 0.622 $\pm $ 0.085 & 0.000 $\pm $ 0.002 & 6.338 $\pm $ 0.297 & 1.882 $\pm $ 0.215 & 0.017 $\pm $ 0.043 & 0.090 $\pm $ 0.042 \\
 \midrule
Kernel & Kernel WL Optimal Assignment & 0.800 $\pm $ 0.075 & 0.009 $\pm $ 0.005 & - & - & 0.145 $\pm $ 0.056 & - \\
 & Kernel Core Framework MDS & 0.700 $\pm $ 0.090 & 0.029 $\pm $ 0.018 & 3.843 $\pm $ 0.664 & 9.330 $\pm $ 2.056 & 0.037 $\pm $ 0.040 & 0.065 $\pm $ 0.035 \\
 & Kernel Shortest Path MDS & 0.694 $\pm $ 0.076 & -0.030 $\pm $ 0.013 & 25.418 $\pm $ 24.425 & 0.710 $\pm $ 0.570 & 0.038 $\pm $ 0.026 & 0.034 $\pm $ 0.026 \\
 & Kernel Core Framework & 0.694 $\pm $ 0.103 & 0.029 $\pm $ 0.021 & - & - & 0.033 $\pm $ 0.038 & - \\
 & Kernel Pyramid Match & 0.672 $\pm $ 0.091 & -0.009 $\pm $ 0.007 & - & - & -0.001 $\pm $ 0.014 & - \\
 & Kernel WL MDS & 0.672 $\pm $ 0.115 & -0.036 $\pm $ 0.015 & 25.410 $\pm $ 68.230 & 0.991 $\pm $ 0.586 & 0.002 $\pm $ 0.014 & 0.015 $\pm $ 0.019 \\
 & Kernel WL & 0.661 $\pm $ 0.088 & -0.017 $\pm $ 0.015 & - & - & 0.009 $\pm $ 0.016 & - \\
 & Kernel Shortest Path & 0.639 $\pm $ 0.080 & -0.062 $\pm $ 0.015 & - & - & 0.017 $\pm $ 0.024 & - \\
 & Kernel WL Optimal Assignment MDS & 0.639 $\pm $ 0.083 & -0.001 $\pm $ 0.017 & 4.799 $\pm $ 0.956 & 6.008 $\pm $ 1.697 & 0.101 $\pm $ 0.054 & 0.102 $\pm $ 0.046 \\
 & Kernel Pyramid Match MDS & 0.622 $\pm $ 0.092 & -0.009 $\pm $ 0.011 & 6.867 $\pm $ 2.211 & 3.509 $\pm $ 0.942 & 0.029 $\pm $ 0.030 & 0.023 $\pm $ 0.034 \\
 & Kernel Neighbourhood Hash & 0.539 $\pm $ 0.117 & -0.020 $\pm $ 0.005 & - & - & -0.002 $\pm $ 0.010 & - \\
 & Kernel Neighbourhood Hash MDS & 0.489 $\pm $ 0.065 & -0.043 $\pm $ 0.006 & 27.497 $\pm $ 47.287 & 0.601 $\pm $ 0.419 & -0.005 $\pm $ 0.009 & -0.006 $\pm $ 0.010 \\
 \midrule
TopER & TopER Degree & 0.700 $\pm $ 0.079 & -0.072 $\pm $ 0.016 & 5.790 $\pm $ 3.233 & 8.171 $\pm $ 3.373 & 0.027 $\pm $ 0.019 & 0.031 $\pm $ 0.022 \\
 & TopER Popularity & 0.661 $\pm $ 0.091 & -0.066 $\pm $ 0.023 & 4.008 $\pm $ 2.100 & 12.780 $\pm $ 3.919 & 0.050 $\pm $ 0.027 & 0.060 $\pm $ 0.030 \\
 & TopER Forricci & 0.622 $\pm $ 0.078 & -0.049 $\pm $ 0.022 & 3.615 $\pm $ 1.488 & 12.333 $\pm $ 4.063 & 0.042 $\pm $ 0.026 & 0.050 $\pm $ 0.027 \\
 & TopER All Filtrations & 0.600 $\pm $ 0.065 & -0.061 $\pm $ 0.019 & 4.282 $\pm $ 1.701 & 11.334 $\pm $ 3.697 & 0.050 $\pm $ 0.028 & 0.046 $\pm $ 0.027 \\
 & TopER Olricci & 0.594 $\pm $ 0.108 & -0.061 $\pm $ 0.019 & 4.774 $\pm $ 3.093 & 8.447 $\pm $ 3.225 & 0.035 $\pm $ 0.025 & 0.036 $\pm $ 0.025 \\
 & TopER Closeness & 0.589 $\pm $ 0.120 & -0.058 $\pm $ 0.028 & 4.114 $\pm $ 2.327 & 13.275 $\pm $ 4.404 & 0.053 $\pm $ 0.032 & 0.071 $\pm $ 0.031 \\
 \midrule
Persistent Homology & PH Degree & 0.717 $\pm $ 0.107 & -0.096 $\pm $ 0.020 & 3.451 $\pm $ 0.386 & 6.487 $\pm $ 1.265 & 0.005 $\pm $ 0.010 & 0.056 $\pm $ 0.039 \\
 & PH Olricci & 0.639 $\pm $ 0.115 & -0.094 $\pm $ 0.026 & 2.729 $\pm $ 0.484 & 12.478 $\pm $ 2.966 & 0.014 $\pm $ 0.016 & 0.081 $\pm $ 0.035 \\
 & PH Laplacian & 0.639 $\pm $ 0.143 & -0.096 $\pm $ 0.018 & 3.627 $\pm $ 0.355 & 6.482 $\pm $ 1.446 & 0.009 $\pm $ 0.013 & 0.051 $\pm $ 0.043 \\
 & PH Vietoris Rips & 0.611 $\pm $ 0.099 & -0.085 $\pm $ 0.030 & 4.222 $\pm $ 1.024 & 4.581 $\pm $ 1.386 & 0.010 $\pm $ 0.013 & 0.016 $\pm $ 0.019 \\
 & PH Alpha Complex & 0.606 $\pm $ 0.128 & -0.053 $\pm $ 0.032 & 4.534 $\pm $ 2.876 & 8.124 $\pm $ 3.283 & 0.025 $\pm $ 0.029 & 0.065 $\pm $ 0.032 \\
 & PH Fricci & 0.600 $\pm $ 0.124 & -0.097 $\pm $ 0.027 & 3.291 $\pm $ 0.536 & 8.157 $\pm $ 1.917 & 0.005 $\pm $ 0.011 & 0.053 $\pm $ 0.033 \\
\bottomrule
\end{tabular}
    }
    \label{tab:simulation results SBM}
\end{table}


\clearpage

\newpage

\subsection{Distinguishing between trajectory or cluster structures in single-cell graphs}
\label{app:sc}
Single-cell RNA sequencing datasets are often characterised as neighbourhood graphs that connect individual cells based on their molecular similarity. These single-cell graphs can further be characterised as displaying clustering-like  or trajectory-like geometries \citep{lim2024quantifying}. Clustering-like graphs are dominated by clearly separated clusters of cells. Trajectory-like patterns are characterised by continuous transitions between cell states that form continuously connected trajectories on the single-cell graphs \citep{saelens2019comparison}. Knowing which type of geometry a graph displays directly gives actionable insights into which type of downstream analysis and visualisation methods are most appropriate for a given dataset. However, despite its practical relevance, the task of distinguishing between trajectory and clustering structures has not yet been solved. The most recent and to our knowledge only study on this task by \citet{lim2024quantifying} considers a collection of 169 single-cell datasets, out of which 85 show cluster-like structures and 84 show trajectory-like patterns. These datasets have been collected from a wide range of existing studies on evaluating trajectory inference methods \citep{saelens2019comparison}, evaluating clustering methods \citep{kim2019impact, krzak2019benchmark}, and across a list of further datasets \citep{lim2024quantifying, the2022tabula, kowalczyk2015single, nakamura2016developmental, nakamura2017single, han2018mapping, buettner2015computational, leng2015oscope, tian2019benchmarking, frazer2017transcriptomic, hochgerner2018conserved, treutlein2014reconstructing, joost2016single, treutlein2016dissecting, guo2015transcriptome, li2017single, yang2017single, petropoulos2016single, park2018single, saliba2016single, trapnell2014dynamics, loh2016mapping, hayashi2018single, burns2015single, engel2016innate, li2017classifying, marques2016oligodendrocyte, qiu2017deciphering, plass2018cell, shalek2014single, goolam2016heterogeneity, deng2014single, biase2014cell, zeisel2015cell, klein2015droplet, darmanis2015survey, tasic2016adult, tirosh2016dissecting, camp2015human, chu2016single, xin2016rna, li2017reference, gokce2016cellular, granit2018regulation, habib2016div, close2017single, villani2017single, olsson2016single}. 
Out of the final collection of datasets, we consider two versions: First, we consider all 169 datasets altogether. Second, we consider a subset of 27 trajectory-like datasets whose trajectory structure has been verified gold-standard through biological validation \citep{saelens2019comparison} as well as 60 cluster-like datasets, which have been annotated as "organs" or "clusters" by \citet{lim2024quantifying}. Overall, this leaves a subset of 87 graphs, whose annotations are presumed to be less ambiguous.


Previously, \citet{lim2024quantifying} reach a classification accuracy of around 57\% on the collection of all 169 single-cell datasets as reported in their publication. Specifically, \citet{lim2024quantifying} used a set of 5 geometry-inspired scores to construct a landscape i.e. a shared representation of real single-cell datasets as well as simulated trajectory- or cluster-like datasets. They then use kNN-classification to predict the type of the real datasets based on the simulated examples, leading to a substantial proportion of graphs being misclassified. Given that \citet{lim2024quantifying} provide exact numbers on how many graphs per category were classified correctly or incorrectly in Figure 6.c of their manuscript, we use these numbers to simulate the performance of their Landscape approach across 5-fold stratified cross validation and report the estimated accuracies in \Cref{tab:sc_all_app}. Overall, we note that this reaches only about $57\%$ accuracy. We aim to improve upon these results in our own study by investigating this classification task across a broader range of baseline methods.




Following \citet{lim2024quantifying}, we download their dataset of processed single-cell datasets, and the set of scores they computed for each dataset, which have been made available with their codebase\footnote{Made available at \href{https://github.com/pqiu/Quantifying_clusterness_trajectoriness}{https://github.com/pqiu/Quantifying\_clusterness\_trajectoriness}  with the publication of their paper \citep{lim2024quantifying}.} and 
their publication. 
Preprocessing and quality control for the RNA sequencing datasets has been applied as described by \citet{saelens2019comparison}.   \citet{lim2024quantifying} further used PCA to reduce the dimensionality of each dataset to 20 principal components. Given the substantial size differences between datasets, ranging between around 19953 and 56 observations, we apply geometric sketching to subsample each dataset to at most 200 cells \citep{hie2019geometric}. 
Further, we convert each dataset into a graph by connecting each cell to its 10 closest neighbours.
\Cref{fig:sc_curves} then shows representative examples of four of these processed single-cell graphs, two trajectory-like graphs (A: cellbench-SC2\_luyitian, GSE118767, simulated trajectory-like human cells \citep{tian2019benchmarking, saelens2019comparison}; B: aging-hsc-young\_kowalczyk, GSE59114, hematopoietic stem cell of mice \citep{kowalczyk2015single, saelens2019comparison}), as well as two cluster-like graphs (
C: mouse-cell-atlas-combination-5, cells from a mouse tissue atlas \citep{han2018mapping, saelens2019comparison}; D: GSE82187, mouse brain cells \citep{gokce2016cellular, kim2019impact}). Across these examples, we see that cluster-like graphs have distinct communities, while trajectory-like graphs display continuous transitions between cells. Our aim is now to distinguish these two types of graph through their structure. 

We input these graphs into graph embedding models, or kernel methods. Further, we compute diversity curves with shortest-path distances, diffusion distances, or Euclidean distances between the scaled node features. The resulting mean diversity curves per group and their standard deviation are shown in \Cref{fig:sc_curves_app}. 
These diversity curves demonstrate that cluster-like single-cell graphs have on average higher structural diversity as measured by diversity curves, which use adjacency-based distances, such as shortest-path or diffusion distances. This is motivated by the fact that clusters of cells in these graphs are more distinct and more well separated structurally. These trends carry over to the PCA plot in \Cref{fig:sc_curves}, which is computed from the standard scaled diversity curves using shortest-path as well as feature-based distances. We further match each of the four example graphs A to D to their corresponding representation in the PCA visualisation. Overall, this demonstrates visually that diversity curves distinguish the geometry of single-cell graphs.

Beyond the results reported in our main experiments, we further provide comparisons with feature-based baselines. Specifically, we apply mean pooling across the node features of the downsampled single-cell datasets as well as on the standard scaled version of these features. Comparing graph-based embeddings with these feature baselines then helps to assess to what extent the construction of neighbourhood graphs aids with distinguishing the geometries of single-cell datasets.

\begin{table}[t]
\centering
\caption{Model performance metrics (mean $\pm$ std) at distinguishing between trajectory- and cluster-like single-cell graphs using kNN-classification across 5-fold stratified cross validation. 
}
\resizebox{\textwidth}{!}{
\begin{tabular}{llllll}
\toprule
                 & Method                        & Accuracy - All Graphs & AUROC - All Graphs & Accuracy - Gold & AUROC - Gold  \\
                 \midrule
Diversity Curves & Diversity Curves All Metrics  & \textbf{0.766 ± 0.062}          & \textbf{0.810 ± 0.068}      & \textbf{0.860 ± 0.062}   & 0.824 ± 0.068 \\
                 & Shortest Path + Feature Dist. & 0.759 ± 0.077           & 0.808 ± 0.081      & \textbf{0.860 ± 0.077}    & 0.831 ± 0.081 \\
                 & Diffusion Dist.               & 0.743 ± 0.061         & 0.786 ± 0.065      & 0.807 ± 0.061   & 0.795 ± 0.065 \\
                 & Feature Dist.                     & 0.721 ± 0.075         & 0.759 ± 0.089      & 0.843 ± 0.075   & 0.809 ± 0.089 \\
                 & Shortest Path                 & 0.682 ± 0.095         & 0.723 ± 0.115      & 0.773 ± 0.095   & 0.757 ± 0.115 \\
                 \midrule
Embeddings       & NetLSD                        & 0.731 ± 0.047         & 0.802 ± 0.057      & 0.823 ± 0.047   & \textbf{0.848 ± 0.057} \\
                 & FEATHER                       & 0.715 ± 0.071         & 0.771 ± 0.085       & 0.830 ± 0.071   & 0.824 ± 0.085  \\
                 & SpectralZoo                  & 0.675 ± 0.074         & 0.734 ± 0.095      & 0.821 ± 0.074   & 0.799 ± 0.095 \\
                 & Graph2Vec                     & 0.622 ± 0.074         & 0.648 ± 0.088       & 0.741 ± 0.074   & 0.790 ± 0.088  \\
                 \midrule
Scores \citep{lim2024quantifying}          & All Scores                    & 0.716 ± 0.053         & 0.760 ± 0.061      & 0.804 ± 0.053   & 0.797 ± 0.061 \\
                 & PH Score                      & 0.697 ± 0.058          & 0.747 ± 0.066      & 0.829 ± 0.058    & 0.830 ± 0.066 \\
                 & Reachability Score            & 0.632 ± 0.070         & 0.667 ± 0.087       & 0.727 ± 0.070   & 0.709 ± 0.087 \\
                 & Vector Norm Score             & 0.626 ± 0.074         & 0.671 ± 0.102      & 0.694 ± 0.074   & 0.668 ± 0.102 \\
                 & Diffusion Distance Score      & 0.520 ± 0.089         & 0.516 ± 0.100      & 0.611 ± 0.089    & 0.600 ± 0.100 \\
                 & Landscape                     & 0.570 ± 0.025         & 0.586 ± 0.037                   &         -        &   -            \\
                 & Ripley's K Score              & 0.532 ± 0.079         & 0.559 ± 0.088      & 0.598 ± 0.079   & 0.458 ± 0.088 \\
                 \midrule
Kernel           & Kernel Pyramid Match          & 0.685 ± 0.092         & 0.721 ± 0.081      & 0.784 ± 0.092   & 0.757 ± 0.081 \\
                 & Kernel WL                     & 0.652 ± 0.079         & 0.699 ± 0.099       & 0.785 ± 0.079   & 0.773 ± 0.099   \\
                 & Kernel Shortest Path          & 0.649 ± 0.085          & 0.690 ± 0.104      & 0.782 ± 0.085   & 0.770 ± 0.104 \\
                 & Kernel Core Framework         & 0.649 ± 0.085          & 0.690 ± 0.104      & 0.782 ± 0.085   & 0.770 ± 0.104 \\
                 & Kernel WL Optimal Assignment  & 0.600 ± 0.062         & 0.626 ± 0.083      & 0.679 ± 0.062   & 0.653 ± 0.083 \\
                 & Kernel Neighborhood Hash      & 0.568 ± 0.082         & 0.588 ± 0.083        & 0.704 ± 0.082   & 0.667 ± 0.083   \\
                 \midrule
TopER            & TopER All Filtrations         & 0.624 ± 0.086         & 0.669 ± 0.121      & 0.672 ± 0.086   & 0.649 ± 0.121 \\
                 & TopER Olricci                 & 0.586 ± 0.058         & 0.607 ± 0.073      & 0.649 ± 0.058   & 0.659 ± 0.073 \\
                 & TopER Forricci                & 0.581 ± 0.053           & 0.600 ± 0.064      & 0.670 ± 0.053   & 0.610 ± 0.064  \\
                 & TopER Popularity              & 0.516 ± 0.079         & 0.543 ± 0.090      & 0.662 ± 0.079   & 0.563 ± 0.090 \\
                 & TopER Closeness               & 0.505 ± 0.069         & 0.520 ± 0.076      & 0.624 ± 0.069   & 0.476 ± 0.076  \\
                 & TopER Degree                  & 0.503 ± 0.090         & 0.521 ± 0.089      & 0.630 ± 0.090   & 0.480 ± 0.089\\
                 \midrule
Node Features            & Mean Pooling        & 0.619 ± 0.091         & 0.652 ± 0.103      & 0.661 ± 0.091   & 0.614 ± 0.103 \\
                 &  Standard Scaling + Mean Pooling                & 0.559 ± 0.068         & 0.573 ± 0.006      & 0.553 ± 0.068   & 0.585 ± 0.060 \\
                 \bottomrule
\end{tabular}}
\label{tab:sc_all_app}
\end{table}

\begin{figure}[t]
    \centering
    \includegraphics[width=0.32\textwidth]{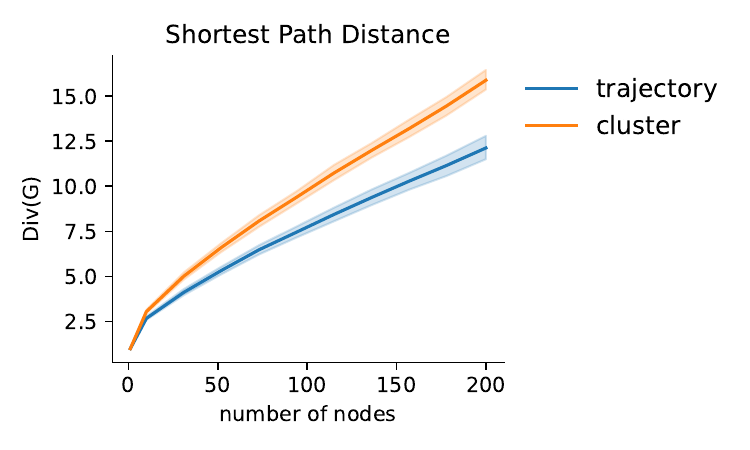}
    \includegraphics[width=0.32\textwidth]{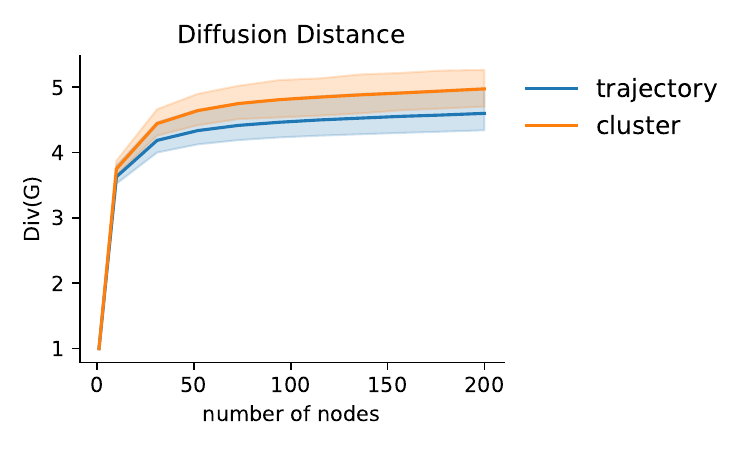}
    \includegraphics[width=0.32\textwidth]{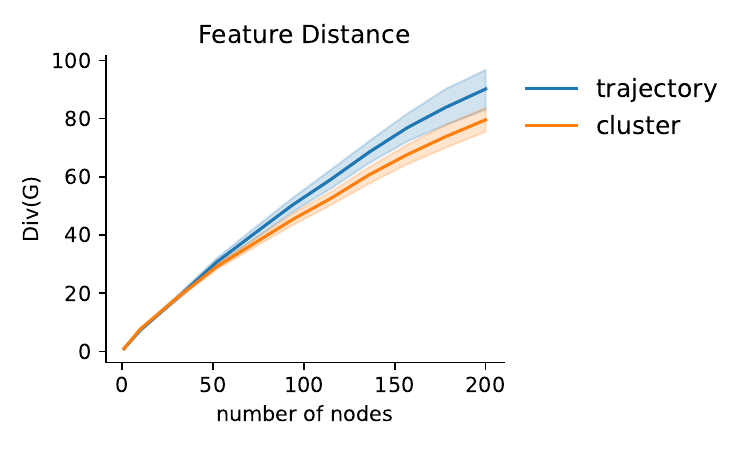}
    \caption{Mean diversity curves per trajectory type computed using shortest-path distances, diffusion distances, or 
    Euclidean distances based on standard scaled node features.
    }
    \label{fig:sc_curves_app}
\end{figure}

\Cref{tab:sc_all_app} reports kNN-classification results for each method across 5-times repeated 5-fold stratified cross-validation, where 64\% of the data is used for training, 16\% for validation, and 20\% for testing. The number of k nearest neighbours is tuned by grid search on the validation dataset in a range of $\{1, 3, 5, 7, 9\}$. We use a standard scaling transform on the diversity curves as this transformation on the training set helps highlight differences across graph sizes. Finally, \Cref{tab:sc_all_app} shows the performance of each method, including the geometry scores by \citet{lim2024quantifying}, in terms of both classification accuracies and AUROC scores.


\subsection{Characterising molecular graph datasets}


To describe structural similarities between enzyme classes, we download the \texttt{PROTEINS} and \texttt{ENZYMES} datasets \citep{borgwardt2005protein, dobson2003distinguishing}, which are available under the TUDataset \citep{morris2020tudataset} collection through PyTorch Geometric (available under an MIT Licence). 
We map the classes from the \texttt{PROTEINS} graphs to the labels "enzymes" or "non-enzymes" by checking with the original publications \citep{borgwardt2005protein, dobson2003distinguishing} and verifying the number of graphs in each class matches what is described by \citet{borgwardt2005protein}. Classes in \texttt{ENZYMES} are characterised by their Enzyme Commission (EC) number. We then compare the class of proteins that are not enzymes, to the classes of proteins that are enzymes across both datasets. To do so, we compute diversity curves using shortest-path distances on the graphs across node ranges from 1 to 34 nodes. The mean curves per class are visualised in \Cref{fig:datasets_app} and highlight that non-enzymes have notably lower structural diversity on average than enzyme graphs. To verify this, we implement the statistical test for equality in the mean diversity curves as described in \Cref{sec:enzymes}. 
This case study demonstrates that diversity curves can be used for statistical testing, which highlights their utility for data-driven and uncertainty-aware analysis. Further, diversity curves directly allow for comparisons between differently sized graphs and graph distributions across varying cardinalities. Based on this, we believe there is promising potential for further applications to describe, explore, and compare the structural properties of graph datasets, both to describe 
the quality of graph datasets \citep{coupette2025no}, as well as towards evaluating structural similarities for transfer learning tasks \citep{huang2023learning}.

\begin{figure}[t]
    \centering
    \includegraphics[width=0.8\linewidth]{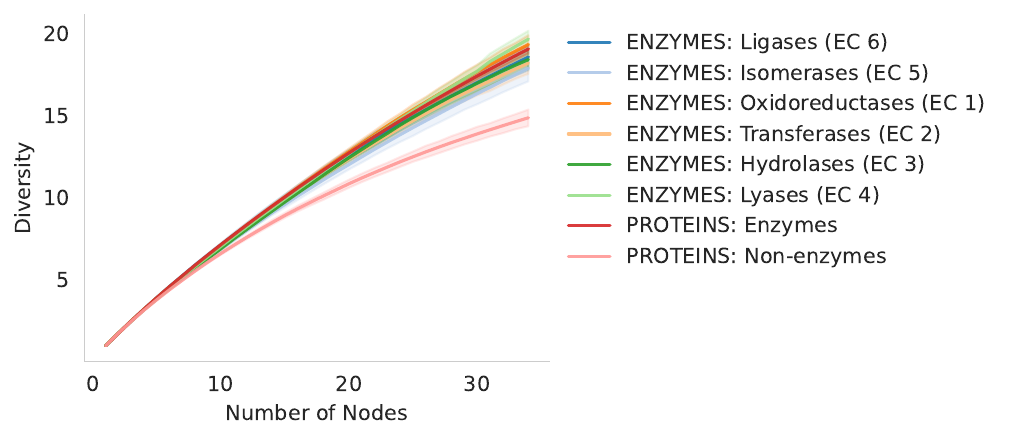}
    \caption{Mean diversity curves per class in the \texttt{PROTEINS} and \texttt{ENZYMES} datasets.}
    \label{fig:datasets_app}
\end{figure}

\subsection{Distinguishing geometric shapes}


To test whether our methods encode the inherent structure of data, we use the \texttt{MANTRA}\footnote{\href{https://github.com/aidos-lab/MANTRA}{https://github.com/aidos-lab/MANTRA} available under a BSD-3-Clause Licence.}
dataset \citep{ballester2024mantra}, which has been proposed to benchmark the capability of geometric and topological deep learning models to distinguish between combinatorial triangulations of manifolds. \Cref{tab:mantra_details} specifically shows the number of triangulations of 2-manifolds in the \texttt{MANTRA} dataset, out of which triangulations of the Klein bottle ($K$), torus ($T^2$), sphere ($S^2$), and real projective plane ($\mathrm{RP}^2$) are associated with explicit homeomorphism types that characterise their shape. For our main experiments, we address the task of distinguishing between these four types.

We follow the experimental setup by \citet{ballester2024mantra} to evaluate each classification model across 5-fold stratified cross validation, taking 20\% of the data for testing, 16\% for validation, and 64\% for training. In this manner, we train Support Vector Machines with radial basis function kernels using default parameters as implemented in scikit-learn (version 1.8.0) unless mentioned otherwise. For hyperparameter tuning, we use randomised grid search to select the cost parameter C in $[0.01, 100]$ and gamma in $[0.0001, 100]$ with log-uniform probabilities across 20 random iterations and 3-fold cross validation on the validation set. We let our models predict probabilities to report the one-vs-rest multi-class AUROC score on the test set, and further report the test accuracies to compare with the results by \citet{ballester2024mantra}. 
All performance scores for the Cellular Transformer (CT) are cited directly from Table 11 in \citet{ballester2024mantra}, where this model is the best performing architecture at the intended task. 

To compute diversity curves, we use edge contraction for coarsening as described throughout our paper. But we also note that some triangulations in \texttt{MANTRA} only have very few nodes making direct comparisons with larger graphs difficult. Therefore, we also employ upsampling operations via upsampling each 3-clique in the graph, or each triangle in the triangulation as described in \Cref{app:coarsening}. Specifically, we uniformly select triangles making sure that each triangle is picked at least once before a previously upsampled triangle could be selected again, and subdivide the selected triangle by adding a central vertex and connecting it to all corners. This way, we uniformly subdivide the surface of the meshes ($\mathcal{T}$), or the 3-cliques in the graphs ($\mathit{G}$). 

We report three versions of our results:

$\mathit{G}$: First, to train on graphs, denoted by $\mathit{G}$, we convert each triangulation into 
their one-skeleton. Then we compute diversity curves based on shortest-path distances on these graphs, for cardinalities between 1 and 15, the maximum number of nodes in the graphs. 
Note that for this version of the task, tori and Klein bottles will have identical 1-skeleton graphs, and thus be indistinguishable by their graph structure alone \citep{ballester2024mantra}. Based on this observation, 
we choose to report the performance of an optimal or dummy classifier in \Cref{tab:mantra}, whose performance we simulate 
by assuming that all classes other than tori get classified correctly with probability $1$, while the triangulations of a torus  get confidently misclassified as Klein bottles again with probability $1$, which we use to simulate both the accuracy and the multi-class AUROC of this classifier baseline.

 $\mathcal{T}$: Second, we input the triangulations themselves into our model, denoted by $\mathcal{T}$, which include information on nodes, edges, and faces allowing us to integrate higher order features into the predictions. Specifically, we compute heat kernel distances between faces to compute diversity curves on these higher order structures for cardinalities between 3 and 15. We then concatenate these curves with the previously computed graph-based diversity curves, as inputs for our prediction models.

\begin{table}[t]
\centering
\caption{The number of triangulations of 2-manifolds per shape for all classes with more than 100 objects in the \texttt{MANTRA} dataset.}
    \resizebox{\textwidth}{!}{
\begin{tabular}{lrrrrrrrrr}
\toprule
Surface & $\#^4 \mathrm{RP}^2$ & $\#^3 \mathrm{RP}^2$ & $\#^5 \mathrm{RP}^2$ & Klein bottle & $T^2$ & $\mathrm{RP}^2$ & $\#^6 \mathrm{RP}^2$ & $\#^2 T^2$ & $S^2$ \\
\midrule
Number  & 13695 & 11917 & 7053 & 4657 & 2247 & 1365 & 1022 & 868 & 308 \\
\bottomrule
\end{tabular}
}
\label{tab:mantra_details}
\end{table}

\begin{table}[t]
\centering
\caption{Model performance (mean $\pm$ std) at distinguishing triangulations of 2-manifolds across stratified 5-fold CV.}

\begin{tabular}{llcc}
\toprule
Method & Version & Accuracy  & AUROC \\
\midrule
Diversity Curve  
& $\mathcal{T}$ & 1.000 $\pm$ 0.001  & 1.000 $\pm$ 0.000 \\
\midrule
Diversity Curve
& $\mathit{G}$ & 0.928 $\pm$ 0.001 & 0.975 $\pm$ 0.002 \\
 Optimal Classifier 
 & $\mathit{G}$ &  0.948 $\pm$ 0.000  & 0.941 $\pm$ 0.000\\
\midrule
Diversity Curve 
& $\mathcal{T}_{\text{up}}$  & 0.3882 $\pm$ 0.006  & 0.841 $\pm$ 0.014 \\
\bottomrule
\end{tabular}
\label{tab:mantra_v2}
\end{table}


 $\mathcal{T}_{\text{up}}$: Further, we apply barycentric subdivision to subdivide each mesh. As described by \citet{ballester2024mantra}, we then aim to address the task of training our models on the original dataset, but testing it on these subdivided meshes. Effectively, models that can solve this challenge have learned the intrinsic shape of each object while being robust to differences in scale and cardinality. This version of the task is denoted by  $\mathcal{T}_{\text{up}}$. We compute diversity curves using both shortest-path distances on the graphs, and higher order heat kernel distances on the faces. For computing these higher-order distances we set $t=20$ and $k=2$ to compute a metric between triangles. Further, we use triangular upsampling to compute diversity curves for node sizes from 1 to 15 with a step size of one, and 20 to 60 with a step size of 5, extending the evaluation scales to capture trends at higher cardinalities in a size-aware manner. The final results for each version of the task are reported in \Cref{tab:mantra}. Further, in this scenario we apply standard scaling to the test and the training set independently to specifically account for domain-shift between the original shapes and upsampled shapes as we expect their diversity curves to follow the same overall trends but across varying ranges.

Finally, we extend the three tasks above to a larger set of triangulations in \texttt{MANTRA}. Instead of just taking the four shapes with known homeomorphism types, we also consider all other classes of 2-manifolds in \texttt{MANTRA} that consist of more than 100 elements. Results for this extended version are reported in \Cref{tab:mantra_v2}.


\subsection{Expressivity}
\label{app:expressiv_experiment}

The \texttt{BREC}\footnote{\href{https://github.com/GraphPKU/BREC}{https://github.com/GraphPKU/BREC} available under an MIT licence.} dataset \citep{wang2023empirical}
consists of 400 pairs of graphs, which are used to test the expressivity of graph learning methods. All pairs of graphs are 1-WL indistinguishable, and 270 remain 3-WL indistinguishable \citep{wang2023empirical}. We first collect results that have been reported in existing literature and summarise how many pairs of graphs each method can distinguish. 

\citet{wang2023empirical} report a range of graph learning models, for which we report the three best performing GNNs, I2-GNN, KP-GNN, and GSN, as well as Graphormer as a notably worse-performing transformer baseline. 
Relevant to our use of edge contraction pooling for comparing graphs, further work by \citet{lachi2025expressive}
has investigated under which conditions graph pooling operators do not only maintain but also improve expressivity. Their results prove that edge-based pooling in particular increases expressivity. In contrast, node-based pooling operators that do not incorporate any information beyond WL into their pooling choices, such as DiffPool which clusters nodes by utilising MPNNs, cannot improve expressivity. In fact, results by \citet{lachi2025expressive} show DiffPool distinguishes no pair of graphs in \texttt{BREC} correctly, whereas EdgePool as well as their own edge-based pooling operator, EXpressive Pooling (XP), both improve expressivity on \texttt{BREC}. These results are summarised in \Cref{tab:brec} as cited from the original publications \citep{wang2023empirical, lachi2025expressive}.

Further, we report results by \citet{ballester2025expressivity} on the performance of alternative geometric methods, namely persistent homology (PH) at distinguishing graphs from \texttt{BREC}. Specifically, we report the performance of different filtrations when filtering through all 2-simplices in the graphs, which represent 3-cliques, and computing persistent homology up to dimension 2, which tracks the evolution of connected components, circles, and two-dimensional voids across the filtration. See \citet{ballester2025expressivity} for a more detailed description. 

Motivated by our theoretical results on the expressivity of diversity curves, we investigate the performance of our methods on \texttt{BREC} in practice. 
First, we check if spread itself can distinguish each pair of graphs by testing whether their difference in spread is above a given error tolerance  \citep{southern2023curvature}, which we set to $10^{-5}$. We try this for shortest-path, diffusion, and heat kernel distances on the graph, and note that the latter two metric choices can distinguish a larger set of graphs based on their structural diversity alone. We reason this is likely because both are computed from the normalised graph Laplacian, which can be more sensitive to subtle changes in graph structure, than shortest-path distances. Nevertheless, even structural diversity based on shortest-path distances can already distinguish 1-WL indistinguishable graphs.

\begin{table}[t]

\caption{Number of pairs from the \texttt{BREC} dataset distinguished via structural diversity or diversity curves. Compared to representative results from \citep{lachi2025expressive, wang2023empirical,southern2023curvature, ballester2025expressivity}.}
\centering
\resizebox{\textwidth}{!}{
\begin{tabular}{llccccc}
\toprule
 Method & Details & Basic (60) & Regular (140) & Extension (100) & CFI (100) & All (400) \\
\midrule
3 WL &(from \citep{wang2023empirical})  & 60 (100.0\%)  & 50 (35.7\%) & 100 (100.0\%)  & 60 (60.0\%) & 270 (67.5\%)\\
1 WL &(from \citep{southern2023curvature})  & 0 (0.0\%)  & 0 (0.0\%) & 0 (0.0\%)  & 0 (0.0\%) & 0 (0.0\%)\\
\midrule
I2-GNN &(from \citep{wang2023empirical})  & 60 (100.0\%)  & 100 (71.4\%) & 100 (100.0\%)  & 21 (21.0\%) & 281 (70.2\%)\\
KP-GNN &(from \citep{wang2023empirical})  & 60 (100.0\%)  & 106 (75.7\%) & 98 (98.0\%)  & 11 (11.0\%) & 275 (68.8\%)\\
GSN &(from \citep{wang2023empirical})  & 60 (100.0\%)  & 99 (70.7\%) & 95 (95.0\%)  & 0 (0.0\%) & 254 (63.5\%)\\
Graphormer &(from \citep{wang2023empirical})  & 16 (26.7\%)  & 12 (8.6\%) & 41 (41.0\%)  & 10 (10.0\%) & 79 (19.8\%)\\
\midrule
CurvPool &(from \citep{lachi2025expressive})  & 35 (58.3\%)  & 65 (46.4\%) & 53 (53.0\%)  & 4 (4.0\%) & 157 (39.3\%)\\
XP &(from \citep{lachi2025expressive})  & 48 (80.0\%)  & 3 (2.1\%) & 72 (72.0\%)  & 1 (1.0\%) & 124 (31.0\%)\\
EdgePool &(from \citep{lachi2025expressive})  & 13 (21.7\%)  & 0 (0.0\%) & 10 (10\%)  & 0 (0.0\%) & 23 (5.8\%)\\
DiffPool &(from \citep{lachi2025expressive})  & 0 (0.0\%)  & 0 (0.0\%) & 0 (0.0\%)  & 0 (0.0\%) & 0 (0.0\%)\\
\midrule
PH Laplacian & (from \citep{ballester2025expressivity}) & 60 (100\%) & 49 (35\%) & 100 (100\%) &6 (6\%) & 216 (54\%) \\
PH Ollivier–Ricci Curvature & (from \citep{ballester2025expressivity}) & 60 (100\%) & 54 (38.6\%) & 92 (92\%) & 3 (3\%) & 208 (52\%) \\
PH Forman–Ricci Curvature & (from \citep{ballester2025expressivity}) & 59 (98\%) & 50 (35.7\%) & 59 (59\%) & 3 (3\%) & 172 (43\%) \\
PH Degree Filtration & (from \citep{ballester2025expressivity}) & 47 (78\%) & 39 (27.9\%) & 26 (26\%) & 3 (3\%) & 116 (29\%) \\
PH Vietoris-Rips & (from \citep{ballester2025expressivity}) & 31 (52\%) & 39 (27.9\%) & 11 (11\%) & 3 (3\%) & 84 (21\%) \\
\midrule
NetLSD & \citep{tsitsulin2018netlsd,rozembereczki2020karateclub}& 60 (100.0\%)  & 50 (35.7\%) & 100 (100.0\%)  & 3 (3.0\%) & 213 (53.2\%)\\
FEATHER & \citep{rozemberczki2020characteristic,rozembereczki2020karateclub} & 52 (86.7\%)  & 46 (32.9\%) & 5 (5.0\%)  & 0 (0.0\%) & 103 (25.8\%)\\
\midrule
Div(G) & heat kernel distance & 60 (100.0\%) & 50 (35.7\%) &  92 (92.0\%) & 3 (3.0\%) & 205 (51.3\%)\\
Div(G) &diffusion distance & 60 (100.0\%) & 49 (35.0\%) &  84 (84.0\%)  & 3 (3.0\%) & 196 (49.0\%)\\
Div(G) &shortest path metric & 16 (26.7\%) & 13 (9.3\%) &  41 (41.0\%)& 6 (6.0\%)& 76 (19.0\%)\\ 
\midrule
DivCurve$(\mathcal{G}_{\{n-1, n\}})$ &heat kernel distance & 60 (100.0\%) & 65 (46.4\%) & 96 (96.0\%) &  3 (3.0\%) & 224 (56.0\%) \\ 
DivCurve$(\mathcal{G}_{\{n-1, n\}})$ &diffusion distance & 60 (100.0\%) & 58 (41.4\%) & 94 (94.0\%) & 3 (3.0\%) & 215 (53.8\%)\\ 
DivCurve$(\mathcal{G}_{\{n-1, n\}})$ &shortest path metric & 60 (100.0\%) & 48 (34.3\%) & 87 (87.0\%) & 8 (8.0\%) & 203 (50.8\%) \\ 

\bottomrule
\end{tabular}}
\label{tab:brec}
\end{table}

Next, we extend our experiment to investigate whether edge contraction pooling can improve expressivity. To do so, we check each pair of graphs that cannot be distinguished through their spread alone, and iterate through all possible choices for contracting one edge, and compute the average diversity curve at $n-1$ for each graph, where $n$ denotes the cardinality. Based on this, we again check if the difference in diversity curves at $n-1$ is larger than the error tolerance $10^{-5}$. Across different choices of graph metric, we observe that edge contraction pooling enables diversity curves to distinguish more pairs of graphs than structural diversity alone. In particular, diversity curves based on shortest-path distances on the graph benefit from this coarsening approach. 

For each comparison, we further check a reliability set to ensure that within the chosen error tolerance, pairs of isomorphic graphs are not distinguished. This empirical check holds across all comparisons verifying the validity of our results. From a theoretical perspective, we know that isomorphic graphs have the same deck of edge contractions \citep{poirier2018reconstructing}, and thus have the same average diversity across all possible edge contraction choices as further described in \Cref{app:invariant}. Note that 
the expressivity results for diversity curves presented in \Cref{tab:brec} represent a best case scenario  where the sequences of coarsened graphs for which diversity is computed contain all possible edge contractions. This allows us to compare the difference in diversity curves exactly rather than requiring a more complex statistical approach to account for randomness in the graph embeddings, which necessitates more stringent error control and can be unreliable \citep{wang2023empirical}. 

For further comparison, we repeat the same procedure above on deterministic graph embedding methods, NetLSD and FEATHER, by computing the embedding vector for each graph and testing if the $L^{2}$-norm between embedding vectors is above the error threshold $10^{-5}$. 

Finally, \Cref{tab:brec} reports how well each of the baseline methods or our diversity curves can distinguish pairs of graphs from \texttt{BREC}. Overall, we find that diversity curves perform better or competitively to existing geometry-aware baseline methods, such as persistent homology, or graph embedding methods, such as NetLSD and FEATHER. In fact, results improve upon any previously reported performance of edge pooling layers, and are better than certain baselines reported for notably more parametrised models, such as Graphormer. Overall, we thus see diversity curves improve the expressivity of diversity through edge contraction pooling in practice, while ensuring competitive performance on \texttt{BREC}.



\subsection{Computational efficiency}
\label{app:empirical_costs}
\begin{figure}[b]
\centering
\includegraphics[width=0.7\textwidth]{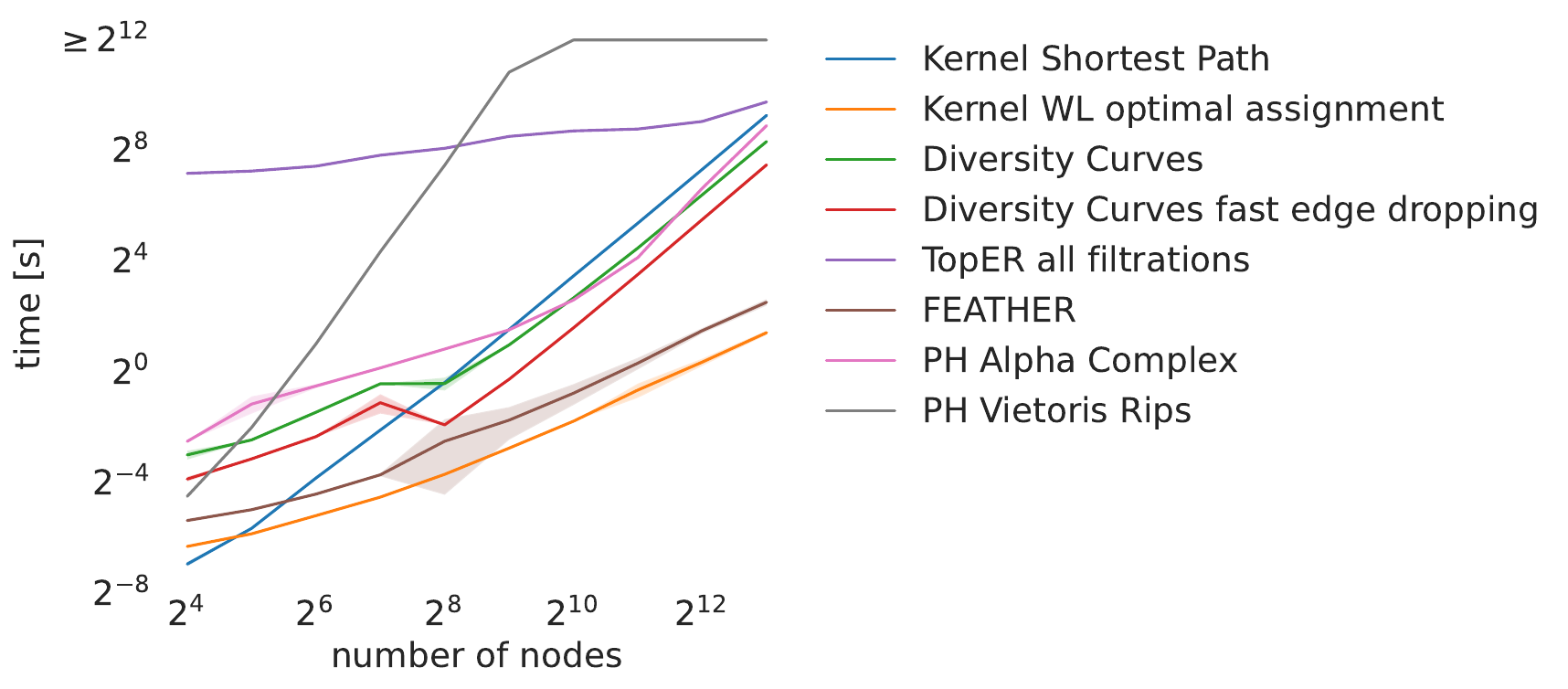}
\captionof{figure}{Empirical runtimes of computing the diversity curves and baselines on \texttt{ER} graphs with increasing number of node sizes in a log-log-plot.}
\label{fig:empirical_time}
\end{figure}
To investigate computational costs empirically, we compare the runtime of our method to relevant baseline methods. 
We generate $30$ Erd\H{o}s-R\'{e}nyi graphs $G_{n,p}$ for $n \in \{2^i|i\in \{3,\dots,13\}\}$ and $p=1/n$. For each method, we average the runtime over $10$ runs.
We choose to evaluate diversity curves at the cardinalities $\mathcal{I} = \{20\cdot i+10| 0 \leq i \leq 4\}$, hence the maximal cardinality at which we compute the diversity is $n_{max} = 90$.
We include a version of the diversity curves computation with improvements in the computational efficiency. Concretely, the edge dropping is implemented more efficiently using a random permutation of the edges rather than first generating and then sorting the edge scores, and edges are collapsed as long as possible, stopping as soon as the minimal cardinality of the evaluation scales is reached. This speed-up is also discussed in \Cref{app:computational_complexity}. 

The results can be seen in \Cref{fig:empirical_time}. While certain variations of kernel methods and persistent homology are faster to compute than the diversity curves for this experimental setup, we note that we expect kernels to scale worse with a larger dataset size. Further, since our configuration of the diversity curves relies on the shortest-path distance in the graph, when comparing with other shortest-path based methods, such as persistent homology using the Vietoris-Rips filtration, we note that our method outperforms this baseline especially at larger graph sizes.
Hence, we conclude that diversity curves 
can still be reasonably efficiently computed for larger graphs, and are slightly faster than for example the shortest-path kernel or persistent homology using the alpha complex. 
Because coarsening is an integral component of our framework, computational costs can be efficiently 
reduced by preprocessing graphs to smaller sizes, or evaluating diversity curves at fewer evaluation scales or lower node numbers.
To compute persistent homology for this experiment, we build a clique complex of the graph by adding all $2$-simplices, which correspond to $3$-cliques in the graph, and compute the persistent homology on the filtrations of the clique complex up to dimension $2$. This is the same setting that we used in the persistent homology baselines in the experiment from \Cref{sec:sim}.
The analysis of the empirical runtimes was performed on a compute node with $\times 2$ CPU \textsc{AMD EPYC 7763 64-Core Processor} and  1 TiB \textsc{DRAM}.

In addition, we repeat the experiment on distinguishing graphs from different simulated distributions from \Cref{tab:sim_var_dist} with larger graphs of cardinalities ranging from $16$ to $4096$, where we generate $5$ graphs from each cardinality. We use the fast edge dropping method described above and additionally for all methods chunk the graphs and process the chunks in parallel. We choose parameters, so that the graphs remain connected and evaluate the diversity curves at cardinalities $\mathcal{I} = \{16, 30, 50, 70, 90\}$. \Cref{tab:vardist_large} shows accuracies and silhouette scores as well as runtimes for diversity curves against NetLSD, FEATHER and SpectralZoo. The other baseline methods encountered computational bottlenecks, reaching memory or runtime errors and are thus not reported. Overall, we find that diversity curves retain competitive performance at distinguishing graph distributions while remaining a fast alternative, even at more vastly varying ranges of graph sizes. Particularly, the use of coarsening and a smart selection of evaluation scales our methods remain scalable on larger graphs.

\begin{table}[tbh]
\centering
\caption{Comparison of methods across node ranges.}
\label{tab:vardist_large}
\begin{tabular}{llrrr}
\toprule
Node Range & Method & Accuracy & Silhouette Score & Time \\
\midrule
\multirow{4}{*}{$2^4$--$2^8$}
  & Diversity Curve (ours)           & {\boldmath\textbf{0.74 $\pm$ 0.21}} & {\boldmath\textbf{0.08 $\pm$ 0.06}} & 6.0s $\pm$ 0.0s \\
  & NetLSD       & 0.60 $\pm$ 0.10 & $-$0.12 $\pm$ 0.05 & {\boldmath\textbf{1.1s $\pm$ 0.1s}} \\
  & FEATHER      & 0.57 $\pm$ 0.09 & $-$0.07 $\pm$ 0.02 & 3.5s $\pm$ 0.0s \\
  & SpectralZoo & 0.54 $\pm$ 0.10 & $-$0.19 $\pm$ 0.04 & 1.2s $\pm$ 0.0s \\
\midrule
\multirow{4}{*}{$2^4$--$2^9$}
  & Diversity Curve (ours)           & {\boldmath\textbf{0.74 $\pm$ 0.19}} & {\boldmath\textbf{$-$0.02 $\pm$ 0.06}} & 9.5s $\pm$ 0.1s \\
  & NetLSD       & 0.73 $\pm$ 0.19 & $-$0.15 $\pm$ 0.05 & 9.1s $\pm$ 0.0s \\
  & FEATHER      & 0.60 $\pm$ 0.11 & $-$0.06 $\pm$ 0.01 & 31.4s $\pm$ 0.0s \\
  & SpectralZoo & 0.60 $\pm$ 0.12 & $-$0.21 $\pm$ 0.04 & {\boldmath\textbf{2.1s $\pm$ 0.1s}} \\
\midrule
\multirow{4}{*}{$2^4$--$2^{10}$}
  & Diversity Curve (ours)           & 0.72 $\pm$ 0.21 & $-$0.08 $\pm$ 0.05 & 37.8s $\pm$ 0.2s \\
  & NetLSD       & {\boldmath\textbf{0.74 $\pm$ 0.11}} & $-$0.17 $\pm$ 0.05 & 8m 7s $\pm$ 2.8s \\
  & FEATHER      & 0.69 $\pm$ 0.09 & {\boldmath\textbf{$-$0.05 $\pm$ 0.01}} & 4m 4s $\pm$ 0.8s \\
  & SpectralZoo & 0.66 $\pm$ 0.11 & $-$0.23 $\pm$ 0.03 & {\boldmath\textbf{5.3s $\pm$ 0.1s}} \\
\midrule
\multirow{4}{*}{$2^4$--$2^{11}$}
  & Diversity Curve (ours)           & 0.58 $\pm$ 0.34 & $-$0.13 $\pm$ 0.05 & 2m 42s $\pm$ 0.5s \\
  & NetLSD       & {\boldmath\textbf{0.78 $\pm$ 0.19}} & $-$0.18 $\pm$ 0.05 & 27m 28s $\pm$ 16.7s \\
  & FEATHER      & 0.69 $\pm$ 0.13 & {\boldmath\textbf{$-$0.05 $\pm$ 0.01}} & 35m 53s $\pm$ 3.2s \\
  & SpectralZoo & 0.63 $\pm$ 0.19 & $-$0.23 $\pm$ 0.04 & {\boldmath\textbf{17.6s $\pm$ 0.1s}} \\
\midrule
\multirow{4}{*}{$2^4$--$2^{12}$}
  & Diversity Curve (ours)           & 0.67 $\pm$ 0.31 & {\boldmath\textbf{$-$0.17 $\pm$ 0.05}} & 11m 30s $\pm$ 0.4s \\
  & NetLSD       & {\boldmath\textbf{0.81 $\pm$ 0.19}} & $-$0.21 $\pm$ 0.04 & 1h 23m $\pm$ 13.9s \\
  & FEATHER      & -- & -- & $>$ 2h \\
  & SpectralZoo & 0.67 $\pm$ 0.19 & $-$0.24 $\pm$ 0.04 & {\boldmath\textbf{1m 15s $\pm$ 0.2s}} \\
\bottomrule
\end{tabular}
\end{table}
 
\subsection{Ablation studies}
\label{app:ablation}
In the following, we provide an extended overview of the ablation studies and additional experiments performed to clarify the impact of parameter, metric, evaluation scale, and coarsening method choices on the evaluation and performance of diversity curves in practice.

\subsubsection{Choice of distance and edge scores}
\label{app:ablation_metric}
We perform an ablation study to investigate the choice of our parameters, specifically the choice of graph distance to compute the spread and the choice of the edge scoring method. With the same experimental setup for distinguishing parameter choices in the \texttt{ER}, \texttt{RP}, \texttt{RG}, and \texttt{SBM} graphs as described in \Cref{app:simulated_graphs}, we compute the diversity curves for the spread with shortest-path distance, diffusion distance, and resistance distance. 
Furthermore, for each choice of distance we vary the edge scoring method between random edge scores, edge scores based on the difference in spread, and an approximation of the edge scores based on the spread which 
approximates the distance between nodes after edge contraction by updating the distances on the original graph by taking the minimum distance to the original nodes \citep{limbeck2025geometryaware}. 
We report all the accuracies and clustering metrics in \Cref{tab:ablation_ER}, \Cref{tab:ablationSBM}, \Cref{tab:ablationsRG}, and \Cref{tab:ablationsRPG}. We further visualise the different silhouette scores in \Cref{fig:ablation_silh}. We find that over our chosen random graph models, the shortest-path distance performs the best and the choice of edge scoring does not have any significant impact, which is why we decided to opt for random edge scoring in order to reduce the computational time complexity.

\begin{figure}[tbh]
\centering
\includegraphics[width=0.8\textwidth]{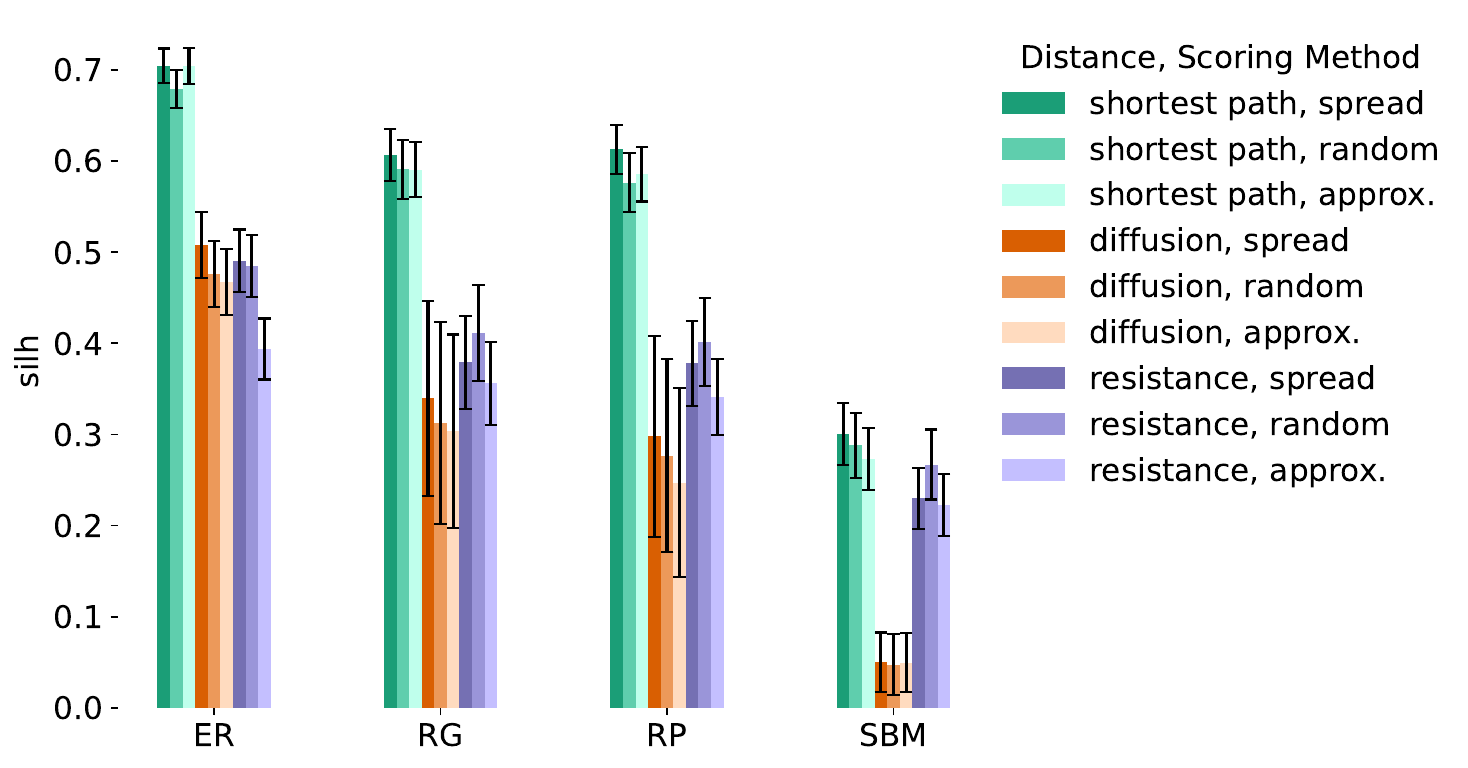}
\caption{Silhouette scores at distinguishing different parameter choices, varying the distance used to compute the spread and the edge scoring method.}
\label{fig:ablation_silh}
\end{figure}

\begin{table}[tbh]
    \centering
    \caption{Accuracies and clustering scores for distinguishing parameter choices for \texttt{ER} graphs.
    }
    \resizebox{\textwidth}{!}{
\begin{tabular}{llllllll}
\toprule
 &  & Accuracy & Silhouette Score & Davies-Bouldin & Calinski-Harabasz & ARI agglomerative & ARI k-means \\
Distance & Score &  &  &  &  &  &  \\
\midrule
shortest path & spread & \textbf{1.000 $\pm $ 0.000} & \textbf{0.704 $\pm $ 0.019} & 0.313 $\pm $ 0.017 & \textbf{288.211 $\pm $ 57.504} & \textbf{0.491 $\pm $ 0.048} & 0.442 $\pm $ 0.014 \\
 & random & \textbf{1.000 $\pm $ 0.000} & 0.679 $\pm $ 0.021 & 0.346 $\pm $ 0.018 & 272.892 $\pm $ 54.881 & 0.486 $\pm $ 0.053 & \textbf{0.444 $\pm $ 0.056} \\
 & approx. & \textbf{1.000 $\pm $ 0.000} & 0.704 $\pm $ 0.020 & \textbf{0.312 $\pm $ 0.019} & 282.266 $\pm $ 57.066 & 0.491 $\pm $ 0.047 & 0.441 $\pm $ 0.015 \\
diffusion & spread & \textbf{1.000 $\pm $ 0.000} & 0.508 $\pm $ 0.036 & 0.552 $\pm $ 0.056 & 69.180 $\pm $ 15.427 & 0.190 $\pm $ 0.091 & 0.271 $\pm $ 0.061 \\
 & random & \textbf{1.000 $\pm $ 0.000} & 0.476 $\pm $ 0.036 & 0.560 $\pm $ 0.052 & 67.561 $\pm $ 15.040 & 0.194 $\pm $ 0.102 & 0.270 $\pm $ 0.063 \\
 & approx. & \textbf{1.000 $\pm $ 0.000} & 0.467 $\pm $ 0.036 & 0.557 $\pm $ 0.049 & 69.157 $\pm $ 15.465 & 0.191 $\pm $ 0.094 & 0.271 $\pm $ 0.061 \\
resistance &  spread & 0.994 $\pm $ 0.017 & 0.491 $\pm $ 0.034 & 0.579 $\pm $ 0.036 & 163.684 $\pm $ 34.156 & 0.390 $\pm $ 0.068 & 0.375 $\pm $ 0.033 \\
 &  random & \textbf{1.000 $\pm $ 0.000} & 0.485 $\pm $ 0.034 & 0.569 $\pm $ 0.034 & 146.921 $\pm $ 30.415 & 0.376 $\pm $ 0.074 & 0.365 $\pm $ 0.034 \\
 &  approx. & \textbf{1.000 $\pm $ 0.000} & 0.394 $\pm $ 0.033 & 0.791 $\pm $ 0.053 & 158.064 $\pm $ 33.425 & 0.368 $\pm $ 0.069 & 0.361 $\pm $ 0.033 \\
\bottomrule
\end{tabular}
    }
    \label{tab:ablation_ER}
\end{table}
\begin{table}[tbh]
    \centering
    \caption{Accuracies and clustering scores for distinguishing parameter choices for \texttt{RP} graphs.
    }
    \resizebox{\textwidth}{!}{
\begin{tabular}{llllllll}
\toprule
 &  & Accuracy & Silhouette Score & Davies-Bouldin & Calinski-Harabasz & ARI agglomerative & ARI k-means \\
Distance & Score &  &  &  &  &  &  \\
\midrule
shortest path & spread & \textbf{0.978 $\pm $ 0.027} & \textbf{0.613 $\pm $ 0.027} & 0.505 $\pm $ 0.037 & \textbf{285.518 $\pm $ 36.276} & \textbf{0.480 $\pm $ 0.153} & \textbf{0.789 $\pm $ 0.188} \\
 & random & \textbf{0.978 $\pm $ 0.037} & 0.576 $\pm $ 0.032 & \textbf{0.501 $\pm $ 0.038} & 285.485 $\pm $ 47.777 & 0.448 $\pm $ 0.062 & 0.674 $\pm $ 0.139 \\
 & approx. & \textbf{0.978 $\pm $ 0.037} & 0.585 $\pm $ 0.030 & 0.532 $\pm $ 0.043 & 246.763 $\pm $ 34.457 & 0.433 $\pm $ 0.135 & 0.740 $\pm $ 0.187 \\
diffusion & spread & 0.961 $\pm $ 0.056 & 0.298 $\pm $ 0.110 & 0.829 $\pm $ 0.155 & 55.944 $\pm $ 13.154 & 0.218 $\pm $ 0.041 & 0.218 $\pm $ 0.041 \\
 & random & 0.967 $\pm $ 0.037 & 0.277 $\pm $ 0.106 & 0.846 $\pm $ 0.174 & 62.807 $\pm $ 15.787 & 0.218 $\pm $ 0.041 & 0.218 $\pm $ 0.041 \\
 & approx. & 0.972 $\pm $ 0.028 & 0.247 $\pm $ 0.104 & 0.884 $\pm $ 0.168 & 55.656 $\pm $ 13.187 & 0.218 $\pm $ 0.041 & 0.218 $\pm $ 0.041 \\
resistance &  spread & 0.972 $\pm $ 0.037 & 0.378 $\pm $ 0.047 & 0.777 $\pm $ 0.060 & 146.244 $\pm $ 27.256 & 0.346 $\pm $ 0.051 & 0.371 $\pm $ 0.045 \\
 &  random & 0.967 $\pm $ 0.044 & 0.402 $\pm $ 0.048 & 0.728 $\pm $ 0.058 & 181.881 $\pm $ 39.252 & 0.364 $\pm $ 0.062 & 0.384 $\pm $ 0.047 \\
 &  approx. & 0.956 $\pm $ 0.069 & 0.341 $\pm $ 0.042 & 0.847 $\pm $ 0.049 & 141.658 $\pm $ 28.112 & 0.346 $\pm $ 0.042 & 0.359 $\pm $ 0.053 \\
\bottomrule
\end{tabular}
    }
    \label{tab:ablationsRPG}
\end{table}
\begin{table}[tbh]
    \centering
    \caption{Accuracies and clustering scores for distinguishing parameter choices for \texttt{RG} graphs.
    }
    \resizebox{\textwidth}{!}{
\begin{tabular}{llllllll}
\toprule
 &  & Accuracy & Silhouette Score & Davies-Bouldin & Calinski-Harabasz & ARI agglomerative & ARI k-means \\
Distance & Score &  &  &  &  &  &  \\
\midrule
shortest path & spread & \textbf{0.983 $\pm $ 0.036} & \textbf{0.607 $\pm $ 0.028} & \textbf{0.445 $\pm $ 0.037} & 393.501 $\pm $ 80.612 & 0.485 $\pm $ 0.051 & \textbf{0.486 $\pm $ 0.120} \\
 & random & 0.967 $\pm $ 0.044 & 0.591 $\pm $ 0.032 & 0.483 $\pm $ 0.046 & \textbf{427.861 $\pm $ 97.530} & \textbf{0.486 $\pm $ 0.056} & 0.478 $\pm $ 0.129 \\
 & approx. & 0.972 $\pm $ 0.037 & 0.591 $\pm $ 0.030 & 0.472 $\pm $ 0.042 & 381.289 $\pm $ 83.355 & 0.478 $\pm $ 0.053 & 0.462 $\pm $ 0.114 \\
diffusion & spread & 0.956 $\pm $ 0.042 & 0.339 $\pm $ 0.107 & 0.784 $\pm $ 0.184 & 124.151 $\pm $ 26.770 & 0.292 $\pm $ 0.107 & 0.357 $\pm $ 0.036 \\
 & random & 0.950 $\pm $ 0.046 & 0.313 $\pm $ 0.110 & 0.845 $\pm $ 0.194 & 124.600 $\pm $ 27.846 & 0.287 $\pm $ 0.110 & 0.352 $\pm $ 0.037 \\
 & approx. & 0.950 $\pm $ 0.046 & 0.304 $\pm $ 0.106 & 0.863 $\pm $ 0.188 & 124.430 $\pm $ 26.946 & 0.292 $\pm $ 0.107 & 0.357 $\pm $ 0.035 \\
resistance &  spread & 0.950 $\pm $ 0.058 & 0.379 $\pm $ 0.051 & 0.798 $\pm $ 0.085 & 219.894 $\pm $ 57.742 & 0.439 $\pm $ 0.080 & 0.407 $\pm $ 0.029 \\
 &  random & 0.967 $\pm $ 0.044 & 0.411 $\pm $ 0.053 & 0.757 $\pm $ 0.087 & 254.085 $\pm $ 69.380 & 0.460 $\pm $ 0.068 & 0.412 $\pm $ 0.027 \\
 &  approx. & 0.944 $\pm $ 0.056 & 0.356 $\pm $ 0.046 & 0.840 $\pm $ 0.071 & 218.943 $\pm $ 57.164 & 0.438 $\pm $ 0.081 & 0.407 $\pm $ 0.028 \\
\bottomrule
\end{tabular}
    }
    \label{tab:ablationsRG}
\end{table}
\begin{table}[h!]
    \centering
    
    \caption{Accuracies and clustering scores for distinguishing parameter choices of \texttt{SBM} graphs
    }
    \resizebox{\textwidth}{!}{
\begin{tabular}{llllllll}
\toprule
 &  & Accuracy & Silhouette Score & Davies-Bouldin & Calinski-Harabasz & ARI agglomerative & ARI k-means \\
Distance & Score &  &  &  &  &  &  \\
\midrule
shortest path & spread & \textbf{0.861 $\pm $ 0.062} & \textbf{0.301 $\pm $ 0.034} & 1.122 $\pm $ 0.130 & 94.305 $\pm $ 13.963 & 0.371 $\pm $ 0.117 & \textbf{0.432 $\pm $ 0.076} \\
 & random & 0.856 $\pm $ 0.079 & 0.288 $\pm $ 0.036 & 1.062 $\pm $ 0.115 & \textbf{98.923 $\pm $ 15.813} & \textbf{0.397 $\pm $ 0.064} & 0.427 $\pm $ 0.068 \\
 & approx. & 0.861 $\pm $ 0.062 & 0.273 $\pm $ 0.034 & 1.154 $\pm $ 0.130 & 87.747 $\pm $ 13.267 & 0.376 $\pm $ 0.108 & 0.421 $\pm $ 0.074 \\
diffusion & spread & 0.783 $\pm $ 0.072 & 0.050 $\pm $ 0.033 & 2.158 $\pm $ 0.446 & 25.955 $\pm $ 5.913 & 0.153 $\pm $ 0.058 & 0.136 $\pm $ 0.035 \\
 & random & 0.822 $\pm $ 0.054 & 0.048 $\pm $ 0.034 & 2.156 $\pm $ 0.457 & 25.918 $\pm $ 6.075 & 0.162 $\pm $ 0.047 & 0.160 $\pm $ 0.045 \\
 & approx. & 0.783 $\pm $ 0.072 & 0.050 $\pm $ 0.033 & 2.161 $\pm $ 0.444 & 25.843 $\pm $ 5.908 & 0.152 $\pm $ 0.056 & 0.136 $\pm $ 0.035 \\
resistance &  spread & 0.822 $\pm $ 0.060 & 0.230 $\pm $ 0.033 & 1.213 $\pm $ 0.128 & 70.056 $\pm $ 10.781 & 0.329 $\pm $ 0.095 & 0.374 $\pm $ 0.067 \\
 &  random & 0.817 $\pm $ 0.061 & 0.267 $\pm $ 0.038 & \textbf{1.059 $\pm $ 0.106} & 87.486 $\pm $ 15.035 & 0.361 $\pm $ 0.084 & 0.399 $\pm $ 0.078 \\
 &  approx. & 0.817 $\pm $ 0.050 & 0.223 $\pm $ 0.034 & 1.237 $\pm $ 0.128 & 66.452 $\pm $ 10.478 & 0.328 $\pm $ 0.094 & 0.364 $\pm $ 0.069 \\
\bottomrule
\end{tabular}
    } 
    \label{tab:ablationSBM}
\end{table}

\newpage


\subsubsection{Stability with respect to the edge contraction sequence}
\label{app:ablation_stability}
Because computing multiset of diversity curves over all possible edge contractions is infeasible in practice, particularly on large graphs, we have to limit ourselves to evaluating a smaller selection of possible edge contraction sequences. 
This will introduce some randomness in the selection of the edge contraction sequence. Even when choosing the edge contraction sequence by using a process that is based only on the graph structure, edges in symmetric graphs that cannot structurally be distinguished will obtain the same importance score. Only a canonical labelling of the edges in the graph would give us a unique score for each edge that is invariant under isomorphisms, but this difficult problem related to the graph isomorphism problem is infeasible to use in practice \citep{babai1983canonical}. 
Hence, it is desirable to investigate the stability of diversity curves with respect to the randomness of choosing a subset of all possible edge contractions. In this ablation we study how stable the diversity curves are with respect to different edge contraction sequences and random seeds.

We follow the experiment from \Cref{sec:sim} to distinguish between the different random graph distributions and choose the same parameter settings as described in \Cref{app:simulated_graphs}. For each graph we compute its diversity curve based only on one sequence of random edge contractions and repeat the computation over $10$ different random seeds. 
As a first measure of the stability, we consider the distribution of the mean accuracies from the kNN-classification over the $10$ different seeds. We find that the average and standard deviation of the mean accuracies obtained directly from the embedding through the diversity curves is $0.88\pm0.009$.

\begin{minipage}[t]{0.45\textwidth}
\centering
\includegraphics[trim={0 0 0 1.8cm},clip,width=\linewidth]{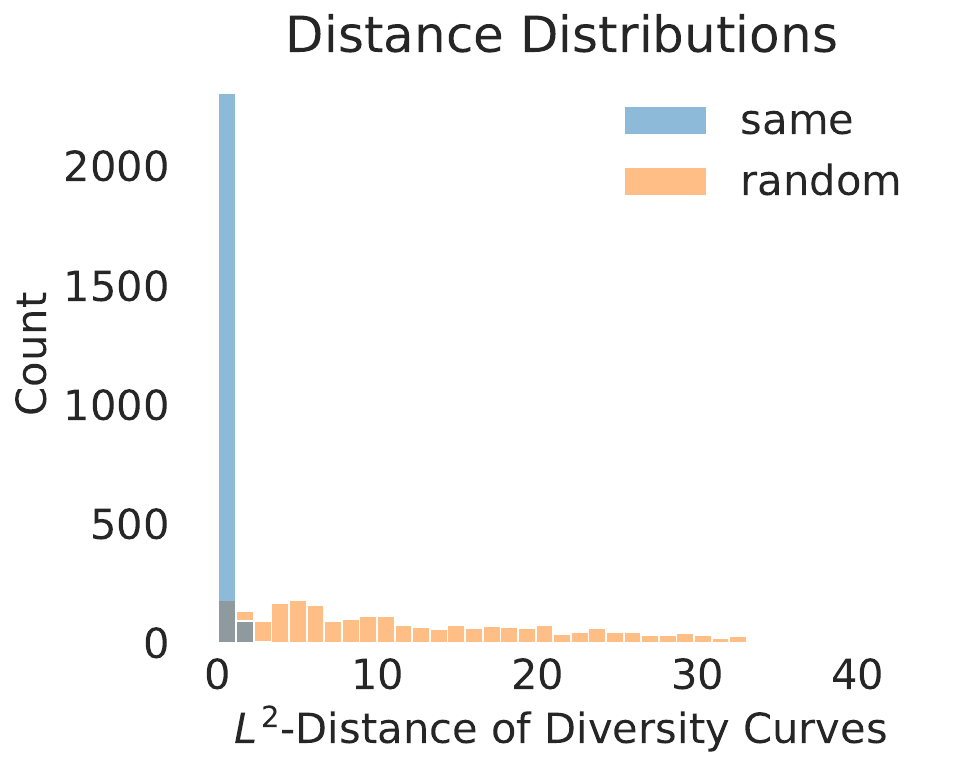}
\captionof{figure}{Distribution of differences between the diversity curve of each graph in each seed and its mean diversity curve over all seeds (\textit{same}) or when comparing to a randomly chosen graph from another graph model (\textit{random}).}
\label{fig:seed_stability}
\end{minipage} \hfill
\begin{minipage}[t]{0.45\textwidth}
\centering
\includegraphics[width=\linewidth]{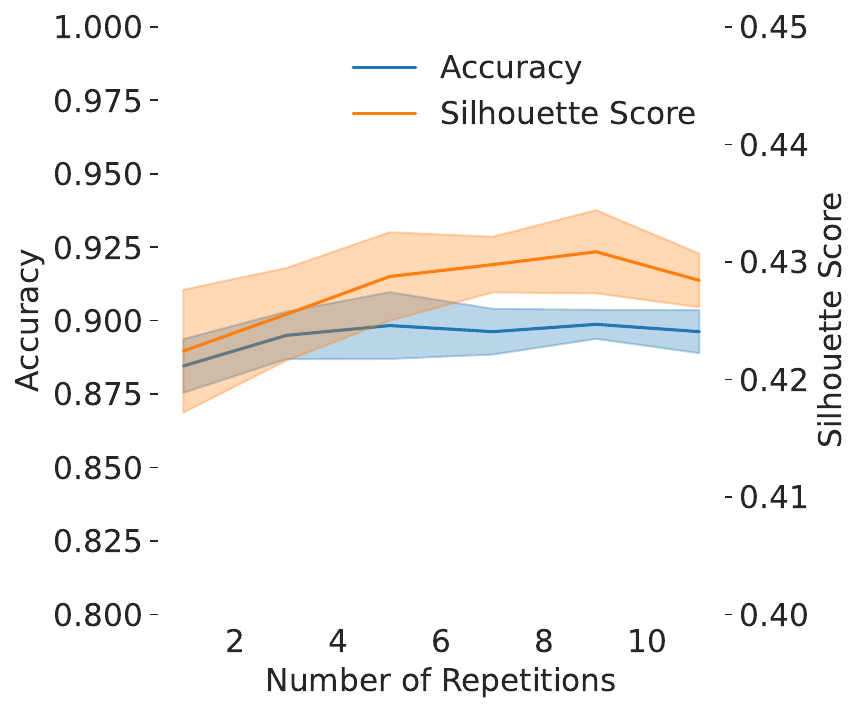}
\captionof{figure}{Mean accuracy and silhouette score (mean$\pm$std) at distinguishing graph models over $10$ different seeds for varying number of repetitions over which we compute and average the diversity curves.}
\label{fig:ablation_nreps}
\end{minipage}

Secondly, we compare the differences between the diversity curves themselves. To do this, we compute the mean diversity curve of each graph over all $10$ seeds and then for each graph in each seed we once get the $L^2$-distance to the mean curve of the same graph and to a randomly selected graph from a different graph distribution. The mean and standard deviation for the distance distributions of the diversity curve between the same graphs is $0.48 \pm 0.33$, while for the distance distribution of the diversity curve with a randomly chosen graph from a different distribution is $14.68 \pm 8.99$. \Cref{fig:seed_stability} shows a histogram of the two distributions. 
To conclude, we find that empirically the diversity curves are stable with respect to different random edge contractions, validating our approach of considering a subset of the multiset of the diversity curves for our empirical evaluation.

As we saw, the choice of random edge contraction sequence introduces some variance in the diversity curves, thus in the experiments we choose to repeat the diversity curves over $3$ different random edge contraction sequences and take the mean. This ablation addresses the influence on the number of repeats in computing the diversity curves. Concretely, we repeat the experiment on the simulated graph dataset for distinguishing between different distributions from \Cref{tab:sim_var_dist} while varying the number of repetitions for which we compute diversity curves, ranging over $\{1,3,5,7,9,11\}$. For each fixed number of repetitions, we repeat the experiment over $10$ random seeds and report the mean and standard deviation over the different seeds of the mean accuracy and mean silhouette score in \Cref{fig:ablation_nreps}. We find that for only one repetition, the scores are lowest but that overall the diversity curves remain stable with respect to varying the number of repetitions.

\subsubsection{Choice of evaluation scales, expressivity, and computational cost}
\label{app:tradeoff_ev_scales}
In this ablation study, we consider the trade-off between the choice of evaluation scales, the computational efficiency and expressivity. We repeat the experiment on distinguishing simulated graph distributions from \Cref{tab:sim_var_dist} by varying the numbers of evaluation scales used for computing diversity curves, evenly chosen between 1 and 29. The results are reported in \Cref{tab:tradeof_ev_scales}. Computing the diversity curves at fewer scales reduces the runtimes while the accuracy remains reasonably high, only decreasing for lower numbers of scales. This demonstrates that diversity curves remain robust and expressive across scales.

\subsubsection{Performance across up- or downsampling}
\label{app:up_downsampling}

To investigate the impact of the up- or downsampling strategies, we modify the experiment on distinguishing parameter choices on simulated graphs from \Cref{tab:sim_acc_param} to consider down- or up-sampling operations individually. We compute coarsened graph representations that only use downsampling to less than 10 nodes, and another version where we upsample every graph to between 31 and 50 nodes. The results are found in \Cref{tab:up_down_sampling}. In this setup, exclusively downsampling reaches an accuracy below the main results reported in our paper, likely due to the low cardinalities, while upsampling alone shows an accuracy comparable to the accuracies reported in \Cref{tab:sim_acc_param}. Overall, the strength of our methods lies in being able to upsample some graphs and downsample others to the same node sizes to compare them across a shared range of cardinalities.
\begin{table}[t]
\centering
\caption{Downsample and upsample accuracy across graph models.}
\label{tab:up_down_sampling}
\begin{tabular}{lrrrr}
\toprule
 & \texttt{ER} & \texttt{RG} & \texttt{RP} & \texttt{SBM} \\
\midrule
Downsample Accuracy & 0.910 $\pm$ 0.062 & 0.900 $\pm$ 0.060 & 0.906 $\pm$ 0.090 & 0.800 $\pm$ 0.100 \\
Upsample Accuracy   & {\boldmath\textbf{1.000 $\pm$ 0.000}} & {\boldmath\textbf{0.978 $\pm$ 0.037}} & {\boldmath\textbf{0.983 $\pm$ 0.025}} & {\boldmath\textbf{0.856 $\pm$ 0.067}} \\
\bottomrule
\end{tabular}
\end{table}

\subsection{Extended results on graph classification datasets}
\label{app:extenden_large_datasets}

Note that our work focuses on evaluating the utility of structural descriptors for unsupervised tasks across coarsening levels. Our methods are designed to encode geometric and structural properties of unattributed graphs and meshes. When applied to attributed graphs and graph classification tasks, for example from TU datasets \citep{morris2020tudataset}, we do not expect unsupervised methods to reach state-of-the-art performance, but we can nevertheless compare the performance of different unsupervised baselines.

Hence, we provide extended results on \texttt{MCF-7} and \texttt{MOLT-4} in \Cref{tab:accuracy_MCF_MOLT,tab:auroc_MCF_MOLT}, two large molecular graph datasets from TU datasets \citep{morris2020tudataset}, containing 27,770 and 39,765 graphs respectively. Specifically, we compare kNN classification accuracies across 10-fold stratified cross-validation following the experimental setup reported in our paper. Some of the kernel methods were not evaluated due to computational bottlenecks and memory allocation limitations. This highlights the computational advantages of our approach for processing large graph datasets over less scalable baseline methods.

Overall, diversity curves reach competitive performance and comparable accuracies to relevant baseline methods, such as FEATHER, SpectralZoo, TopER, and kernel methods. To interpret the results, we note that it remains uncertain whether the suggested graph datasets show more meaningful local graph structures or global differences in graph structure, which could warrant further investigation. Nevertheless, our methods show satisfactory performance across these graph classification tasks. 

\begin{table}[htbp]
\centering
\begin{minipage}[h]{0.49\linewidth}
\centering
\caption{Accuracy on \texttt{MCF-7} and \texttt{MOLT-4}.}
\label{tab:accuracy_MCF_MOLT}
\resizebox{\linewidth}{!}{%
\begin{tabular}{lrr}
\toprule
Method & \texttt{MCF-7} & \texttt{MOLT-4} \\
\midrule
Diversity Curves (Diffusion Dist.) & 0.915 $\pm$ 0.002 & 0.919 $\pm$ 0.002 \\
Diversity Curves (Shortest Path)   & 0.916 $\pm$ 0.002 & 0.919 $\pm$ 0.001 \\
\midrule
Graph2Vec                          & {\boldmath\textbf{0.923 $\pm$ 0.002}} & {\boldmath\textbf{0.927 $\pm$ 0.001}} \\
NetLSD                             & 0.919 $\pm$ 0.003 & 0.922 $\pm$ 0.001 \\
FEATHER                            & 0.915 $\pm$ 0.002 & 0.919 $\pm$ 0.002 \\
Spectral Zoo                       & 0.916 $\pm$ 0.001 & 0.920 $\pm$ 0.000 \\
Kernel Core Framework              & 0.917 $\pm$ 0.004 & OOM \\
Kernel Shortest Path               & 0.917 $\pm$ 0.001 & OOM \\
TopER Degree                       & 0.915 $\pm$ 0.002 & 0.919 $\pm$ 0.003 \\
TopER Forman-Ricci                 & 0.914 $\pm$ 0.002 & 0.920 $\pm$ 0.002 \\
TopER Popularity                   & 0.914 $\pm$ 0.002 & 0.918 $\pm$ 0.001 \\
\bottomrule
\end{tabular}%
}
\end{minipage}%
\hfill
\begin{minipage}[h]{0.49\linewidth}
\centering
\caption{AUROC on \texttt{MCF-7} and \texttt{MOLT-4}.}
\label{tab:auroc_MCF_MOLT}
\resizebox{\linewidth}{!}{%
\begin{tabular}{lrr}
\toprule
Method & \texttt{MCF-7} & \texttt{MOLT-4} \\
\midrule
Diversity Curves (Diffusion Dist.) & 0.668 $\pm$ 0.007 & 0.610 $\pm$ 0.012 \\
Diversity Curves (Shortest Path)   & 0.650 $\pm$ 0.018 & 0.619 $\pm$ 0.021 \\
\midrule
Graph2Vec                          & 0.692 $\pm$ 0.013 & {\boldmath\textbf{0.699 $\pm$ 0.013}} \\
NetLSD                             & {\boldmath\textbf{0.698 $\pm$ 0.018}} & 0.684 $\pm$ 0.017 \\
FEATHER                            & 0.669 $\pm$ 0.017 & 0.663 $\pm$ 0.023 \\
Spectral Zoo                       & 0.525 $\pm$ 0.009 & 0.510 $\pm$ 0.012 \\
Kernel Core Framework              & {\boldmath\textbf{0.698 $\pm$ 0.022}} & OOM \\
Kernel Shortest Path               & 0.657 $\pm$ 0.011 & OOM \\
TopER Degree                       & 0.632 $\pm$ 0.019 & 0.612 $\pm$ 0.008 \\
TopER Forman-Ricci                 & 0.631 $\pm$ 0.017 & 0.617 $\pm$ 0.014 \\
TopER Popularity                   & 0.638 $\pm$ 0.025 & 0.628 $\pm$ 0.014 \\
\bottomrule
\end{tabular}%
}
\end{minipage}
\end{table}

\newpage

\clearpage
\section*{NeurIPS Paper Checklist}

\begin{enumerate}

\item {\bf Claims}
    \item[] Question: Do the main claims made in the abstract and introduction accurately reflect the paper's contributions and scope?
    \item[] Answer: \answerYes{} 
    \item[] Justification: The claims and contributions described in our paper are all supported either by theoretical or empirical evidence on the performance of our methods across a range of graph-level tasks taking into account baseline methods and existing literature on graph representation learning. 

\item {\bf Limitations}
    \item[] Question: Does the paper discuss the limitations of the work performed by the authors?
    \item[] Answer: \answerYes{} 
    \item[] Justification: We include a section on the practical and theoretical limitations of our methods in \Cref{sec:limitations}. We discuss computational efficiency in \Cref{sec:costs_main} and \Cref{app:empirical_costs}. Further, we reflect upon the relevant practical and theoretical imitations when detailing each experiment and theoretical result in the appendix and main text.

\item {\bf Theory assumptions and proofs}
    \item[] Question: For each theoretical result, does the paper provide the full set of assumptions and a complete (and correct) proof?
    \item[] Answer: \answerYes{} 
    \item[] Justification: We report all theoretical results in detail in \Cref{app:theoretical results} and a condensed summary in \Cref{sec:theory}. All theorems are formally stated and all relevant assumptions are detailed throughout our theoretical analysis. 

    \item {\bf Experimental result reproducibility}
    \item[] Question: Does the paper fully disclose all the information needed to reproduce the main experimental results of the paper to the extent that it affects the main claims and/or conclusions of the paper (regardless of whether the code and data are provided or not)?
    \item[] Answer: \answerYes{} 
    \item[] Justification: We provide extended experimental details for each experiment in \Cref{app:extended_experiments} as well as in \Cref{sec:experiments}, which can be used to reproduce our main results. All existing datasets and codebases which our results are based on are referenced throughout the paper, enabling reproducibility.  We publish the code for reproducing all of our results on \href{https://github.com/aidos-lab/diversity_curves/}{GitHub}. 
    As our main contribution is a novel graph representation method, our algorithm in described in detail throughout the main text in \Cref{sec:methods} and in the extended methods description in  \Cref{app:diversity_curves_methods}.

\item {\bf Open access to data and code}
    \item[] Question: Does the paper provide open access to the data and code, with sufficient instructions to faithfully reproduce the main experimental results, as described in supplemental material?
    \item[] Answer: \answerYes{} 
    \item[] Justification: We provide a link to our \href{https://github.com/aidos-lab/diversity_curves/}{GitHub}  
    for fully reproducing our methods and all results. 
    We provide clear instructions on how to run our algorithm and the experiments.

\item {\bf Experimental setting/details}
    \item[] Question: Does the paper specify all the training and test details (e.g., data splits, hyperparameters, how they were chosen, type of optimizer) necessary to understand the results?
    \item[] Answer: \answerYes{} 
    \item[] Justification: We clarify the choice of classifier, embedding model, and data splits throughout our paper and in the extended experimental description in \Cref{app:extended_experiments}. Altogether, this provides the necessary detail to reproduce our main results. 

\item {\bf Experiment statistical significance}
    \item[] Question: Does the paper report error bars suitably and correctly defined or other appropriate information about the statistical significance of the experiments?
    \item[] Answer: \answerYes{} 
    \item[] Justification: When evaluating model performance, we report standard deviations across all of our experiments, and take care to interpret the performance of our method accordingly. Each experimental section and the extended descriptions in \Cref{app:extended_experiments} specify whether these deviations were reported across random seed, cross validation splits, or resamples of the data. Further, we use diversity curves for hypothesis testing in \Cref{sec:enzymes} showing that our methods themselves can be used for statistical analysis.

\item {\bf Experiments compute resources}
    \item[] Question: For each experiment, does the paper provide sufficient information on the computer resources (type of compute workers, memory, time of execution) needed to reproduce the experiments?
    \item[] Answer: \answerYes{} 
    \item[] Justification: We specify the maximal compute resources used to run our experiments in \Cref{app:compute}. 
    
\item {\bf Code of ethics}
    \item[] Question: Does the research conducted in the paper conform, in every respect, with the NeurIPS Code of Ethics \url{https://neurips.cc/public/EthicsGuidelines}?
    \item[] Answer: \answerYes{} 
    \item[] Justification: We have reviewed and comply with the NeurIPS Code of Ethics. 

\item {\bf Broader impacts}
    \item[] Question: Does the paper discuss both potential positive societal impacts and negative societal impacts of the work performed?
    \item[] Answer: \answerNA{} 
    \item[] Justification: We do consider the contributions of our work to benefit foundational machine research as a whole, hence there are many potential consequences. However, we do not find any specific harmful society implications or immediate negative consequences of our work that we feel must be highlighted here. 
    
\item {\bf Safeguards}
    \item[] Question: Does the paper describe safeguards that have been put in place for responsible release of data or models that have a high risk for misuse (e.g., pre-trained language models, image generators, or scraped datasets)?
    \item[] Answer: \answerNA{} 
    \item[] Justification: We do not release any data or models that have a high risk of misuse. Our main contribution is the publication of method for learning graph-level representations, which we feel does not pose a high safety risk.

\item {\bf Licenses for existing assets}
    \item[] Question: Are the creators or original owners of assets (e.g., code, data, models), used in the paper, properly credited and are the license and terms of use explicitly mentioned and properly respected?
    \item[] Answer: \answerYes{} 
    \item[] Justification: We cite existing assets and publications our work is based on throughout our paper and throughout the extended experimental description in \Cref{app:extended_experiments}. 

\item {\bf New assets}
    \item[] Question: Are new assets introduced in the paper well documented and is the documentation provided alongside the assets?
    \item[] Answer: \answerYes{}{} 
    \item[] Justification: The main new asset introduced in our publication is an algorithm for computing diversity curves, which are novel graph representations. We detail this algorithm throughout our paper and publish code to use our methods with the publication of this manuscript.

\item {\bf Crowdsourcing and research with human subjects}
    \item[] Question: For crowdsourcing experiments and research with human subjects, does the paper include the full text of instructions given to participants and screenshots, if applicable, as well as details about compensation (if any)? 
    \item[] Answer: \answerNA{} 
    \item[] Justification: This work does not include any research or crowdsourcing with human subjects. 

\item {\bf Institutional review board (IRB) approvals or equivalent for research with human subjects}
    \item[] Question: Does the paper describe potential risks incurred by study participants, whether such risks were disclosed to the subjects, and whether Institutional Review Board (IRB) approvals (or an equivalent approval/review based on the requirements of your country or institution) were obtained?
    \item[] Answer: \answerNA{} 
    \item[] Justification: This work does not include any research or crowdsourcing with human subjects.

\item {\bf Declaration of LLM usage}
    \item[] Question: Does the paper describe the usage of LLMs if it is an important, original, or non-standard component of the core methods in this research? Note that if the LLM is used only for writing, editing, or formatting purposes and does \emph{not} impact the core methodology, scientific rigor, or originality of the research, declaration is not required.
    \item[] Answer: \answerNA{} 
    \item[] Justification: The core method development in this work does not involve LLMs as any important, original, or non-standard components.

\end{enumerate}


\end{document}